\documentclass[sigconf,screen,9pt]{acmart}

\setcopyright{acmcopyright}
\acmConference[SIGMOD'23]{ACM SIGMOD International Conference on Management of Data}{June, 2023}{Seattle, WA, USA}
\acmBooktitle{ACM SIGMOD International Conference on Management of Data (SIGMOD'23), June, 2023, Seattle, WA, USA}
\acmYear{2023}
\copyrightyear{2023}

\usepackage{enumitem}
\usepackage[linesnumbered, lined, ruled]{algorithm2e}
\usepackage{subfigure}
\usepackage{multirow}
\usepackage{arydshln}

\newlist{compactitem}{itemize}{3} % 3 is max-depth
\setlist[compactitem]{label=\textbullet, nosep, leftmargin=0cm,itemindent=.5cm}

\newtheorem{theorem}{Theorem}[section]

\theoremstyle{definition}
\newtheorem{definition}{Definition}[section]
\newcommand{\mS}{\mathcal{S}}
\newcommand{\mR}{\mathcal{R}}
\newcommand{\mC}{\mathcal{C}}
\newcommand{\mL}{\mathcal{L}}
\newcommand{\sys}{STAIR\xspace}

\newcommand{\wy}[1]{\textcolor{blue}{#1}}

\begin{document}

\title{Interpretable Outlier Summarization}

\author{Yu Wang}
\affiliation{%
  \institution{University of California, San Diego}
  \city{San Diego, California}
  \country{USA}
}
\email{yuw164@ucsd.edu}

\author{Lei Cao}
\authornote{Corresponding Author}
\affiliation{%
  \institution{Massachusetts Institute of Technology}
  \city{Cambridge}
  \state{MA}
  \country{USA}
}
\email{lcao@csail.mit.edu}

\author{Yizhou Yan}
\affiliation{%
  \institution{Worcester Polytechnic Institute}
  \city{Worcester}
  \state{MA}
  \country{USA}
}
\email{yyan2@wpi.edu}

\author{Samuel Madden}
\affiliation{%
  \institution{Massachusetts Institute of Technology}
  \city{Cambridge}
  \state{MA}
  \country{USA}
}
\email{madden@csail.mit.edu}

\begin{abstract}
Outlier detection is critical in real applications to prevent financial fraud, defend network intrusions, or detecting imminent device failures.
To reduce the human effort in evaluating outlier detection results and effectively turn the outliers into actionable insights, the users often expect a system to automatically produce interpretable summarizations of subgroups of outlier detection results.
Unfortunately, to date no such systems exist. 
To fill this gap, we propose \sys which learns a compact set of human understandable \textit{rules} to summarize and explain the anomaly detection results. 
Rather than use the classical decision tree algorithms to produce these rules, \sys proposes a new optimization objective to produce a small number of rules with least complexity, hence strong interpretability, to accurately summarize the detection results.
The learning algorithm of \sys produces a rule set by iteratively splitting the large rules and is optimal in maximizing this objective in each iteration.    
Moreover, to effectively handle high dimensional, highly complex data sets which are hard to summarize with simple rules, we propose a {\it localized} \sys approach, called L-\sys. Taking data locality into consideration, it simultaneously partitions data and learns a set of localized rules for each partition.
Our experimental study on many outlier benchmark datasets shows that \sys significantly reduces the complexity of the rules required to summarize the outlier detection results, thus more amenable for humans to understand and evaluate, compared to the decision tree methods.

%That is, \sys produces a set of human understandable abstractions, each describing the common sub-properties of the detection results. This allows users to efficiently inspect a large number of anomalies and diagnose their root causes by only examining a small set of interpretable abstractions. 

%If a system can summarize anomaly detection results into groups and explain why each group of objects is considered to be abnormal, this will greatly reduce the effort of users in evaluating anomaly detection results.

\end{abstract}


\maketitle

% \pagestyle{\vldbpagestyle}
% \begingroup\small\noindent\raggedright\textbf{PVLDB Reference Format:}\\
% \vldbauthors. \vldbtitle. PVLDB, \vldbvolume(\vldbissue): \vldbpages, \vldbyear.\\
% \href{https://doi.org/\vldbdoi}{doi:\vldbdoi}
% \endgroup
% \begingroup
% \renewcommand\thefootnote{}\footnote{\noindent
% This work is licensed under the Creative Commons BY-NC-ND 4.0 International License. Visit \url{https://creativecommons.org/licenses/by-nc-nd/4.0/} to view a copy of this license. For any use beyond those covered by this license, obtain permission by emailing \href{mailto:info@vldb.org}{info@vldb.org}. Copyright is held by the owner/author(s). Publication rights licensed to the VLDB Endowment. \\
% \raggedright Proceedings of the VLDB Endowment, Vol. \vldbvolume, No. \vldbissue\ %
% ISSN 2150-8097. \\
% \href{https://doi.org/\vldbdoi}{doi:\vldbdoi} \\
% }\addtocounter{footnote}{-1}\endgroup
% %%% VLDB block end %%%

% %%% do not modify the following VLDB block %%
% %%% VLDB block start %%%
% \ifdefempty{\vldbavailabilityurl}{}{
% \vspace{.3cm}
% \begingroup\small\noindent\raggedright\textbf{PVLDB Artifact Availability:}\\
% The source code, data, and/or other artifacts have been made available at \url{\vldbavailabilityurl}.
% \endgroup
% }

\section{Introduction}
\label{introduction}

% Background
% Limitations of Current Outlier Detection algorithms
\begin{sloppypar}

\noindent\textbf{Motivation.}
Outlier detection is critical in enterprises, with applications ranging from preventing financial fraud in finance, and defending network intrusions in cyber security, to detecting imminent device failures in IoT.

To turn outliers into actionable insights, the users often expect an anomaly detection system to produce human understandable information to explain the detected outliers, e.g., by highlighting the features that contribute the most to the identified anomalies. 
Otherwise, the detected outliers are just a set of isolated data objects without any indication of their significance to the users.
Further, outlier detection frequently returns a large number of outlier candidates. This raises the problem of how to best present results such that the users do not have to sift through a huge number of results to assess the validity of the outliers one by one. 

If a system is able to {\it summarize} anomaly detection results into groups and {\it explain} why each group of objects is considered to be {\it abnormal} or {\it normal}, this will greatly reduce the effort of users in evaluating anomaly detection results.

%Such a system is to the great need of enterprises. 
%For example, in the collaboration with a world-leading supplier of industrial lighting products, its data scientists found that \sys was extremely useful in helping them quickly identify malfunctioning devices and software bugs from a large set of detected anomalies. 

%by identifying common sub-properties of anomalies. 
%abstractly summarize anomaly detection results, e.g., by describing properties of groups of objects that have been identified as abnormal.
%
\noindent\textbf{State-of-the-Art.}
To the best of our knowledge, the problem of {\it summarizing and interpreting} outlier detection results is yet to be addressed. Scorpion~\cite{DBLP:journals/pvldb/0002M13} produces meaningful explanations for anomalies in aggregation queries when the ‘cause’ of an outlier is contained in its provenance. Similar to Scorpion, Cape~\cite{DBLP:journals/pvldb/MiaoZLGKR19} aims to explain the outliers in aggregation queries, but using the objects that counterbalance the outliers. Both works do not tackle the problem of summarizing outliers. 
Macrobase~\cite{bailis2017macrobase} explains outliers by correlating them to some external attributes which are not used to detect anomalies such as location, time of occurrence,  software version, etc. However, Macrobase only targets explaining the outliers captured by its default density-based outlier detector and does not generalize to other outlier detection methods. 

In the broader field of interpretable AI, LIME~\cite{DBLP:conf/kdd/Ribeiro0G16} explains the predictions of a classifier by learning a linear model locally around the prediction with respect to each testing object and pointing out the attributes that are most important to the prediction of the linear model. However, rather than use one model to represent a set of objects, LIME has to learn a linear model for each individual object. Therefore, using LIME to explain a large number of prediction results will be prohibitively expensive. Other methods~\cite{DBLP:conf/aaai/Ribeiro0G18,DBLP:journals/corr/abs-1805-10820,DBLP:journals/corr/abs-1806-09936} explain classification results in the similar way to LIME.

%The weather in Massachusetts is considered to be abnormal if it observes an temperature at $80^oF$ in February. However, it won't be unusual in June. Therefore, the explanation has to take the correlations among multiple attributes into consideration.

%The explanation generated for each single point is not intuitive. Though finding the important attributes for the data could explain the decision of other models or algorithms to some extent, only pointing out the attributes could make users confused of what conditions should these attributes satisfy to have the model make these predictions. 

\noindent\textbf{Challenges.}
Summarizing and interpreting outliers is challenging because of its nature. Outliers are phenomena that significantly deviate from the normal phenomena. In most cases the outliers tend to be very different from each other in their features. Therefore, they cannot be simply summarized based on their similarity in the feature space measured by some similarity function. 

%To solve this problem, we propose to summarize and explain the outliers by exploring their similarities on multiple alternative perspectives in addition to the features

% introduce: intervals
\noindent\textbf{Proposed Approach.} 
To meet the need of an effective outlier summarization and interpretation tool, we have developed \sys. 
It produces a set of human understandable abstractions, each describing the common properties of a group of detection results. This allows the users to efficiently verify a large number of anomaly detection results and diagnose the root causes of the potential outliers by only examining a small set of interpretable abstractions. 

\underline{Rule-based Outlier Summarization and Interpretation.} \sys leverages classical decision tree classification to learn a compact set of human understandable \textit{rules} to summarize and explain the anomaly detection results. 
Using the results produced by an anomaly detection method as training data, \sys learns a decision tree to accurately separate outliers and inliners in the training set.
Each branch of the decision tree is composed of a set of data attributes with associated values that iteratively split the data. 
Therefore, it can be thought of as a \textit{rule} that represents a subset of data sharing the same class (outlier or inlier) and that is easy to understand by humans. 

\underline{Outlier Summarization and Interpretation-aware Objective.} However, decision tree algorithms target maximizing the classification accuracy. Rules learned in this way do not necessarily have the properties desired by outlier summarization and interpretation. 
This is because when handling highly complex data sets, to minimize the classification errors, decision trees often have to be {\it deep} trees with {\it many} branches and hence produce a lot of complex rules which are hard for humans to understand.
Although some methods like CART~\cite{reason:BreFriOlsSto84a} have been proposed to prune a learned decision tree in a post-processing step, they target avoiding overfitting and thus lifting the classification accuracy. They do not guarantee the simplicity of each rule.
%\srm{The idea of using constrained depth decision tree to produce explainable rules has been used before.  How is our approach different?}

To solve the above issues, we propose a new optimization objective customized to outlier summarization and interpretation. It targets producing the minimal number of rules that are as simple as possible, while still assuring the classification accuracy. 
However, the simplicity requirement of outlier summarization and interpretation conflicts with the accuracy requirement, while it is hard for the users to manually set an appropriate regularization term to balance the two requirements. \sys thus introduces a learnable regularization parameter into the objective and relies on the leaning algorithm to automatically make the trade-off.

\underline{Rule Generation Algorithm.} We then design an optimization algorithm to generate the summarization and interpretation-aware rules. Similar to the classic decision tree algorithms~\cite{DBLP:conf/lopal/ElaidiBA18}, \sys produces a rule set by iteratively splitting the decision node. 
In each iteration, \sys dynamically adjusts the regularization parameter to ensure that it is always able to produce a valid split which increases the objective. 
We prove that the regularization parameter and the rule split that \sys learns in each iteration as a combination is {\it optimal} in maximizing the objective.      

\underline{Localized Outlier Summarization and Interpretation.} To solve the problem that one single decision tree with a small number of simple rules is not adequate to satisfy the accuracy requirement when handling high dimensional, highly complex data sets, we propose a {\it localized} \sys approach, called L-\sys.  
Taking data locality into consideration, L-\sys divides the whole data set into multiple partitions and learns a localized tree for each partition.
Rather than first partition the data and then learn the localized tree in two disjoint steps, L-\sys jointly solves the two sub-problems. In each iteration, it optimizes the data partitioning and rule generation objectives alternatively and is guaranteed to converge to a partitioning that can be summarized with simple rules.

\noindent\textbf{Contributions.} The key contributions of this work include:
\begin{compactitem}%[leftmargin=*]
    \item To the best of our knowledge, \sys is the first approach that summarizes the outlier detection results with human interpretable rules.
    
    \item We define an outlier summarization and interpretation-aware optimization objective which targets producing the minimal number of rules with least complexity, while still guaranteeing the classification accuracy. 

    \item We design a rule generation method which is optimal in optimizing the \sys objective in each iteration.   
    
    \item We propose a localized \sys approach which jointly partitions the data and produces rules for each local partition, thus scaling \sys to high dimensional, highly complex data. 
    
    \item Our extensive experimental study confirms that compared to other decision tree methods, \sys significantly reduces the complexity and the number of rules required to summarize outlier detection results.
    %\sys effectively produces rules amenable for human evaluation. \srm{This may be overclaiming- I worry reviewers will want to see evidence w/ data from real humans that our rules work for them!} 
\end{compactitem}

\end{sloppypar}
\section{Preliminary: Decision Tree}
\label{sec.preliminary}

In this section, we overview the decision tree classification problem and its classical learning algorithms. 

\noindent\textbf{Decision Tree Overview.} Decision tree learning is a classical classification technique  where the learned function could be represented by a decision tree. 
It classifies instances by sorting them down the tree from root to the leaf node, which could predict the label of this instance. Each node in the tree denotes the test of the specific attribute, and the instance is classified by moving down the tree branch from this node according to the value of the attribute in the given example. 

\noindent\textbf{Learning Algorithms.} Most algorithms learn the decision trees in a top-town, greedy search manner such as ID3~\cite{DBLP:journals/ml/Quinlan86} and its successor C4.5~\cite{DBLP:books/mk/Quinlan93}. 
The basic algorithm, ID3, will run a {\it statistical test} on choosing the instance attribute to determine how well it could classify the data points. 
From the root node, the algorithm will find the best attribute to form branches and then put all the training examples into the corresponding child nodes. 
It then repeats this entire process using the training data associated with the child nodes to select the appropriate attribute and value for the current node and form new branches from the child nodes. 
%This greedy search algorithm never backtrack to modify the previous decisions. 

\noindent\textbf{Information Gain-based Statistical Test.} There are several strategies for the statistical test in each step. One of the most popular tests is \emph{information gain}, which measures how well a given attribute could separate the training examples. 
Before giving the precise definition of information gain, we need to give the definition of \emph{entropy} first. 
Given a data collection $S$, containing positive and negative examples, the entropy of $S$ is:
\begin{equation}
    Entropy(S) = -p_{+}\log_2 p_{+} - p_- \log_2 p_-
\end{equation}
where $p_+$ and $p_-$ are the proportion of positive and negative examples in $S$, respectively. 

Next, we give the formulation of the information gain of an attribute $A$ with split value $v$, relative to a collection of examples $S$:
\begin{equation}
\label{greedy_objective_dt}
    Gain(S, A, v) = Entropy(S) - \sum_{b\in Branches}\frac{|S_b|}{|S|}Entropy(S_b)
\end{equation}
where \emph{Branches} contains two branches, each of which has the training examples with attribute $A$ smaller or larger than the value $v$, respectively. $S_b$ refers to the collection of examples from branch $b$.
The learning algorithm iteratively splits nodes and forms branches by maximizing Eq.~\ref{greedy_objective_dt} at each step.

Learning the decision tree in this way is equivalent to maximizing the global objective:

\begin{equation}
\label{global_objective_dt}
    \max \sum_{l \in L} n_l (1 - Entropy(S_l))
\end{equation}

where $S_l$ represents the collection of training examples in the leaf node $l$ and $n_l$ represents the number of examples falling into node $l$.

\section{Rule-based Summarization and Interpretation}
\label{sec.rule}
\begin{sloppypar}

In this section, we first give the definition of \textbf{rule} and then explain why rules are good at summarizing and interpreting outlier detection results. 

\begin{definition}
\label{def.rule}
Given a data set $\mathbb{D}$ in a N-dimensional feature space $\mathit{[ x_1, x_2, \cdots, x_N ]}$, a \textbf{Rule} $R_i$ is defined as $\mathit{R_i = (a_1 \leq x^i_1 \leq b_1)}$ $\land$ $\mathit{(a_2 \leq x^i_2 \leq b_2)}$, $\cdots$, $\land$ $\mathit{(a_j \leq x^i_j \leq b_j)}$, $\cdots$, $\land$ $\mathit{(a_L \leq x^i_L \leq b_L)}$. 
$\forall$ clause $\mathit{(a_j \leq x^i_j \leq b_j)}$ of $R_i$, $x^i_j$ corresponds to one attribute $x_j$ $\in$ $\{x_1, x_2, \cdots, x_N \}$; $a_j$ and $b_j$ ($a_j$ $<$ $b_j$) fall in the domain range of attribute $x_j$. 
$L$ indicates the number of attributes in rule $R_i$, or the \textbf{length} of $R_i$.  
\end{definition}

By Def.~\ref{def.rule}, a rule $R_i$ corresponds to a conjunction of domain value intervals, each with respect to some attribute $x_j$.
Rule $R_i$ covers a data subset $\mathbb{D}_i$ $\subseteq$, where $\forall$ object $d_i$ $\in$ $\mathbb{D}_i$, the attributes of object $d_i$ fall into the corresponding interval.   

In the decision tree model~\cite{97458}, each branch corresponds to one rule. 
Fig.~\ref{fig:exp_of_rule} shows a toy decision tree $T_i$ learned from a 2-dimensional data set $\mathbb{D}$.
$T_i$ classifies the objects in $\mathbb{D}$ into outliers and inliers.  
It has three branches, corresponding to 3 rules: $\mathit{R_1 = (-2 \leq x_1 \leq 2)}$, $\mathit{R_2 = (x_1 > 2)}$, and $\mathit{R_3 = (x_1 < -2)}$. 
All rules only contain one attribute $x_1$. Rules $R_2$ and $R_3$ are lower bounded or upper bounded only.
Thus the length of these rules is one.

Note the length of a rule is not equivalent to the depth of the tree. The depth of the decision tree in Figure~\ref{fig:exp_of_decision_tree} is two, while the lengths of the three rules are all one. 
Even if the decision tree gets deeper, the lengths of the rules could still be small. 
This is because a decision tree could use one attribute multiple times on one single branch (rule).

%since the attributes could be used again in the following splitting, the rules could still be simple. Consequently, the depth of decision tree could not directly reflect the number of rules and the complexity of rules directly, and the definition of length is necessary. 

%Note that  

These rules classify the whole data set into three different partitions.
Rule $R_1$ covers all inliers in $\mathbb{D}$, while both $R_2$ and $R_3$ represent outliers.  

Rules effectively summarize and interpret the outliers and inliers in the data. The merit is twofold. 
First, each rule covers a set of inliers or outliers. Therefore, rather than exhaustively evaluating the large number of outliers or inliers one by one, the users now only have to evaluate a small number of rules, thus saving huge amount of human efforts.
Second, the rules are human interpretable, helping the users easily understand why an object is considered as outlier or inlier and identify the root cause of the outliers. For example, rules $R_2$ and $R_3$ intuitively tell users that some objects are abnormal because their $x_1$ values are too large or too small.   

% \todo{Use PDF, not png}
\begin{figure}
\centering
\vspace{-0.2cm}
\subfigure[Rules]{\label{fig:exp_of_rule}\includegraphics[width=0.5\linewidth]{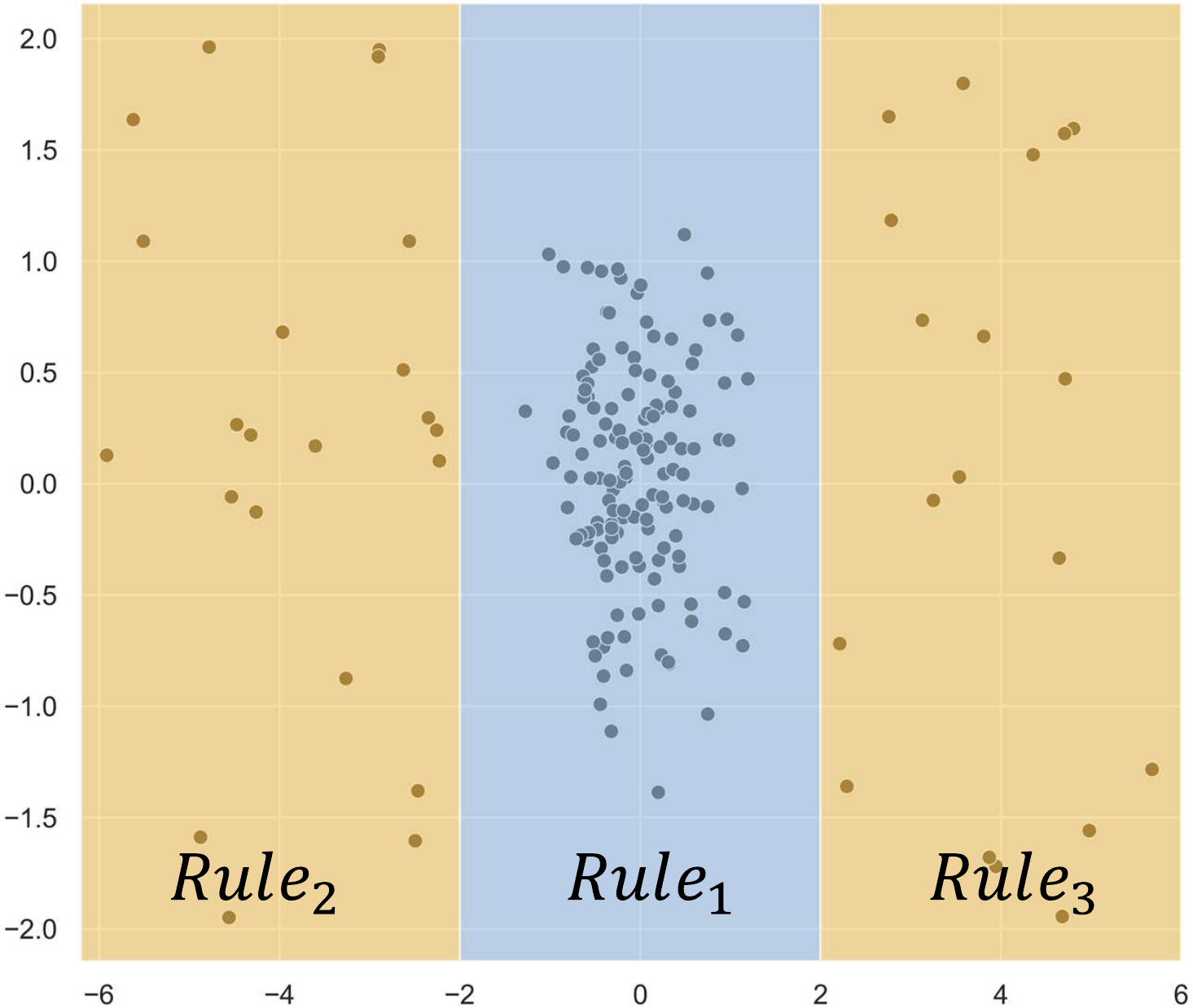}}
\hfill
\subfigure[Decision Tree]{\label{fig:exp_of_decision_tree}\includegraphics[width=0.45\linewidth]{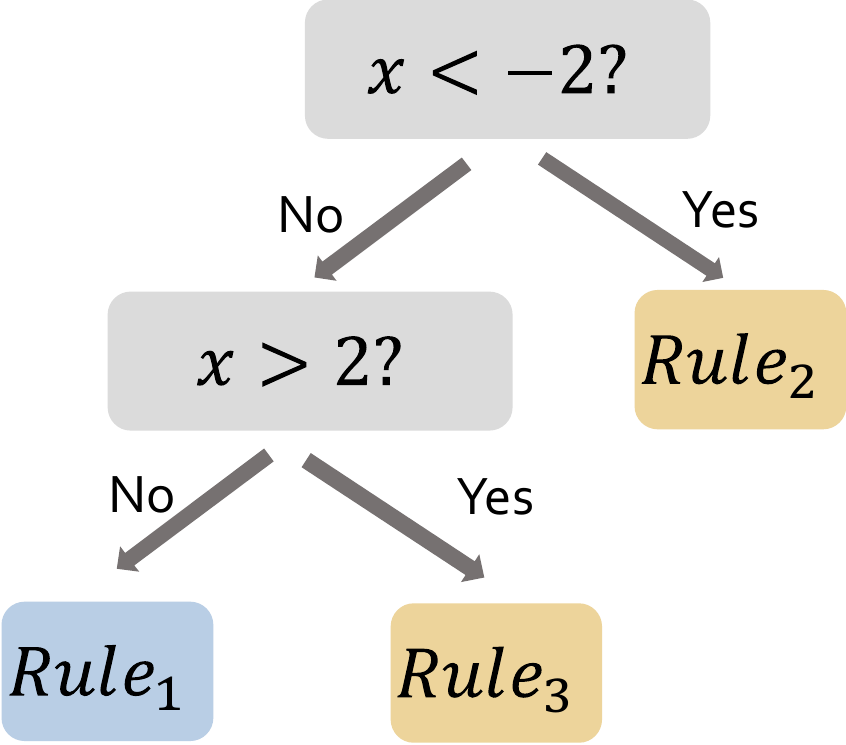}}
\vspace{-0.3cm}
\caption{Example of rules and decision tree. The blue partition covered by $Rule_1$ represents inliers, while the brown partitions covered by $Rule_2$ and $Rule_3$ represent outliers.}
\vspace{-0.2cm}
\label{fig:Example_of_rules_and_decision_trees}
\end{figure}

\end{sloppypar}

% It seems to be a simple solution, however, we might have omitted one important aspect: in order to cover the whole dataset with a single decision tree. Some of the rules generated by this tree might be too complicated for human to evaluate. There is no constraint on the complexity of the rules in conventional decision tree algorithms. 

% We aim to give a new formulation with global objective here which is equivalent to the conventional objective for each step. In the normal scenario, decision tree is the greedy algorithm, which has the objective for each step:
% \begin{align}
%     \max  (n_1+n_2)*E - (n_1 * E_1 + n_2 * E_2) \\
%     E(p) = - p * \log p - (1-p) * \log (1-p)
% \end{align}

% where $E$ is the entropy for current node, and $E_1, E_2$ are the entropy for two new nodes. $n_1, n_2$ are the numbers of points for each node. This objective for each step could also be summarized into the global objective:
% \begin{equation}
%     \max \sum_{r\in \mathcal{R}} n_r (1 - E(r))
% \end{equation}

\section{The optimization objective of Rules Generation}
\label{sec.objective}

\begin{sloppypar}

\subsection{The Insufficiency of Classic Decision Trees}
\label{sec.objective.classic}
Intuitively, to produce rules effectively summarizing and interpreting the outlier detection results, we could directly apply the classical decision tree algorithms such as ID3~\cite{DBLP:conf/lopal/ElaidiBA18}. That is, we use the output of the outlier detection method as ground truth labels to train a decision tree model and then extract rules from the learned decision tree.  

However, decision tree algorithms target producing rules that maximize the classification accuracy. The rules learned in this way do not necessarily have the desired properties when used in outlier summarization and interpretation, for the following reasons: 

First, they may produce rules that contain many attributes and thus are too complicated for humans to evaluate. For examples, humans can easily understand and reason on a rule with a couple of attributes such as the rules in Fig.~\ref{fig:Example_of_rules_and_decision_trees}, while it will be much harder for the humans to obtain any meaningful information from a complicated rule with many attributes. 
For instance, the rule with 20 attributes $a_1\leq x_1 \leq b_1, \cdots, a_{20} \leq x_{20} \leq b_{20}$ will be almost impossible for human to understand.

Second, to maximize the classification accuracy they may produce many rules. However, to reduce the human evaluation efforts, ideally we want to produce as few rules as possible.  

The above situations could happen when handling highly complex data sets which often require a {\it deep} tree with {\it many} branches.

\subsection{Summarization and Interpretation-aware Objective}
\label{sec.objective.optimized}

To address the above concerns, we design an optimization objective customized to outlier summarization and interpretation. It targets producing the minimal number of rules that are as simple as possible, while still guaranteeing the classification accuracy. The objective is composed of two sub-objectives, namely {\it length objective} and {\it entropy objective}.

\noindent\textbf{Length Objective.} To minimize the number of the rules as well as bounding the complexity of each rule, we first introduce an objective with respect to the lengths of the rules in rule set $\mathcal{R}$: 
    
\begin{align}
\label{eq.length_objective}
    \min_{\mathcal{R}} \mathcal{L}(\mathcal{R}), & \textrm{ where } \mathcal{L}(\mathcal{R}) =  \sum_{r_i\in\mathcal{R}} L(r_i)  \\
    \textrm{ s.t. } & L(r_i) \leq L_m \nonumber   
\end{align}

In Eq.~\ref{eq.length_objective}, $\mathcal{R}$ denotes a rule set. $L(r_i)$ denotes the length of a rule $r_i$ in $\mathcal{R}$. $L_m$ is the predefined maximal length of each rule that the users allow. 
Essentially, the total length of all rules represent the complexity of the learn model.
Minimizing it will effectively reduce the number of rules, while at the same time simplifying each rule.

\noindent\textbf{Entropy Objective.} To maximize the classification accuracy of the derived model, we adopt the entropy-based optimization objective from the classical decision tree algorithms~\cite{DBLP:conf/lopal/ElaidiBA18}, i.e. ID3 and C4.5, as illustrated in Sec.~\ref{sec.preliminary}.

\begin{equation}
\label{eq.entropy_objective}
     \max_{\mathcal{R}} \mathcal{S}(\mathcal{R}), \textrm{ where } \mathcal{S}(\mathcal{R}) = \sum_{r_i \in \mathcal{R}} n_{r_i} E(r_i) 
\end{equation}

In Equation~\ref{eq.entropy_objective}, $E(r_i)$ corresponds to $\mathit{1 - Entropy(r)}$. Maximizing Eq.~\ref{eq.entropy_objective} effectively maximizes the classification accuracy.   

Combing Eq.~\ref{eq.entropy_objective}) and Eq.~\ref{eq.length_objective}), our summarization and interpretation-aware objective (Eq.~\ref{eq:formulation}) maximizes the classification accuracy, while at the same time minimizing the total length of the rules.   

\begin{align}
\label{eq:basicformulation}
    \max_{\mathcal{R}} \mS (\mathcal{R}) = \frac{\sum_{r_i\in\mathcal{R}}n_{r_i} E(r_i)}{\sum_{r_i \in\mathcal{R}} L(r_i)} \\
    \textrm{ s.t. } L(r_i) \leq L_m\nonumber, F1(\mathcal{R}) > F1_m 
\end{align}
where $L_m$ corresponds to the {\it maximal} length of a rule that the users allow, while $F1_m$ is a predefined requirement on classification accuracy which is measured by F1 score in the case of outlier detection.  
\noindent\textbf{Optimization Issue.} However, in practice we observed that this objective caused issues in the optimization process. Maximizing the entropy objective typically will lead to more complex rules and in turn the increase of the length objective. However, the length objective often increases faster than the entropy objective. Therefore, the overall objective (Eq.~\ref{eq:basicformulation}) tends to stop increasing in a few iterations. 

\noindent\textbf{Final Objective: Introducing a Stabilizer.}
To solve this problem, we introduce a stabilizer $M$ into the length objective -- the denominator of Eq.~\ref{eq:basicformulation}: 

\begin{align}
\label{eq:formulation}
    \max_{\mathcal{R}, M} \mS (\mathcal{R, M}) = \frac{\sum_{r_i\in\mathcal{R}}n_{r_i} E(r_i)}{\sum_{r_i \in\mathcal{R}} L(r_i) + M} \\
    \textrm{ s.t. } L(r_i) \leq L_m\nonumber, F1(\mathcal{R}) > F1_m
\end{align}

% \begin{equation}
%     \max_{\mathcal{R}} \mS(\mathcal{R}), \textrm{ where } \mS(\mathcal{R}) =  \frac{\sum_{r\in\mathcal{R}}n_r E(r)}{\sum_{r\in\mathcal{R}} L(r) + M}
% \end{equation}

The stabilizer $M$ mitigates the impact of the quickly increasing length objective. 
It ensures that the length objective does not dominate our summarization and interpretation-aware objective.
Intuitively, in the extreme case of setting $M$ to an infinite large value, the increase of the total rule length is negligible to the objective.
Now maximizing Eq.~\ref{eq:formulation} in fact is equivalent to the traditional entropy-based decision tree.

\noindent\textbf{Auto-learning Stabilizer M.}
An appropriate value of $M$ is critical to the quality of the learned rules. However, relying on the users to manually tune it is difficult. 
First, $M$ can be any positive value and thus has infinite number of options.
Second, ideally $M$ should dynamically change to best fit the evolving rule set produced in the iterative learning process.
Therefore, rather than make it a hyper-parameter, $M$ is a learnable parameter in our objective function Eq.~\ref{eq:formulation}.

%F_{1m}$ is the predefined threshold to measure if the rules are good enough to break the training process. 

\end{sloppypar}

\section{\sys: Rule Generation Method}
\label{sec.ruleGen}

\begin{sloppypar}

This section introduces our \underline{S}ummariza\underline{T}ion \underline{A}nd \underline{I}nterpretation-aware \underline{R}ule generation method (\sys).
Similar to the classic decision tree algorithms~\cite{DBLP:conf/lopal/ElaidiBA18}, \sys produces a rule set by iteratively splitting the decision node. 
We prove that in each iteration \sys is {\it optimal} in maximizing our objective in Eq.~\ref{eq:formulation}.

% We use $\mathcal{R}$ to denote a rule set, which includes all the leaf nodes of one decision tree. For the conventional entropy-based decision tree algorithm (i.e. ID3), the objective in each iteration for splitting the node that is equivalent to rule $r$ could be written as:
% \begin{equation}\label{decision_tree_greedy_objective}
%     \max N(r) * Entropy(r) - (N(r_1) * Entropy(r_1) + N(r_2) * Entropy(r_2))
% \end{equation}
% where $r_1, r_2$ are the rules split from $r$, which also corresponding to the child nodes of the node corresponding to $r$. Globally, we can upgrade the objective in Eq.(\ref{decision_tree_greedy_objective}) as follows: 

%$F_1(\mathcal{R})$ is calculated using the predicted labels from the rules against the labels given by the unsupervised methods. Note that the number $M$ in the denominator is also what we aim to minimize, since the smaller $M$ is, the less it will affect the minimization of $\sum_{r\in\mathcal{R}} L(r)$. 

Below we first give the overall process of \sys:

\begin{enumerate}[leftmargin=*]
    \item Initialize the stabilizer $M$ in Eq.~\ref{eq:formulation} to zero; 
    \item Increase the value of $M$;
    \item Find a node to split that could increase the objective in Eq.~\ref{eq:formulation}; go to step 2. 
\end{enumerate}

%and iteratively splitting the node until no possible splitting can the objective in Eq.~\ref{eq:formulation} could not be larger with any possible splitting

In short, \sys iteratively increases the value of $M$ and splits the nodes. 
Next, we first show that the value of $M$ is critical to the performance of \sys and then introduce a method to calculate the optimal value of $M$ at each iteration.   

\subsection{The Value of M Matters}
\label{sec.ruleGen.mValue}
Given a rule set $\mathcal{R}$, splitting a node $n$ is equivalent to dividing one rule $r$ in $\mathcal{R}$ into two rules $r_1$ and $r_2$, where $r_1$ and $r_2$ end at the two child nodes of node $n$ correspondingly. 
Given an $M$ and a rule set $\mathcal{R}$, we say a split $\mathit{sp(\mathcal{R}, M)}$ is \textbf{valid} if $\mathit{\mS(\mathcal{R}\backslash \{r\} \cup \{r_1, r_2\}, M) > \mS(\mathcal{R}, M)}$.
That is, a valid split will increase the objective defined in Eq.~\ref{eq:formulation}.
For the ease of presentation, we use $\mathit{\mS(\mathcal{R}', M)}$ to denote $\mathit{\mS(\mathcal{R}\backslash \{r\} \cup \{r_1, r_2\}, M)}$

Next, we show the smallest $M$ that could produce a valid split is optimal in maximizing Eq.~\ref{eq:formulation}. 
\begin{theorem}
\label{theory.optimal}
\noindent\textbf{Monotonicity Theorem.}
Given a rule set $\mathcal{R}$, if $M_a$ $>$ $M_b$, then $\mathit{\mS(\mathcal{R}'_a, M_a)}$ is guaranteed to be \textbf{smaller} than $\mathit{\mS(\mathcal{R}'_b, M_b)}$, where $\mathit{\mathcal{R}'_a}$, $\mathcal{R}'_b$ denotes the rule set produced by a valid split on $\mathcal{R}$ that maximizes the objective given $M_a$ or $M_b$.
\end{theorem}

\begin{proof}
Because $M_a > M_b$, we have:
\begin{align}
\label{eq1}
    S(\mathcal{R}'_a, M_a) &= \frac{\sum_{r\in\mathcal{R}'_a} n_r E(r)}{\sum_{r\in\mathcal{R'}_a}L(r) + M_a} \nonumber \\
    & < \frac{\sum_{r\in\mathcal{R}'_a} n_r E(r)}{\sum_{r\in\mathcal{R}'_a}L(r) + M_b} = S(\mathcal{R}'_a, M_b)
\end{align}

Because $\mathcal{R}'_b$ corresponds to the best split given $M_b$, we obtain: 
\begin{equation}
\label{eq2}
    S(\mathcal{R}'_a, M_b) \leq S(\mathcal{R}'_b, M_b)
\end{equation}

From Eq.~\ref{eq1} and Eq.~\ref{eq2}, we have:
\begin{equation} \label{optimal_M}
    \mS(\mathcal{R}'_a, M_a) < \mS(\mathcal{R}'_b, M_b)
\end{equation}
This concludes our proof.
\end{proof}

\subsection{Calculating the Optimal $M$}
\label{sec.ruleGen.cal}

By Theorem~\ref{theory.optimal}, to maximize the objective at each iteration, it is necessary to search for the smallest value of $M$ that could produce a valid split.
Intuitively we could find the optimal $M$ by gradually increasing the value of $M$ at a fixed step size. However, this is neither effective nor efficient, because it is hard to set an appropriate step size.
If it is too large, \sys might miss the optimal $M$. 
On the other hand, if the step size is too small, \sys risks to incur many unnecessary iterations not producing any valid splits.   

To solve the above problem, we introduce a method which uses the concept of {\it boundary stabilizer} to directly calculate the optimal $M$. Moreover, the best splitting is discovered as the by-product of this step. 

We use $M_o$ to denote the optimal $M$. Because $M_o$ is the smallest $M$ that could produce a valid split, then $\forall r_0$ and $\forall r_1, r_2$, where $r_1$ and $r_2$ represent the rules produced by splitting rule $r_0$, Eq.~\ref{inequality_of_M} holds: 

\begin{equation}
\label{inequality_of_M}
    \mS(\mathcal{R}, M) > \mS(\mathcal{R} \backslash \{r_0\} \cup \{r_1, r_2\}, M), \forall M < M_o
\end{equation}

\noindent\textbf{Boundary Stabilizer M.} To compute $M_o$, we first define a boundary $M$ denoted as $M_b$ which makes Equation~\ref{equality_of_M} hold:

\begin{align}
\label{equality_of_M}
    \mS(\mathcal{R}, M_b) = \mS(\mathcal{R} \backslash \{r_0\} \cup \{r_1, r_2\}, M_b)
\end{align}

By Eq.~\ref{equality_of_M}, setting the $M$ to $M_b$ will produce a split that does not change the objective. 
That is, under $M_b$ no valid split will increase the objective. But there exists a split that does not decrease the objective. So $M_b$ is called the boundary $M$,

We then expand Eq.~\ref{equality_of_M} as follows: 
\begin{align}
    &\frac{\sum_{r\in \mathcal{R} \backslash \{r_0\}}n_r E(r) + n_{r_0}E(r_0)}{\sum_{r\in \mathcal{R}\backslash \{r_0\}}L(r) + L(r_0) + M_b} \nonumber \\
    &= \frac{\sum_{r\in \mathcal{R} \backslash \{r_0\}}n_r E(r) + n_{r_1}E(r_1) + n_{r_2}E(r_2)}{\sum_{r\in \mathcal{R}\backslash \{r_0\}}L(r) + L(r_1) + L(r_2) + M_b}
    \label{split_rule_ineq}
\end{align}

We define $A = \sum_{r\in\mR\backslash \{r_0\}} n_r E(r), B = \sum_{r\in \mR \backslash \{r_0\}} L(r)$, and $A_0 = \sum_{r\in\mR} n_r E(r), B_0 = \sum_{r\in \mR} L(r)$, then Eq.~\ref{split_rule_ineq} could be rewritten as:
\begin{equation}
    \frac{A + n_{r_0} E(r_0)}{B + L(r_0)+M_b} = \frac{A + n_{r_1}E(r_1) + n_{r_2}E(r_2)}{B + L(r_1) + L(r_2) + M_b}
\end{equation}

Then after some mathematical transformation, we obtain:
\begin{align}
\label{eq:expansion}
    M_b (&n_{r_1}E(r_1) + n_{r_2}E(r_2) - n_{r_0}E(r_0)) \nonumber \\
    & = n_{r_0}E(r_0)(L(r_1) + L(r_2)) + A (L(r_1) + L(r_2) - L(r_0)) \nonumber \\
    & - B (n_{r_1}E(r_1) + n_{r_2}E(r_2) - n_{r_0}E(r_0)) \nonumber \\
    & - (n_{r_1}E(r_1) + n_{r_2} E(r_2))L(r_0) 
\end{align}

Denoting $\mathit{\Delta L = L(r_1) + L(r_2) - L(r_0)}$ and $\mathit{\Delta E = n_{r_1}E(r_1) + n_{r_2}E(r_2) - n_{r_0}E(r_0)}$, we simplify Eq.~\ref{eq:expansion} to: 

\begin{align}
\label{final_ineq}
    M_b \Delta E &= n_{r_0}E(r_0)(L(r_1) + L(r_2)) + A \Delta L - B \Delta E \nonumber \\
    & - (n_{r_1}E(r_1) + n_{r_2} E(r_2))L(r_0) \nonumber \\
    &= A\Delta L- B\Delta E + n_{r_0}E(r_0) \Delta L - L(r_0) \Delta E \nonumber \\
    &= (A + n_{r_0}E(r_0)) \Delta L - (B + L(r_0)) \Delta E \\
    M_b &= A_0 \frac{\Delta L}{\Delta E} - B_0, \forall r_0 \in \mR, \forall r_1, r_2 
\end{align}

% According to Eq.\ref{equality_of_M} and Eq.\ref{split_rule_ineq}, we have 
$\forall M > M_b$, with the same $r_0$ and $r_1, r_2$ in Eq.~\ref{equality_of_M}, 
% according to Eq.~\ref{equality_of_M}, Eq.~\ref{split_rule_ineq}, Eq.~\ref{eq:expansion}, 
% Eq.~\ref{final_ineq},
Eq.~\ref{eq:expansion} becomes: 
\begin{align}
\label{eq:expansion_for_M}
    M (&n_{r_1}E(r_1) + n_{r_2}E(r_2) - n_{r_0}E(r_0)) \nonumber \\
    & > n_{r_0}E(r_0)(L(r_1) + L(r_2)) + A (L(r_1) + L(r_2) - L(r_0)) \nonumber \\
    & - B (n_{r_1}E(r_1) + n_{r_2}E(r_2) - n_{r_0}E(r_0)) \nonumber \\
    & - (n_{r_1}E(r_1) + n_{r_2} E(r_2))L(r_0) 
\end{align}
Note that with expanding $r_0$ to $r_1, r_2$, the entropy of the rules must be lower, which means $n_{r_1}E(r_1) + n_{r_2}E(r_2) - n_{r_0}E(r_0) > 0$. 

Then from Eq.~\ref{eq:expansion} to Eq.~\ref{equality_of_M}, we may easily obtain Eq.~\ref{eq.split} from Eq.~\ref{eq:expansion_for_M}:
\begin{align}
\label{eq.split}
    \mS(\mathcal{R}, M) < \mS(\mathcal{R} \backslash \{r_0\} \cup \{r_1, r_2\}, M), \forall M > M_b
\end{align}

% Setting $M$ to be a value larger than $M_b$ will make Eq.~\ref{eq.split} hold:

% \begin{align}
% \label{eq.split}
%     \mS(\mathcal{R}, M) < \mS(\mathcal{R} \backslash \{r_0\} \cup \{r_1, r_2\}, M), \forall M > M_b
% \end{align}

% \todo{Probably need a proof here to show why Eq.~\ref{eq.split} holds.}

That is, an $M$ larger than $M_b$ is guaranteed to produce a valid split -- splitting rule $r_0$ to $r_1$ and $r_2$.

\noindent\textbf{Calculating Optimal M.} According to the Monotonicity theorem (Theorem~\ref{theory.optimal}), a smallest $M$ is the best in maximizing the objective. Therefore, \sys can directly calculate $M_o$ using Eq.~\ref{eq.optimal}:

\begin{equation}
\label{eq.optimal}
    M_o > \min_{\Delta L / \Delta E} A_0 \frac{\Delta L }{\Delta E} - B_0, \forall r_0 \in \mR, \forall r_1, r_2
\end{equation}

That is, \sys first finds a rule $r_0$ from $\mathcal{R}$ that after split into two rules, produces the smallest $\frac{\Delta L}{\Delta E}$. \sys then sets $M_o$ as a value larger than $\frac{\Delta L}{\Delta E}$ - $B_0$.
In this way, \sys successfully calculates the optimal $M$ and finds the best split in one step, making its learning process effective yet efficient.

\subsection{\sys Learning Algorithm}
\label{sec.ruleGen.training}

\begin{algorithm}[t]
  % \SetKwInOut{Input}{Input}\SetKwInOut{Output}{Output}
  \KwIn{Training data $X$, $F1$ score threshold $F1_m$}
  \KwOut{The target rule set.}
  Initialize $M$ to be zero\;
  Initialize the min heap $H$ to contain only the root node\; 
  Set the rule set $\mathcal{R}=\{r_0\}$\; 
  $A_0=nE(r_0), B_0 = 0$\;
  \While{True}{
    extract from $H$ a rule $r_0$ which has the minimal $\frac{\Delta L}{\Delta E}$ \; 
    % \uIf{$M < A_0 (\frac{\Delta L}{\Delta E})_r - B_0$}{Set $M = A_0 (\frac{\Delta L}{\Delta E})_r - B_0$\;}
    Set $M = A_0 (\frac{\Delta L}{\Delta E})_{r_0} - B_0$\;
    
    \While{the minimal $(\frac{\Delta L}{\Delta E})_{r_0}$ from $H$ $\leq \frac{M + B_0}{A_0}$}{
        Extract rule $r_0$ with the minimal $(\frac{\Delta L}{\Delta E})_{r_0}$ from $H$\; 
        Split $r_0$ into $r_1, r_2$ \; 
        Insert $r_1, r_2$ into $\mathcal{R}$ and $H$;
        Maintain the heap $H$ according to $(\frac{\Delta L}{\Delta E})_{r_1}$ and $(\frac{\Delta L}{\Delta E})_{r_2}$\;
        $A_0 \leftarrow A_0 + E(r_1) + E(r_2) - E(r)$\;
        $B_0 \leftarrow B_0 + L(r_1) + L(r_2) - L(r)$\;
    }
    Calculate the $F1$-score of the current rule set as $F1(\mR)$\;
    \uIf{$F1(\mR) > F1_{m}$}{
    Break\;}
  }
  \caption{Learning Algorithm of \sys}
    \label{alg:stair}
\end{algorithm}

Algorithm~\ref{alg:stair} shows the learning process of \sys.
It starts with initializing $M$ as 0 (Line 1) and uses a min heap structure $H$ to keep all nodes. 
Similar to the decision tree algorithms, it initializes $H$ to contain only the root node (Line 2). 
It then sets the rule set $\mathcal{R}$ to contain only one rule $r_0$ corresponding to the root node (Line 3).
By default, rule $r_0$ classifies all training samples as inliers.
Then based on Eq.~\ref{eq.optimal}, \sys iteratively extracts a rule $r_0$, calculates $(\frac{\Delta L}{\Delta E})$, updates M, and splits $r_0$ into two rules $r_1$ and $r_2$. 
After each split, it calculates $(\frac{\Delta L}{\Delta E})$ with respect to $r_1$/$r_2$, refreshes the rule set $\mathcal{R}$ and min heap $H$, and updates $A_0$ and $B_0$ accordingly. 
The learning process will terminate when the following conditions hold: 
(1) the accuracy reaches the requirement specified by users; and
(2) the $\mathit{\mS(\mathcal{R}, M)}$ does not increase in a few iterations. 

% In this algorithm, $(\frac{\Delta L}{\Delta E})_r$ is defined as follows:
% \begin{equation}
%   (\frac{\Delta L}{\Delta E})_r = \min_{r_1, r_2} \frac{L(r_1) + L(r_2) - L(r)}{E(r_1) + E(r_2) - E(r)}
% \end{equation}
% where $r_1, r_2$ are split rules from $r$. 

\noindent\textbf{Complexity Analysis.}
Compared to the classical decision tree algorithms, the additional overhead that \sys introduces is negligible. 
In each iteration, \sys extracts the rule $r_0$ from min heap $H$ and inserts into $H$ the new rules. 
Assume there are $n$ nodes in the tree. Because the complexity of min heap's retrieve and insert operations is $O(\log n)$, the additional complexity is $O(n \log n)$.
\end{sloppypar}

% From the algorithm, we may find that the value of $M$ does not actually matter in the splitting process as long as we find the minimal $\frac{\Delta L}{\Delta E}$ for each step. However, the variable $M$ is still the most important term in Eq.(\ref{eq:formulation}). Without $M$, our objective would be maximizing $ \frac{\sum_{r\in \mR} n_r E(r)}{\sum_{r\in \mR} L(r)}$. Nevertheless, searching for the minimal $\frac{\Delta L}{\Delta E}$ would not be equivalent to maximizing $\frac{\sum_{r\in \mR} n_r E(r)}{\sum_{r\in \mR} L(r)}$. Thus with the appearance of $M$, this process will make much more sense. 

\section{Localized \sys: Data Partitioning \& Rule Generation}
\label{sec.localized}

\begin{sloppypar}

\begin{figure}[ht]
    \centering
    \includegraphics[width=\linewidth]{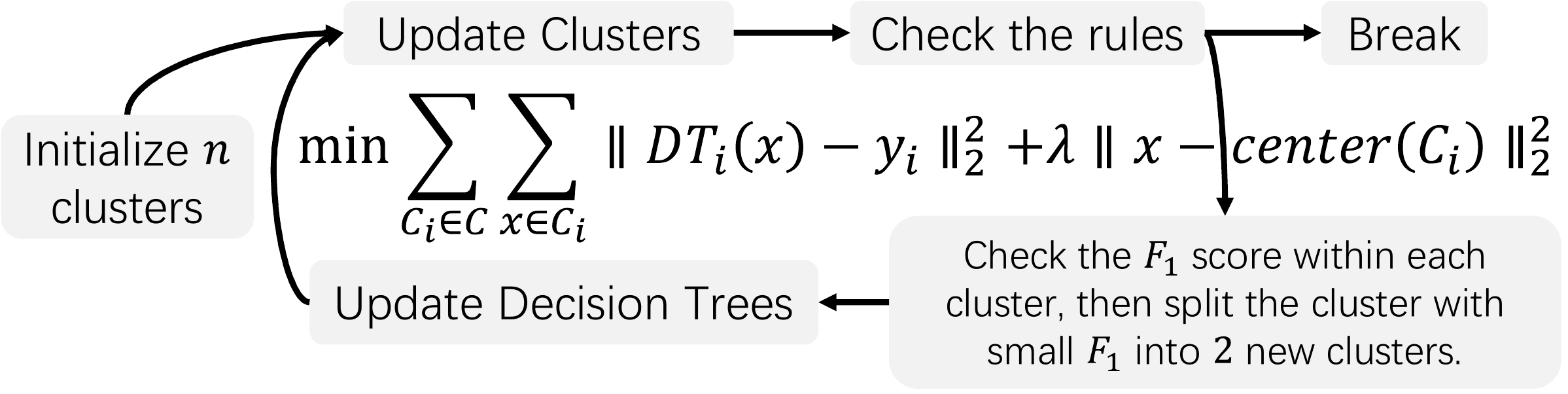}
    \caption{The Localized \sys}
    \label{fig:procedure}
\end{figure}

As shown in our experiments (Sec.~\ref{sec.exp}), although in general \sys performs much better than the classical decision tree algorithms in producing summarization and interpretation friendly rules, its performance degrades quickly on high dimensional, highly complex data sets, for example on the \emph{SpamBase} data set which has 57 attributes.
This is because a single decision tree with a small number of simple rules is not powerful enough to model the complex distribution properties underlying these data sets. 

To solve this problem, we propose a {\it localized} \sys approach, so called L-\sys. 
L-\sys divides the whole data set into multiple partitions and learns a tree model for each partition. 
Taking the data locality into consideration, L-\sys produces data partitions where the data in each partition share the similar statistical properties, while different partitions show distinct properties. 
L-\sys thus is able to produce localized, simple rules that effective summarize and explain each data partition. 

Next, we first introduce the objective of L-\sys in Sec.~\ref{sec.localized.objective} and then give the learning algorithm in Sec.~\ref{sec.localized.algorithm}.

\subsection{Joint Optimization of Data Partitioning and Rule Generalization}
\label{sec.localized.objective}

Intuitively, L-\sys could produce the localized rules in two disjoint steps: (1) partitioning data using the existing clustering algorithms such as k-means~\cite{1017616} or density-based clustering~\cite{ester1996density}; (2) directly applying \sys on each data partition one by one.  
However, this two steps solution is sub-optimal in satisfying our objective, namely producing minimal number of interpretable rules that are as simple as possible to summarize the outlier detection results.
This is because the problems of data partitioning and rule generation are highly dependent on each other.
Clearly, rule generation relies on data partitioning. To generate localized rules, the data has to be partitioned first.
However, on the other hand, without taking the objective of rule generation into consideration, clustering algorithm does not necessarily yield data partitions that are easy to summarize with simple thus interpretable rules. 
Therefore, L-\sys solves the two sub-problems of data partitioning and rule generation jointly. 

To achieve this goal, in addition to the summarization and interpretation-aware objective (Eq.~\ref{eq:basicformulation}) defined in Sec.~\ref{sec.objective.optimized}, L-\sys introduces a partitioning objective composed of {\it error objective} and {\it locality objective.} 

\noindent\textbf{Error Objective.} We denote the partitions of a dataset as $\mC = \{C_i\}_{i=1}^{n}$, where $n$ is the number of partitions and $C_i$ represents the $i$th partition. 
$DT_i$ denotes the decision tree learned for a data partition $C_i$. 
Decision tree $DT_i$ produces a prediction with respect to each object $x$ in data partition $C_i$, denoted as $\mathit{DT_i(x)}$. 

Next, in Eq.~\ref{error} we define an {\it error metric} to measure how good a decision tree $\mathit{DT_i}$ fits the data in $C_i$:

\begin{equation}\label{error}
    \sum_{x\in C_i} ||DT_i(x) - y_i||_2^2
\end{equation}

To ensure the classification accuracy, L-\sys targets minimizing this error metric with respect to all data partitions, which yields the {\bf error objective}:

\begin{equation}
\label{decision_tree_cluster}
    \min_{\mC} \sum_{C_i \in \mC}\sum_{x\in C_i} ||DT_i(x) - y||_2^2
\end{equation}
where $y$ indicates the ground truth label of object $x$.

% \begin{enumerate}[leftmargin=*]
%     \item When the process converges, the clusters generated by this algorithm may not be localized, which means this cluster is not easy to be represented by any of points in this cluster. 
%     \item When more than one trees could correctly classify one single point, which cluster should this point belong to? Randomly choosing one cluster will definitely lead to sub-optimal choices. 
% \end{enumerate}

\noindent\textbf{Locality Objective.}
Although using the above error objective to learn the data partitioning and the corresponding decision trees will effectively minimize the overall classification with respect to the whole dataset, the data partitions produced in this way do not preserve the locality of each data partition. 
Potentially one rule could cover a set of data objects that are scattered across the whole data space, thus is not amenable for human to understand. 
% \todo{Probably we want to give an example with figure here.}
As shown in Figure~\ref{fig:L-STAIR_vis}, when the locality is preserved, the rules are constrained within each cluster. 
This means there is no overlapping between the rules. Then the generated rules will be easier to understand.

\begin{figure}
    \centering
    \includegraphics[width=0.8\linewidth]{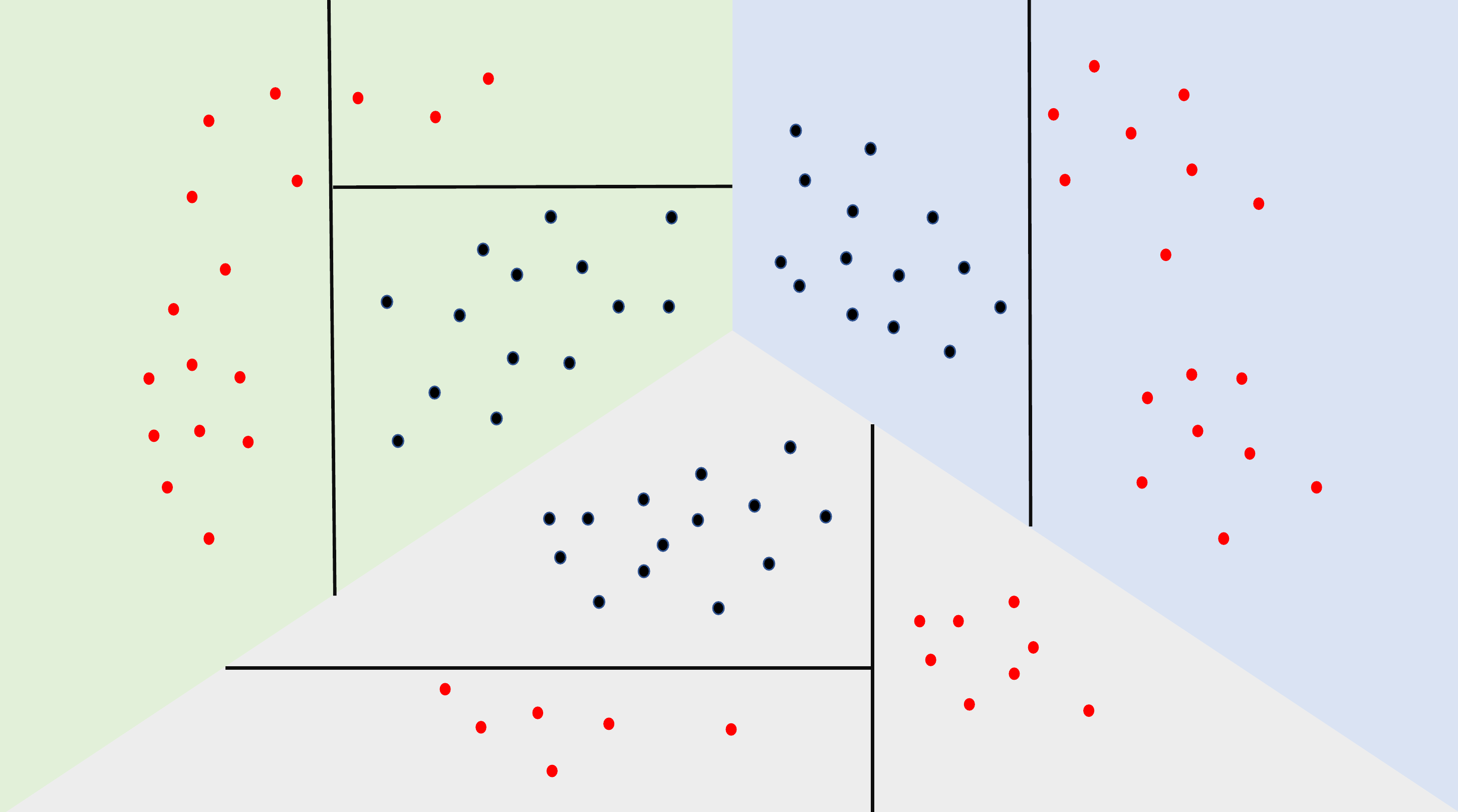}
    \caption{The intuition of locality: The black points are inliers and the red ones are outliers. The backgrounds with different colors refer to three clusters, while the straight lines in each cluster represent the rules. }
    \label{fig:L-STAIR_vis}
\end{figure}

Therefore, to ensure the data locality of each partition, we introduce the \textbf{locality objective}: 

\begin{equation}
\label{kmeans_objective}
    \min_{\mC} \sum_{C_i \in \mC} ||x - center(C_i)||_2^2
\end{equation}

Optimizing on the locality objective enforces the objects within each partition to be close to each other, similar to the objective of clustering such as k-means.

\noindent\textbf{The Final L-\sys Objective.}
Combining Eq.~\ref{kmeans_objective} and Eq.~\ref{decision_tree_cluster} together leads to the final partitioning objective: 
\begin{equation}
\label{objective_localized}
    \min_{\mC} \mL_{L-STAIR}(\mC) = \sum_{C_i \in \mC}\sum_{x\in C_i} ||DT_i(x) - y||_2^2 + \lambda ||x - center(C_i)||_2^2
\end{equation}

Eq.~\ref{objective_localized} uses $\lambda$ ($\mathit{0 < \lambda < 1}$) to balance these two objectives. Setting the $\lambda$ to a small value will give error objective higher priority. 

\subsection{L-\sys Learning Algorithm}
\label{sec.localized.algorithm}
L-\sys jointly optimizes the partitioning objective (Eq.~\ref{objective_localized}) and the summarization and interpretation-aware objective (Eq.~\ref{eq:basicformulation}) in an iterative manner. 
As shown in Figure~\ref{fig:LoMDT_general_routine}, L-\sys starts with initializing $n$ data partitions by using some clustering algorithms such as k-means in our implementation. 
Then we apply the \sys algorithm introduced in Sec.~\ref{sec.ruleGen.training} to learn one summarization and interpretation-aware decision tree for each initial partition.
Next, it iteratively updates the partitions and thereafter builds the decision trees correspondingly.
During this process, L-\sys dynamically modifies the number of partitions based on the classification accuracy of each individual decision tree, making the number of partitions self-adaptive to the data, as further illustrated in Sec.~\ref{sec.localized.partition}.
Similar to our original \sys approach, the learning process of L-\sys terminates after the overall classification accuracy with respect to the whole data set is above the threshold $F1_m$ and the optimization objective does not improve anymore in a few iterations.  

\begin{figure}[h!]
    \centering
    \includegraphics[width=\linewidth]{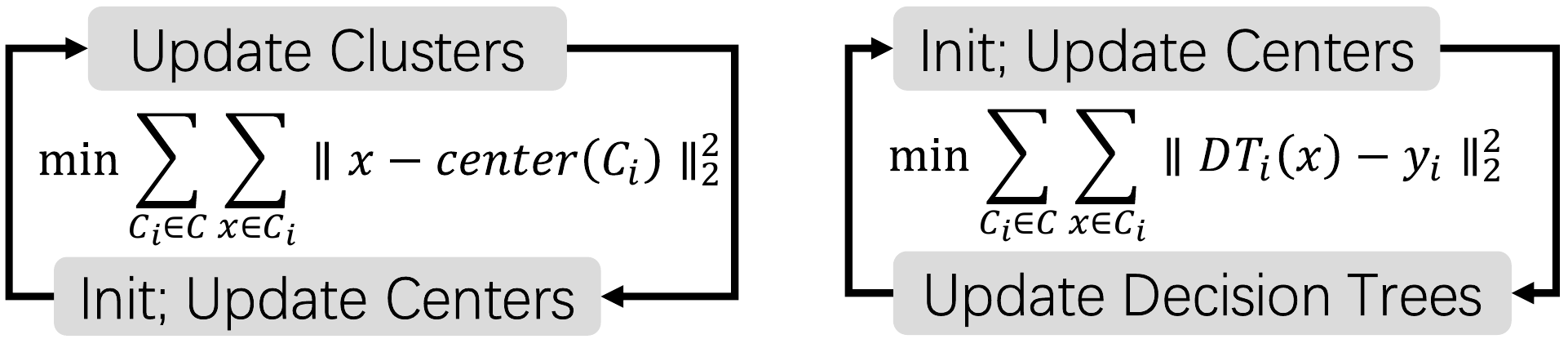}
    \caption{The correlation between our algorithm and the K-means algorithm}
    \label{fig:LoMDT_general_routine}
\end{figure}

\begin{algorithm}[h!]
  % \SetKwInOut{Input}{Input}\SetKwInOut{Output}{Output}
  \KwIn{Training data $X$, cluster number $n$ for initialization, $F_1$ score threshold $F_{1m}$}
  \KwOut{Clusters $\mC$ and decision tree for each cluster $DT_i$, $i\in\{1,\cdots,|\mC|\}$}
  Initialize $n$ clusters using K-means\;
  \While{True}{
     Build $n$ new decision trees $DT_i$, $i=\{1,\cdots,n\}$ using the algorithm introduced in Section \ref{sec.ruleGen.training} for $n$ clusters\;
     Update the clusters $\mC$ according to the objective Eq.(\ref{objective_localized})\;
     Remove empty clusters\;
     Calculate the $F_1$ score of the predictions made by the MDTs(the rules), and denote it as $f_1$\;
     \uIf{$f_1>F_{1m}$}{break\;}
     Check the $F_1$ score within each cluster, split each of the clusters with too small $F_1$ scores to $n$ new clusters using K-means algorithm.
  }
  \caption{L-\sys learning algorithm}
  \label{alg:LomDT}
\end{algorithm}

% As stated earlier, we may consider designing an algorithm similar to K-means to minimize this objective. 
% As shown in Algorithm \ref{alg:LomDT}, if we currently omit the steps in line 5-8, then the algorithm can be simply mapped to K-means algorithm. We also draw the training procedure in figure \ref{fig:procedure}. 
% In the step "updating decision trees", we are actually building MDTs for each cluster. The detailed parameter settings for the experiments could be found in Section \ref{study_of_the_number_of_clusters}. 

\vspace{-10pt}
\subsection{Dynamically Adjusting the Number of Partitions} 
\label{sec.localized.partition}

As shown in Algorithm~\ref{alg:LomDT}, L-\sys uses the hyperparmeter $n$ to specify the number of partitions and initialize each data partition accordingly. 
It is well known that in many clustering algorithms such as k-means the number of clusters is a critical hyper-parameter which determines the quality of data partitioning; and it is hard to tune in many cases~\cite{fu2020estimating}. 
L-\sys does not rely on an appropriate $n$ to achieve good performance, because L-\sys allows the users to set a small $n$ initially and then dynamically adjusts it in the learning process. 

% \noindent\textbf{Removing Partitions.}
% If the initial number of partitions is too large, L-\sys will automatically identify and eliminate the redundant ones. 
% Redundant partitions unnecessarily produce extra decision trees and in turn excess rules for users to evaluate.   

% L-\sys identifies the redundant partitions as those bearing large similarity to others such that merging them into other partitions do not degrade the partitioning objective. Then in the algorithm, after reassigning the data points into the partitions, we may obtain empty partitions, which could be seen as redundant since the data points in these partitions have been merged into other ones, and the obtained assignments have optimized the partitioning objective in the current step.  

% \todo{What do you mean by similar? How similar is similar enough? We need to quantify this. Technical article has to be very precise.}
% L-\sys identifies the redundant partitions as those bearing large similarity to others such that merging them into other partitions do not degrade the partitioning objective. 

%The similarity these redundant partitions holds could cause that the decision trees trained in other partitions could also correctly classify the points in these partitions and these redundant ones could contain no points after the reclustering step in line 4. Thus we can remove the redundant partitions through cancelling out the empty sets. 

% \todo{Again how to determine a partition is too complicated? Do you use F1 score as measurement? If so, make it clear.}

\noindent\textbf{Producing New Partitions.}
L-\sys will produce new partitions by splitting some partitions that are too complicated to summarize and explain with simple rules. The partition is said to be too complicated when the obtained $F1$-score on it is not good enough, more specifically lower than $F1_m$.
This indicates that simple rules could not fully explain this partition.
After identifying a complicated partition, L-\sys uses k-means again to split it into two partitions, then build one decision tree for each new partition. 

\noindent\textbf{Removing Partitions.}
L-\sys identifies the redundant partitions as those bearing large similarity to others such that merging them into other partitions do not degrade the partitioning objective. After identifying redundant partitions, L-\sys will discard them and reassign their data points to other partitions.

Our experiments (Table~\ref{tab:study_of_n}, Sec.~\ref{study_of_the_number_of_clusters}) on 10 datasets show that by starting with a small $n$ L-\sys is always able to produce good results.  

%As stated in line 4 and line 8 in Algorithm \ref{alg:LomDT}, if the number is too large, then the step in line 4 will remove redundant clusters; and if the number is too small, then the step in line 8 will introduce more clusters to interpret the dataset. 

%We have chosen $n=3$ in our implementation. We will discuss the experimental effect of this number in Section xxx. 

\subsection{Convergence Analysis}
\label{sec.localized.converge}
We theoretically show that L-\sys could converge. 
We establish this conclusion by showing that each step in Algorithm~\ref{alg:LomDT} would never make the objective larger. 

In Alg.~\ref{alg:LomDT} there are four steps which could potentially update the objective. We analyze each step one by one.

\textbf{Step 1} (Line 3): Given a decision tree corresponding to one specific partition, if its F-1-score is below $F1_m$, L-\sys will replace it with a new tree which has higher F-1. Therefore, the error objective Eq.(\ref{decision_tree_cluster}) has also been improved approximately. Then since the locality objective Eq. (\ref{kmeans_objective}) will not be affected by the current step, the whole objective Eq.(\ref{objective_localized}) will also get improved. 

\textbf{Step 2} (Line 4): Assume L-\sys reassigns data point $x_j$ which used to belong to partition $P_j$ to partition $P_j'$ when updating the partitioning according to the following formula:
\begin{equation}
    P_j' = arg\min_k \mL_j(DT_k, C_k) = ||DT_k(x_j) - y_j||_2^2 + \lambda ||x_j - center(C_k)||_2^2
\end{equation}

This leads to:

\begin{equation}
 \mL_j(DT_{P_j'}, C_{P_j'}) \leq \mL_j(DT_{P_j}, C_{P_j})
\end{equation}

Therefore, in the first four steps, L-\sys always gets the objective Eq.(\ref{objective_localized}) smaller and will converge eventually.

% \begin{equation}
%     ||DT_{P_j'}(x_j) - y||_2^2 + \lambda ||x_j - center(C_{P_j'}||_2^2 \leq ||DT_{P_j}(x_j) - y||_2^2 + \lambda ||x_j - center(C_{P_j}||_2^2
% \end{equation}

Similarly, denoting the existing partitioning as $\mC$ and the new partitioning as $\mC'$, by Eq.~\ref{objective_localized}, we get:
\begin{align}
    \mL_{L-STAIR}(\mC') &= \sum_{j=1}^n ||DT_{P_j'}(x_j) - y_j||_2^2 + \lambda ||x_j - center(C_{P_j'})|| \nonumber \\
    &= \sum_{j=1}^n \mL_j(DT_{P_j'}, C_{P_j'}) \leq \sum_{j=1}^n \mL_j(DT_{P_j}, C_{P_j})\nonumber \\
    &= \mL_{L-STAIR}(\mC)
\end{align}

Thus, this step gets the objective $\mL_{L-STAIR}$ smaller. 

\textbf{Step 3} (Line 5): By Eq.(\ref{objective_localized}), the empty set contributes nothing to the objective. Therefore, directly removing empty partitions has no impact to the objective. 

\textbf{Step 4} (Line 8): Assume L-\sys splits partition $C_s$ into $n$ new partitions $C_{si}, i\in\{1,\cdots, n\}$ and builds $n$ decision trees. Denoting the decision tree w.r.t $C_s$ and $C_{si}$ as $DT_s$ and $DT_{si}$, we have:

\begin{align}
\sum_{x,y \in C_s}||DT_s(x) - y||_2^2 &\geq \sum_{i=1}^n \sum_{x,y\in C_{si}} ||DT_{si}(x) - y||_2^2 \nonumber \\
\sum_{x \in
C_s} ||x - center(C_s)|| &>  \sum_{i=1}^n \sum_{x\in C_{si}} ||x - center(C_{si})||_2^2   
\end{align}

Because L-\sys always makes both the error objective and locality objective smaller after updating the decision trees, L-\sys is guaranteed to minimize the final L-\sys objective Eq.~\ref{objective_localized} and thus converge eventually. 

% we could find that the minimization process of ours is very similar to that of K-means, except for these several differences: (1) We will initialize with clusters, rather than centers. Since we cannot build decision trees with no clusters and no data. (2) We will update decision trees with new clusters, while K-means need to update centers with these clusters. 

% \todo{We need theoretically show why it can converge probably by borrowing the proof from k-means. In the analysis, you can use the similarity between the k-means and our method.}
% It's easy to see that updating decision trees and clusters with this iterative strategy could ultimately lead to convergence, since in each step, the objective in Eq.(\ref{decision_tree_cluster}) would get smaller. 
\end{sloppypar}

\section{Experiments}
\label{sec.exp}

\begin{table}[t]
    \centering
    \caption{Statistics of the 10 Datasets.}
    \label{tab:statistics}
    \begin{tabular}{c|ccc}
    \toprule
        Dataset & \# Instances & Outlier Fract. & \# of Dims \\
        \midrule
        PageBlock & 5473 & 10\% & 10 \\
        Pendigits & 6870 & 2.3\% & 16 \\
        Shuttle & 49097 &  7\% & 9 \\
        Pima & 768 & 35\% & 8 \\
        Mammography & 11873 & 2.3\% & 6 \\
        Satimage-2 & 5803 & 1.2\%  & 36 \\
        % Musk & 3062 & 3.2\% & 166 \\
        % ALOI & 49534 & 3\% & 27 \\
        Satellite & 6435 &  32\%  & 36 \\
        SpamBase & 4601 & 40\% & 57 \\
        % Annthyroid & 7129 & 7.5\% & 21  \\
        Cover & 286048 &  0.9\% & 10 \\
        % Http & 567498 & 0.4\% &  3 \\
        Thursday-01-03 & 33110 & 28\% & 68 \\
        \bottomrule
    \end{tabular}
\end{table}
Our experimental study aims to answer the following questions:
\begin{compactitem}
    \item \textbf{Q1}: How do \sys and L-\sys compare against other methods in the total rule lengths given a $F_1$ threshold?
    
    \item \textbf{Q2}: How do \sys and L-\sys compare against other methods in $F_1$ score when producing rules with the similar complexity?
    
    \item \textbf{Q3}: How do the parameters $L_m$ and $F1_m$ affect the performance of \sys? 
    
    \item \textbf{Q4}: How does the number of partition $n$ affect the performance of L-\sys?
    
    \item \textbf{Q5}: How good is L-\sys at preserving the locality of the data?
    
    % \item \textbf{Q6}: Are the rules produced by \sys indeed human interpretable?
    
    \item \textbf{Q6}: How does \sys dynamically adjust the value of stabilizer M introduced in our summarization and interpretation-aware optimization objective?
        
    \item \textbf{Q7}: How does \sys perform in multi-class classification? 
\end{compactitem}

\subsection{Experimental Settings}
\label{sec.exp.setting}

\noindent\textbf{Datasets.}
We evaluate the effectiveness of \sys and L-\sys on ten benchmarks outlier detection datasets. Table~\ref{tab:statistics} summarizes their key statistics.

\noindent \textbf{Hardware Settings.} We implement our algorithm with python 3.7. We use the decision tree algorithms in scikit-learn and implement \sys with numpy. 
We train all models on AMD Ryzen Threadripper 3960X 24-Core Processor with 136GB RAM.

\noindent\textbf{Baselines.}
We compare against two decision-tree methods:
\begin{compactitem}
    \item \textbf{ID3}~\cite{DBLP:journals/ml/Quinlan86}: The classic decision tree algorithm. To find the simplest decision tree that satisfies the accuracy threshold $F1_m$, we start with a small tree (depth 3) and iteratively increases its depth until the obtained tree could yield a $F1$ score larger than $F1_m$. 
    \item \textbf{CART}~\cite{10.1007/978-981-10-6747-1_4}: CART uses a post-processing to prune a learned decision tree. The goal is to minimize the complexity of the decision tree, while still preserving the accuracy. We first use ID3 to build a decision tree that is as accurate as possible and then continue to prune it until it is right above the F-1 score threshold.
\end{compactitem}

% \noindent\textbf{Implementation details.}
% %We use the popular outlier detection method: Local Outlier Factor to obtain the labels for the training of the decision tree. 
% In the main experiments, 
% % for \sys, we set the maximum length of the rules $L_m$ to 10 and the $F_1$ score threshold $F1_m$ to 0.8. L-\sys automatically tunes the number of clusters $n$ within the $[2, 4, 8]$. 
% we set the maximal iteration to $10$. The temporary threshold for \sys in each iteration for all clusters is set to be $0.6 + G/10 * (i-2)$, where $i$ refers to current number of clusters, and $G$ is set to 0.2. This means if $n=2$, then the threshold for \sys will increase from $0.6$ to $0.8$ gradually. 

\subsection{Comparison Against Baselines (Q1):  Total Rule Length}
\label{sec.exp.baseline}

\begin{table}[t]
\centering
\caption{Total rule length under similar $F_1$ score (Q1).}
\label{tab:overall_performance_comparison}
\begin{tabular}{c|c|c|c|cc}
\toprule
Dataset & \multicolumn{1}{c|}{ID3} & \multicolumn{1}{c|}{CART}  & \multicolumn{1}{c|}{STAIR} & \multicolumn{1}{c}{L-STAIR}  \\
% Dataset & Length &  \# of Rules & Length &  \# of Rules & Length & \# of Rules & Length & \# of Rules & \# of Clusters \\
%  & Length & Length & Length & Length & \# of Clusters \\
\midrule
PageBlock & 97 & 88 & 50 & \textbf{25} \\ 
Pendigits & 290 & 328 & 187 & \textbf{60}\\
Shuttle & 1520 & 863 & 697 & \textbf{125} \\ 
% Pima & 69 & 35 & 39 & \textbf{10} \\ 
Pima & 20 & 12 & 12 & \textbf{10} \\
Mammography & 79 & 65 & 66 & \textbf{24} \\ 
Satimage-2 & 151 & 117 & 93 & \textbf{38} \\ 
Satellite & 1263 & 471 & 442 & \textbf{70} \\ 
SpamBase & 1546 & 1043 & 1017 & \textbf{150} \\ 
Cover & 6616 & 4869 & 4657 & \textbf{402} \\ 
Thursday-01-03 & 4032 & 1393 & 957 & \textbf{440} \\
\bottomrule
\end{tabular}
\end{table}

\begin{figure}[t]
    \centering
    \subfigure[Pendigits]{\label{fig:accwrtlength_pendigits}\includegraphics[width=0.48\linewidth]{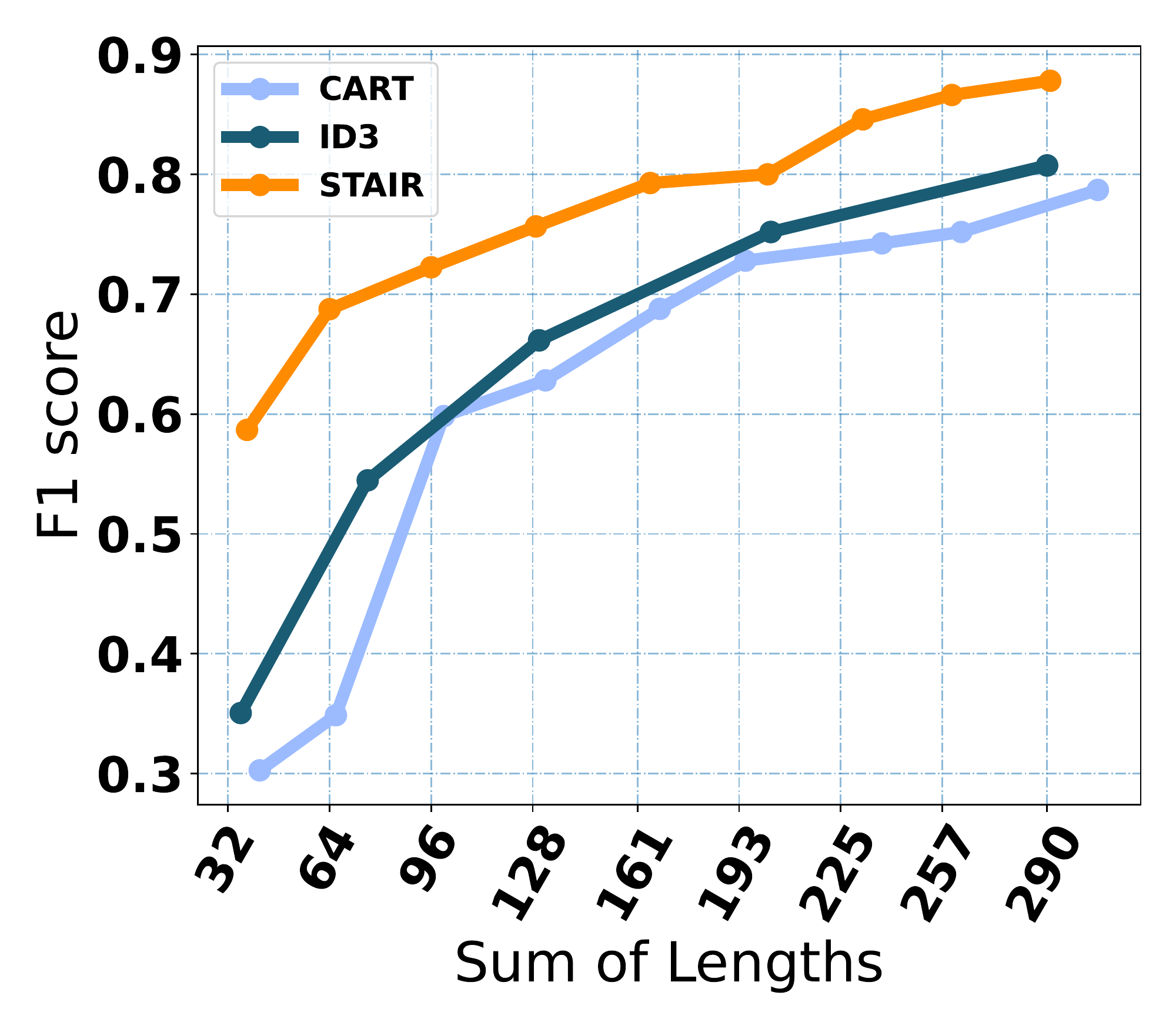}}
    \subfigure[Thursday-01-03]{\label{fig:accwrtlength_Thursday-01-03}\includegraphics[width=0.48\linewidth]{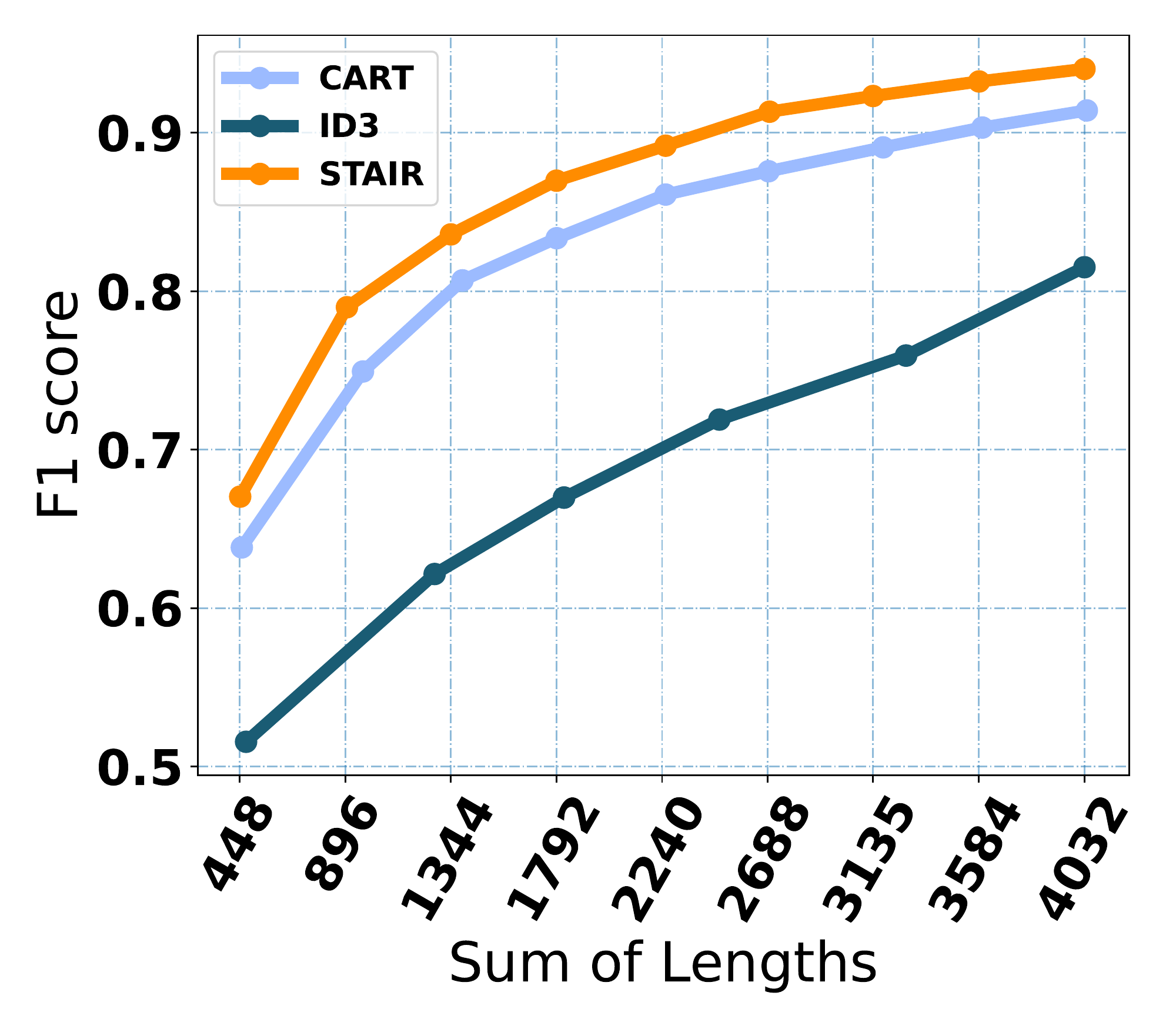}}
    \caption{$F_1$ score with varying total rule lengths (Q2).}
    \label{fig:AccwrtLength}
\end{figure}

\begin{figure*}[ht]
    \centering
    \subfigure[PageBlock]{\label{fig:pageblock_lm}\includegraphics[width=0.19\linewidth]{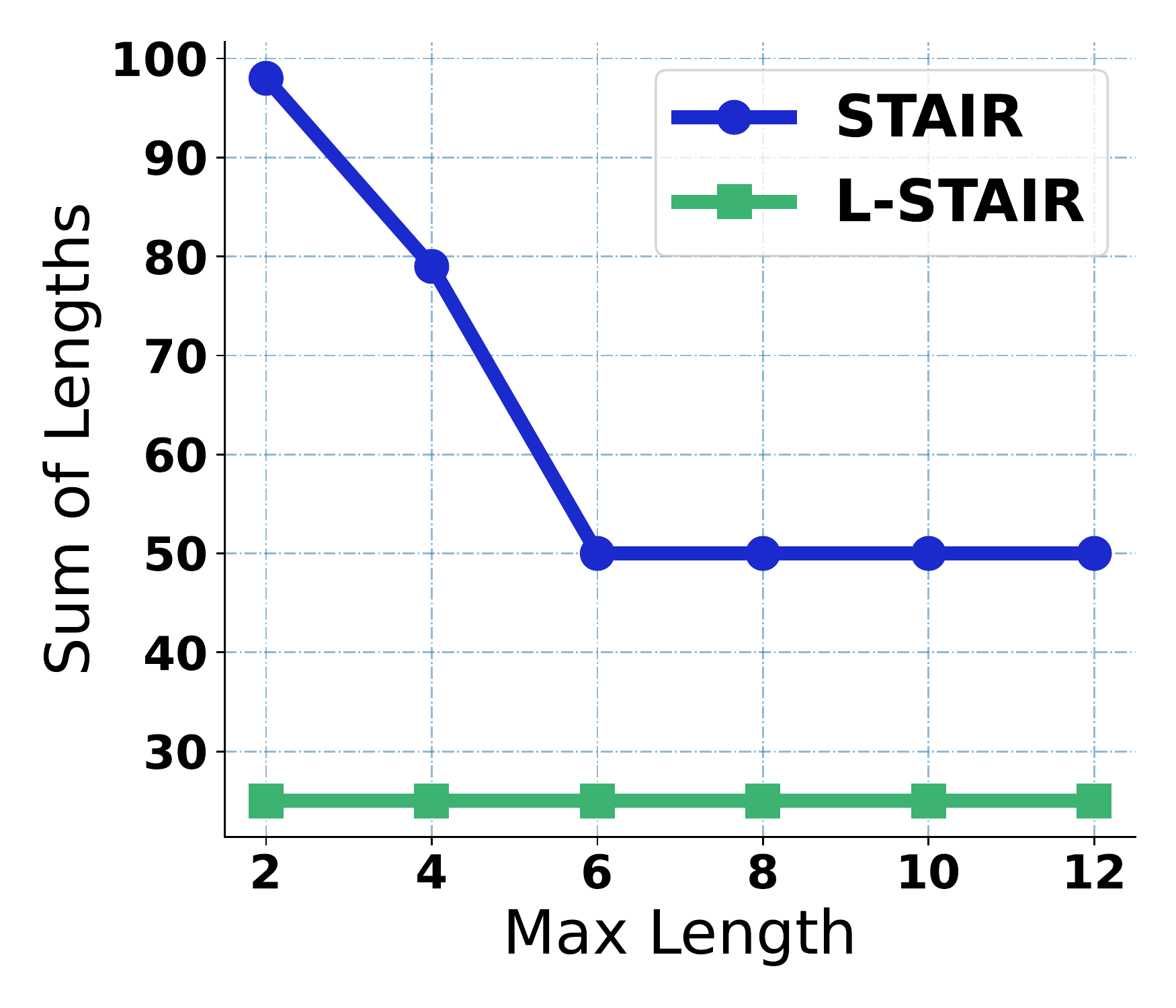}}
    \subfigure[Pendigits]{\label{fig:pendigits_lm}\includegraphics[width=0.19\linewidth]{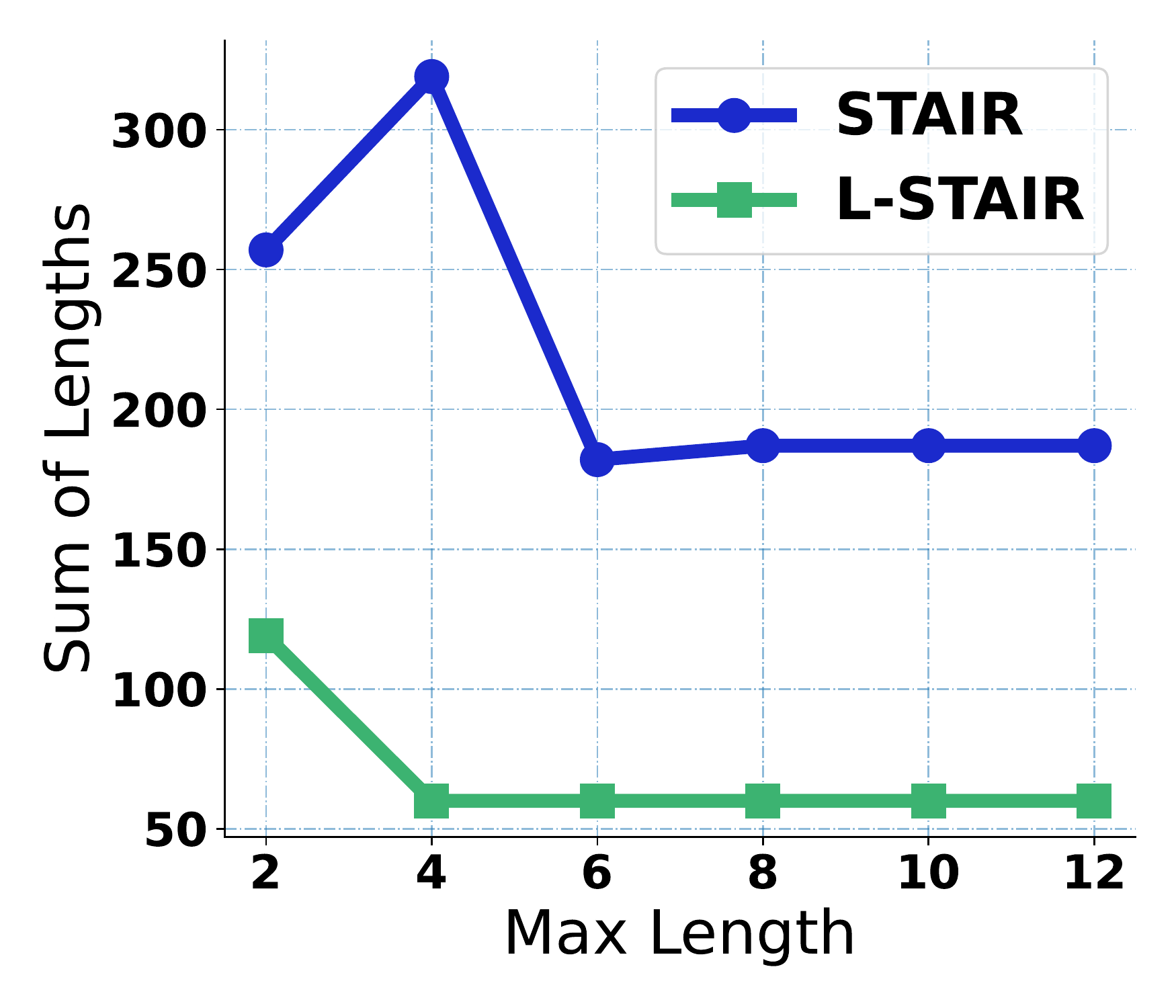}}
    \subfigure[Shuttle]{\label{fig:shuttle_lm}\includegraphics[width=0.19\linewidth]{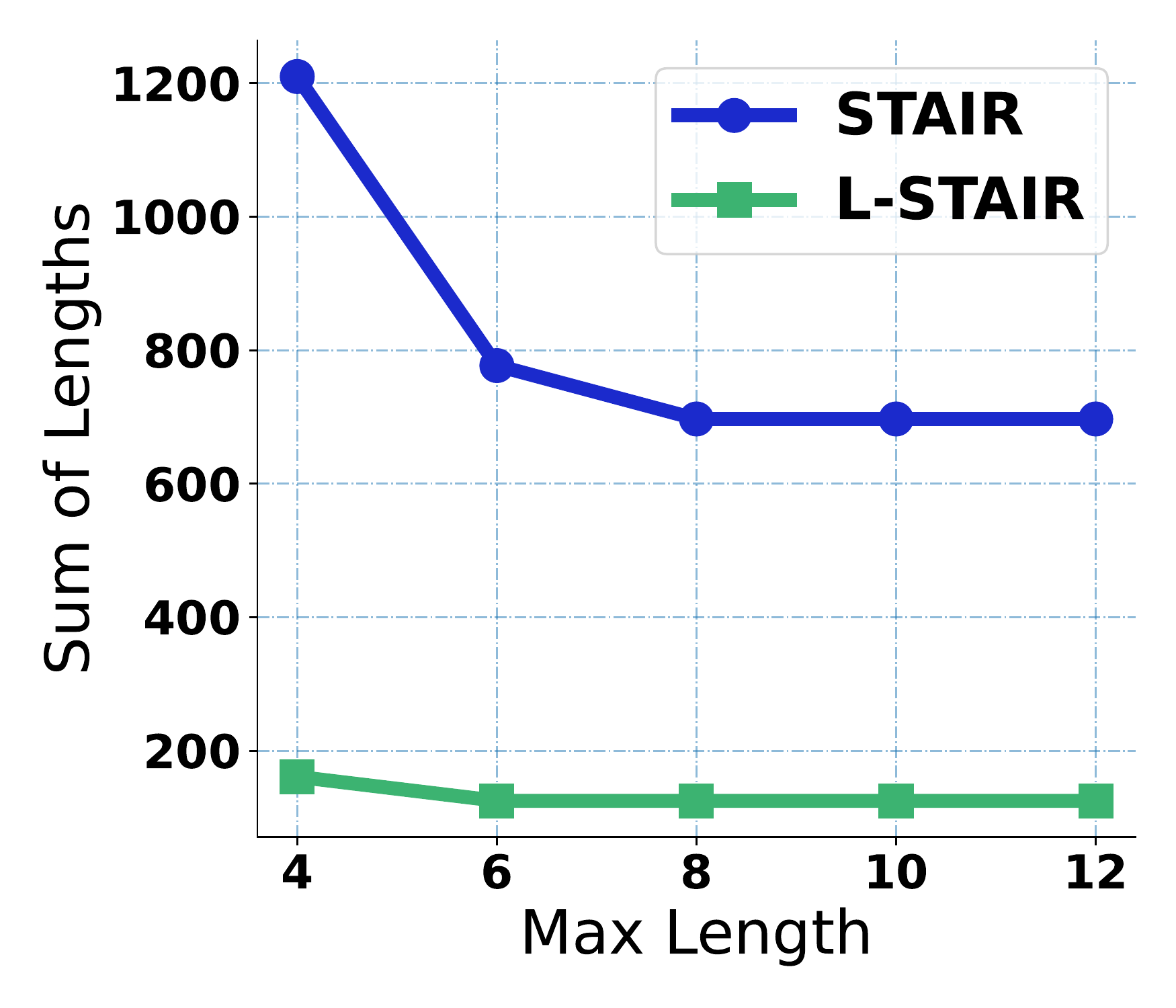}}
    \subfigure[Pima]{\label{fig:pima_lm}\includegraphics[width=0.19\linewidth]{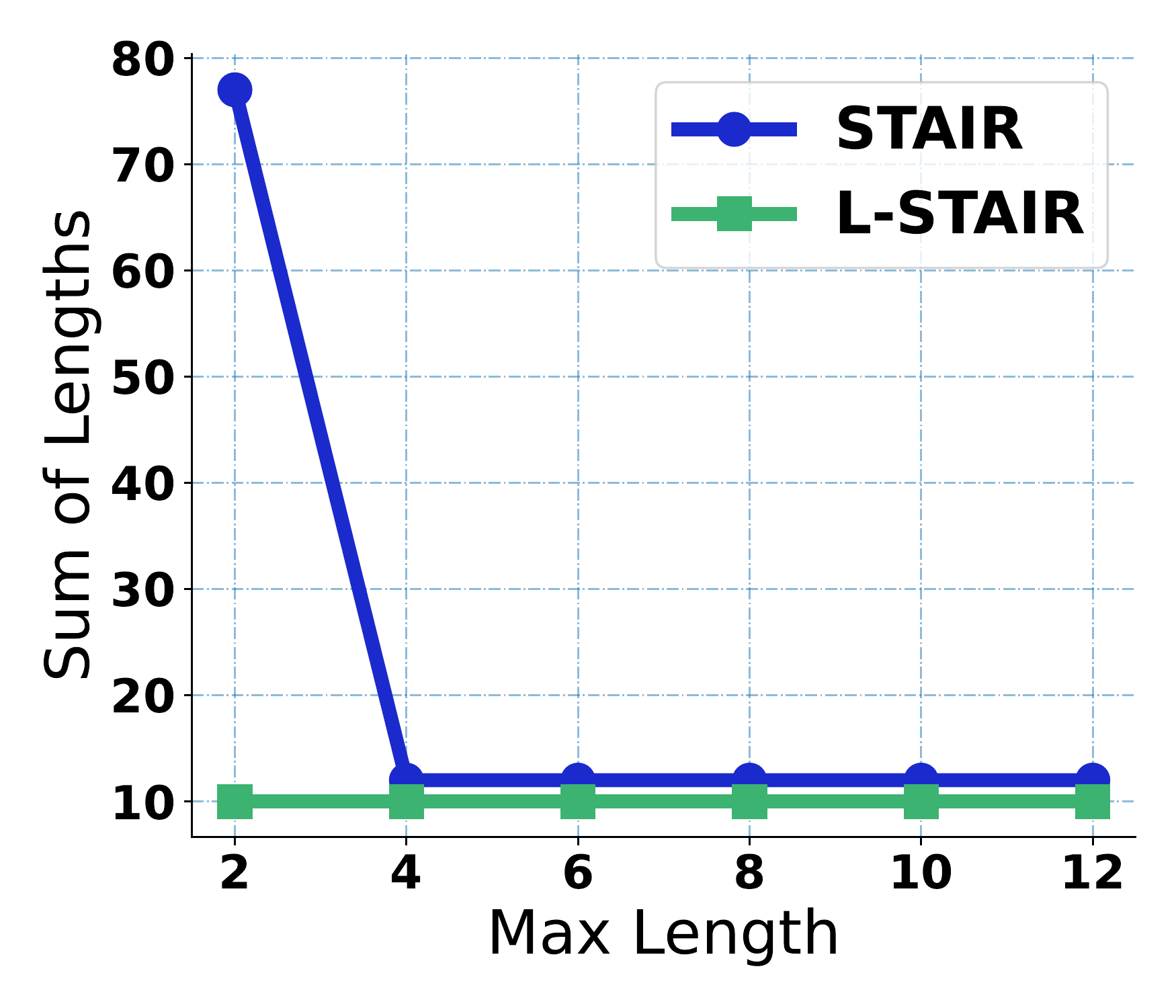}}
    \subfigure[Mammography]{\label{fig:mammography_lm}\includegraphics[width=0.19\linewidth]{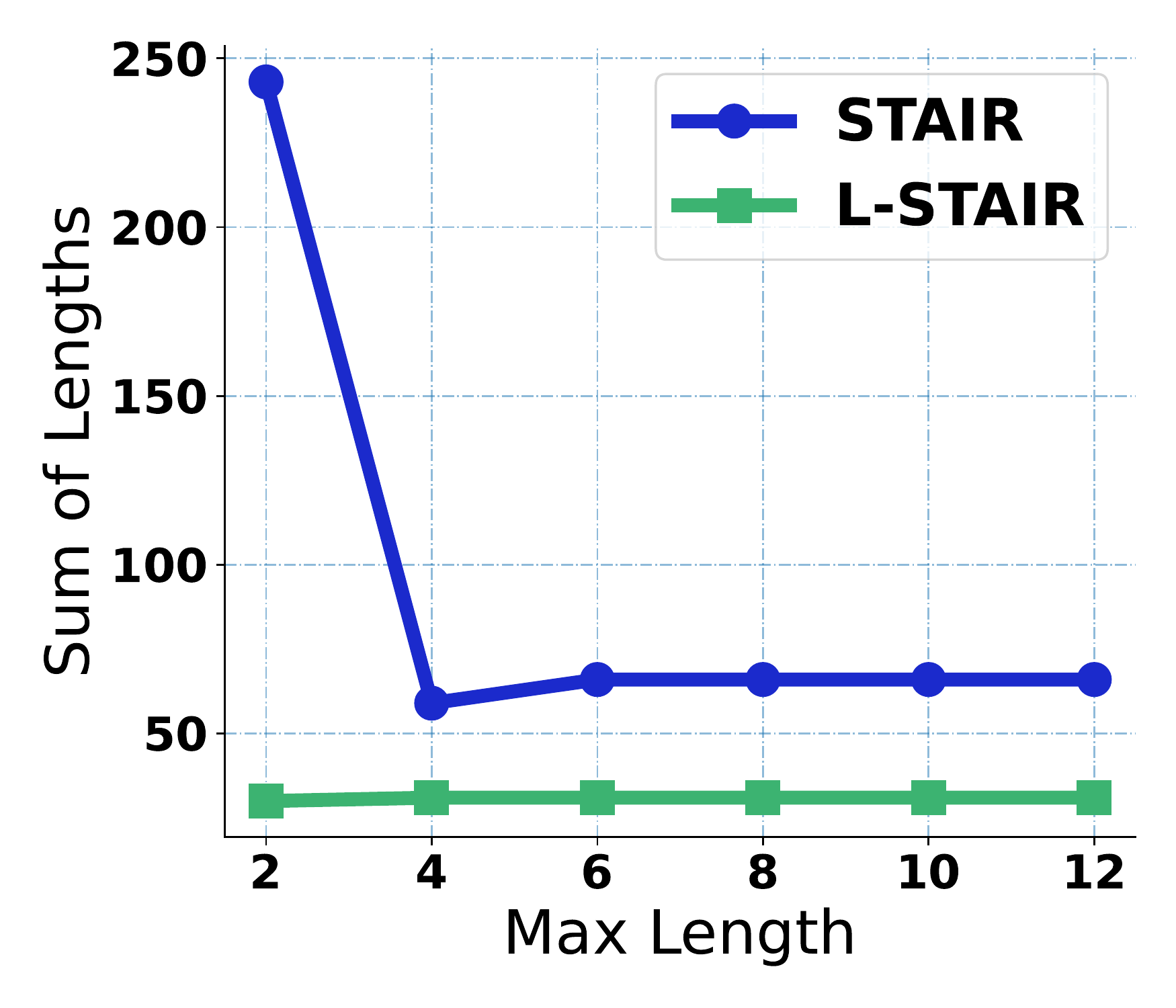}}
    \subfigure[Satimage-2]{\label{fig:Satimage-2_lm}\includegraphics[width=0.19\linewidth]{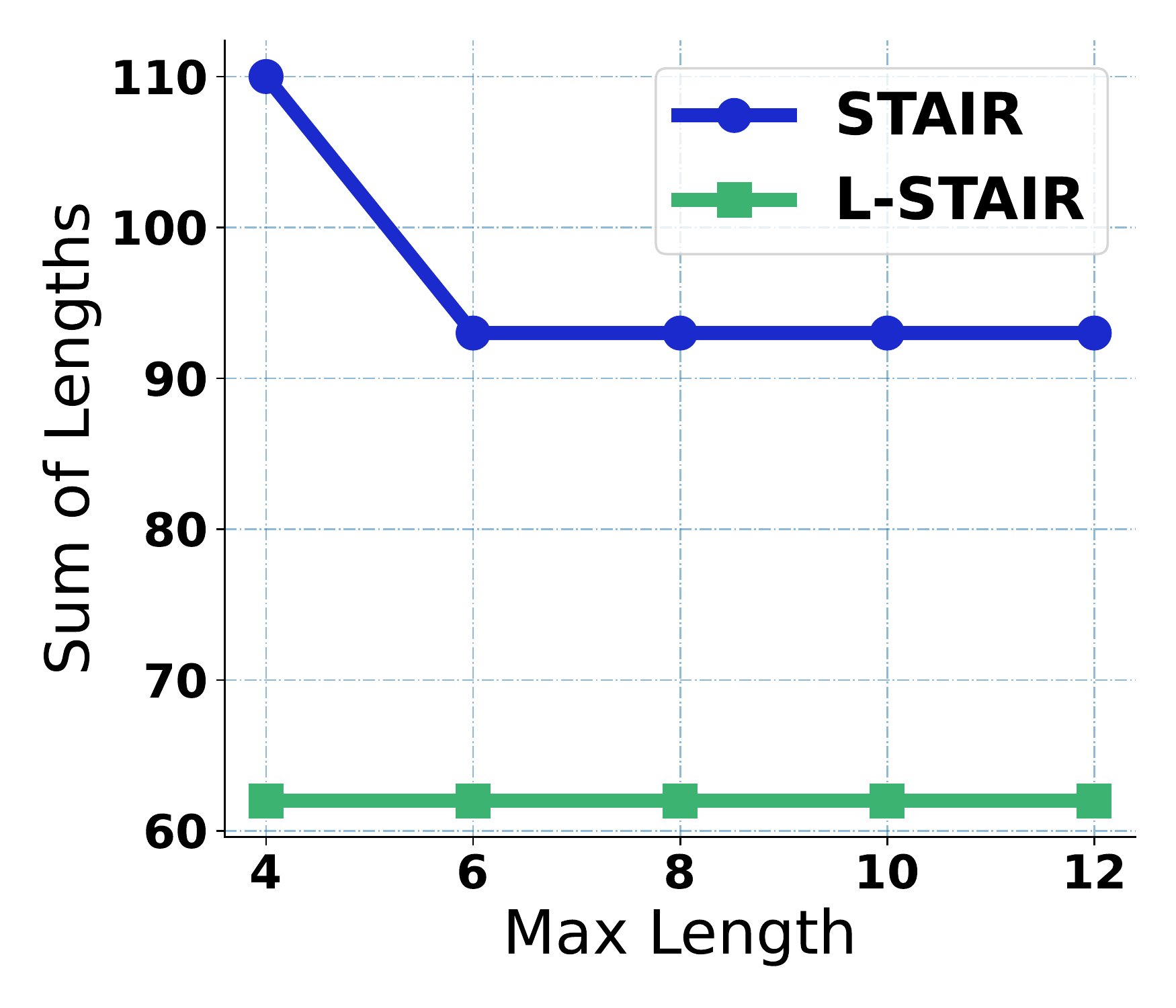}}
    \subfigure[Satellite]{\label{fig:satellite_lm}\includegraphics[width=0.19\linewidth]{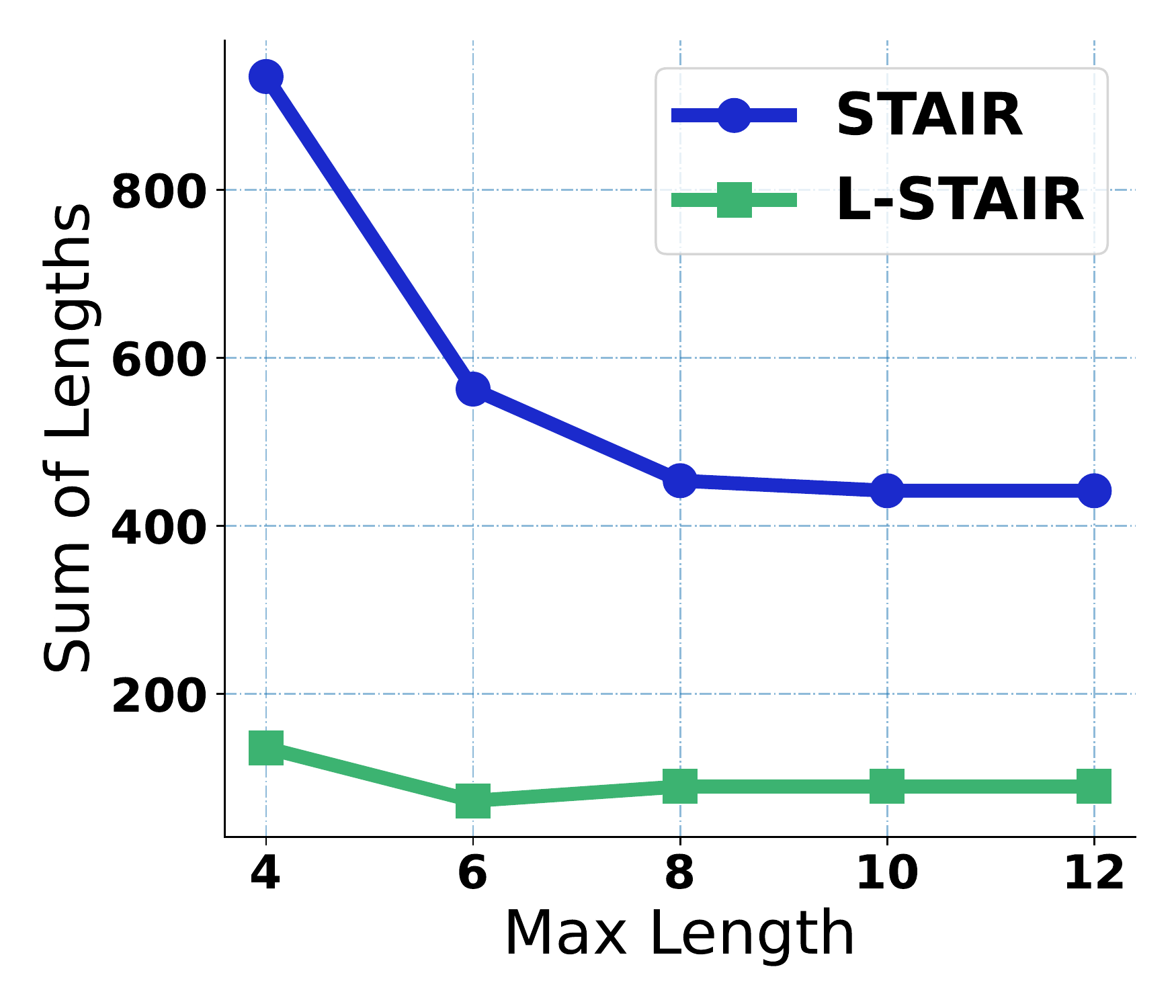}}
    \subfigure[SpamBase]{\label{fig:SpamBase_lm}\includegraphics[width=0.19\linewidth]{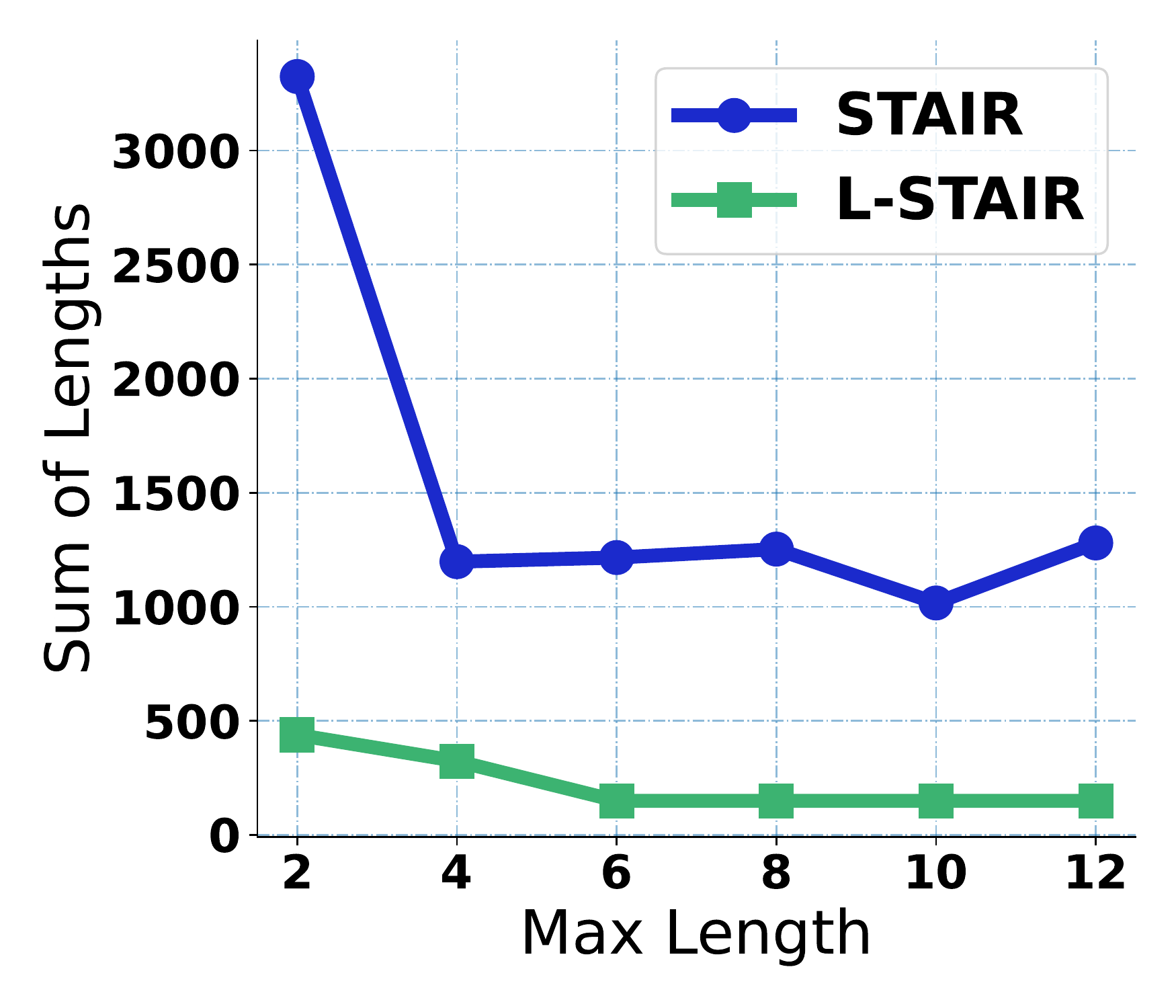}}
    \subfigure[Cover]{\label{fig:cover_lm}\includegraphics[width=0.19\linewidth]{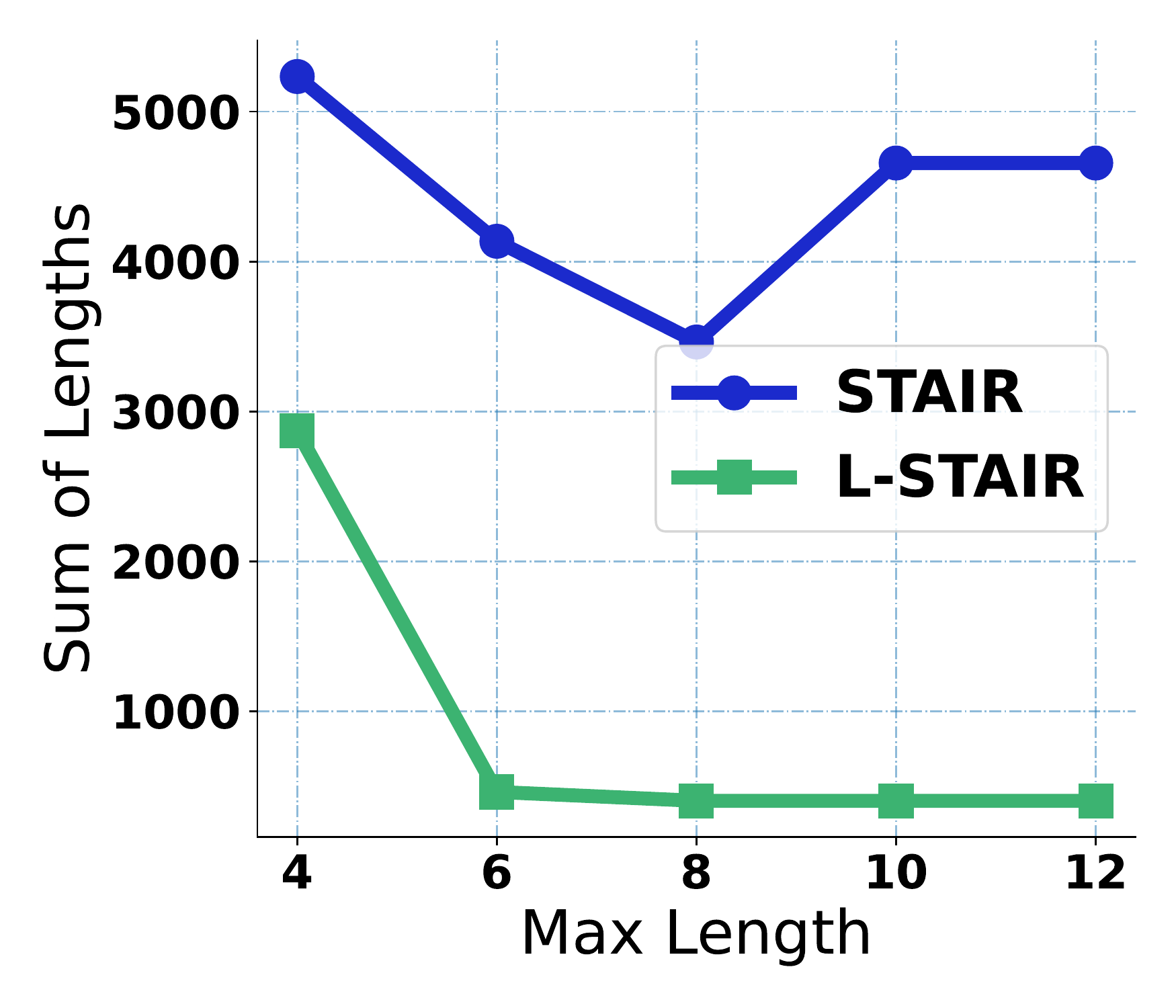}}
    \subfigure[Thursday-01-03]{\label{fig:Thursday-01-03_lm}\includegraphics[width=0.19\linewidth]{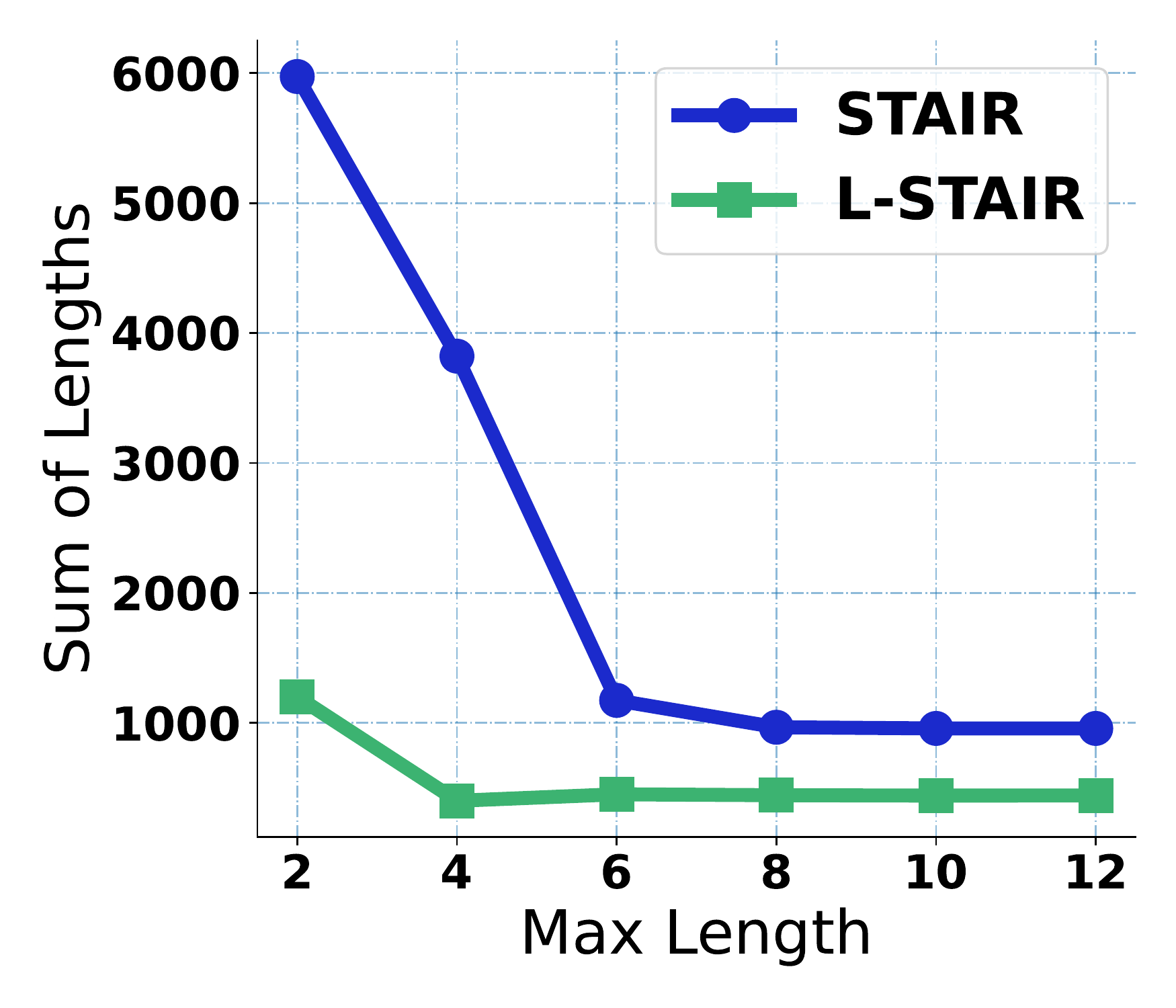}}
    \caption{The effects of the maximal length $L_m$ on the total rule length (Q3).}
    \label{fig:effects_of_lm}
\end{figure*}

\begin{figure*}
    \centering
    \subfigure[PageBlock]{\label{fig:pageblock_f1m}\includegraphics[width=0.19\linewidth]{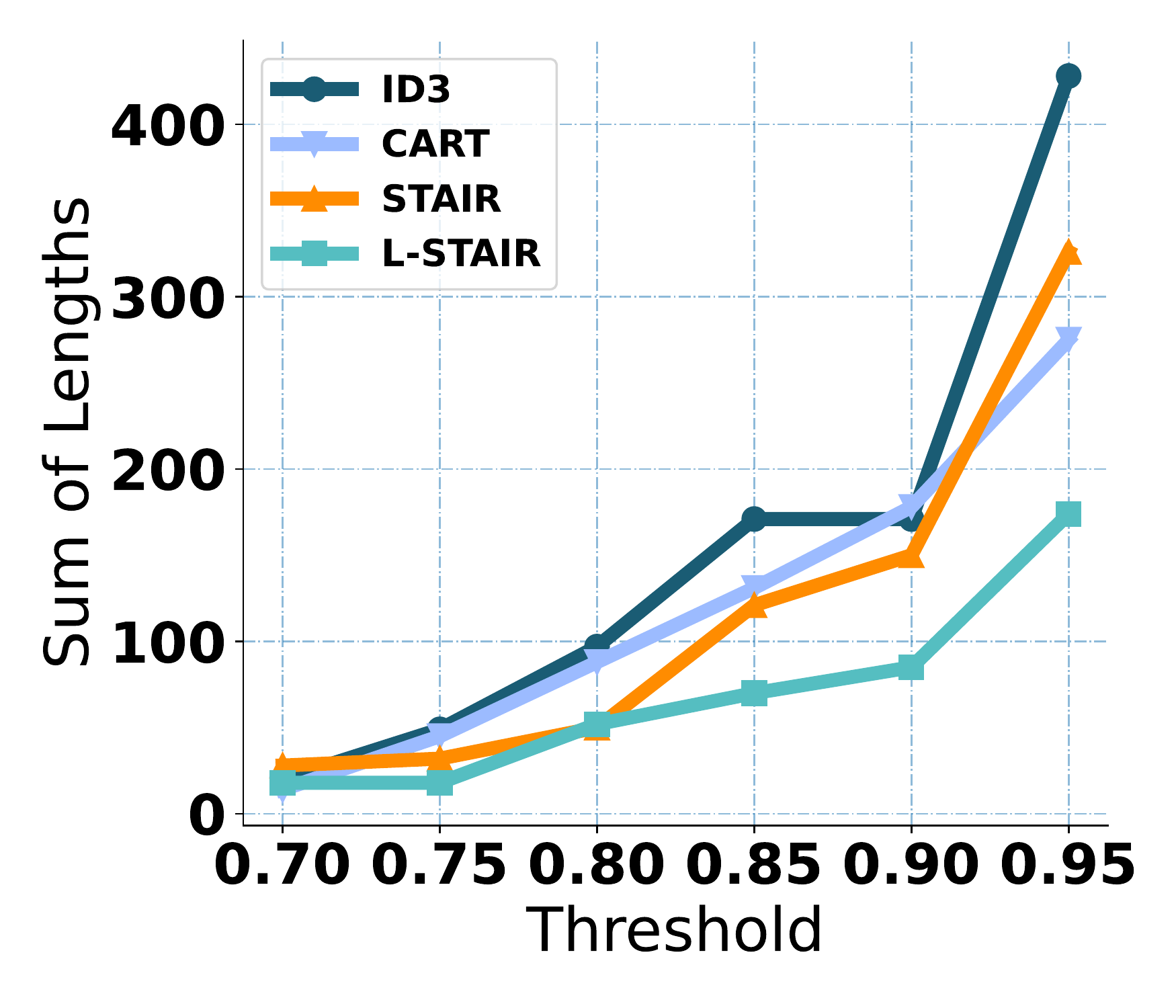}}
    \subfigure[Pendigits]{\label{fig:pendigits_f1m}\includegraphics[width=0.19\linewidth]{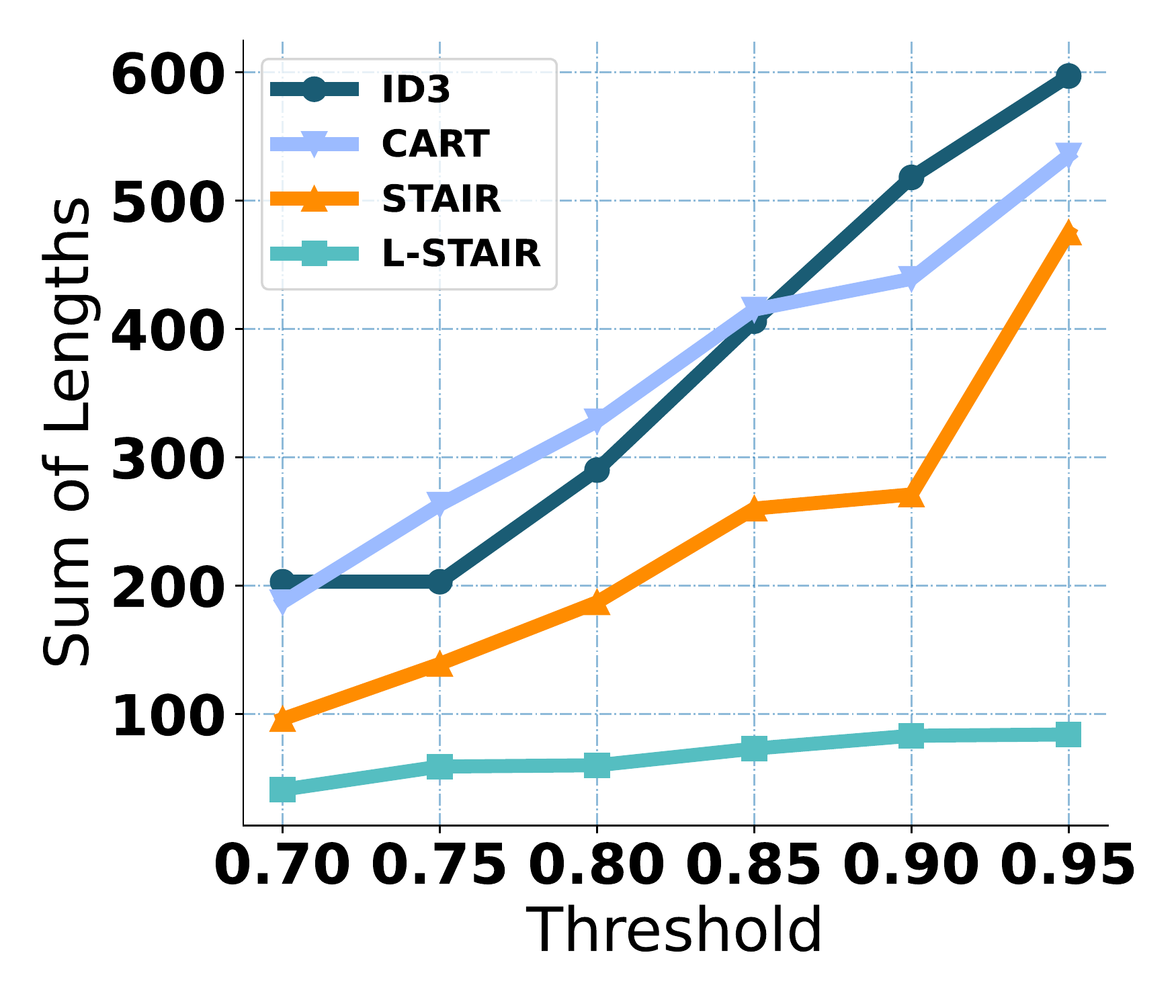}}
    \subfigure[Shuttle]{\label{fig:shuttle_f1m}\includegraphics[width=0.19\linewidth]{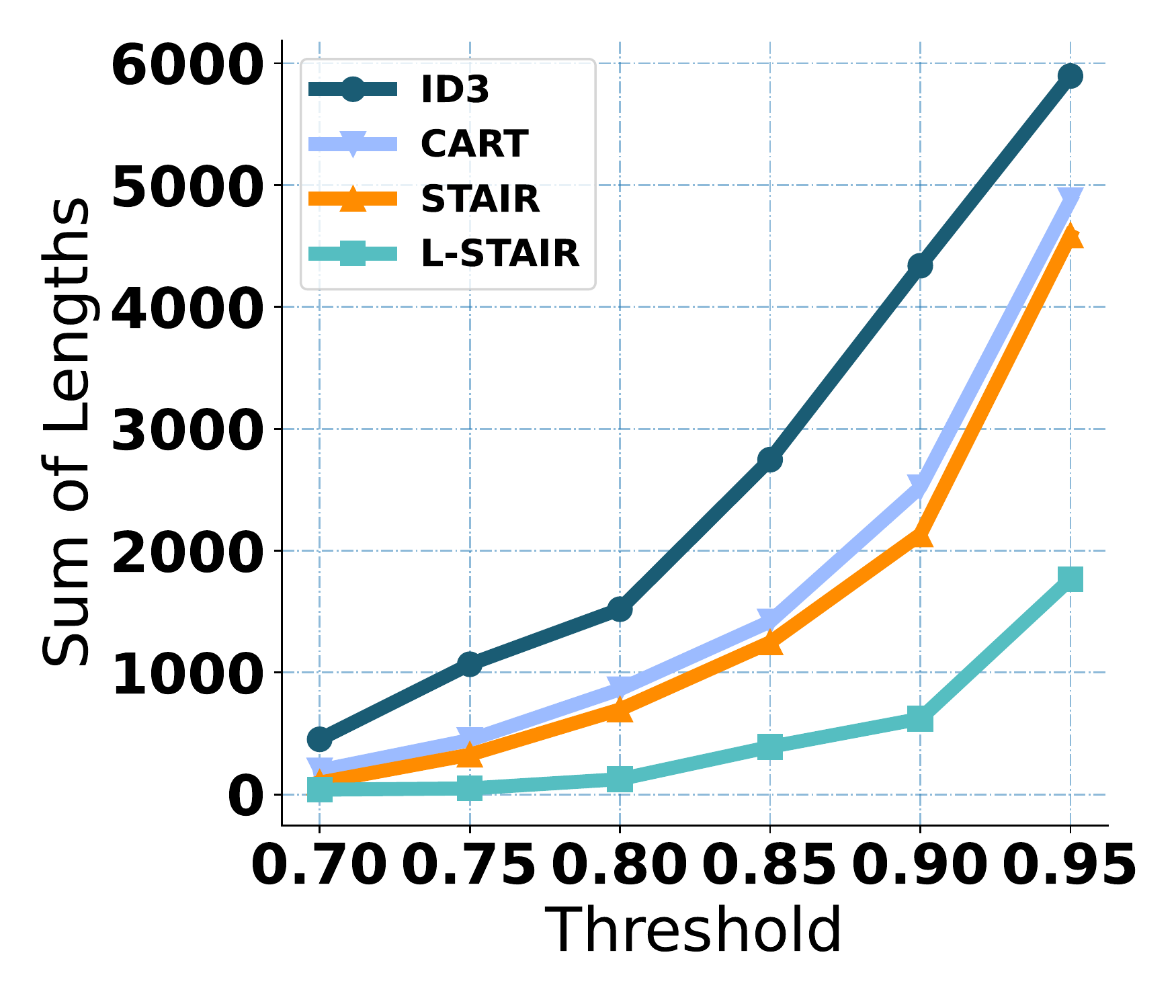}}
    \subfigure[Pima]{\label{fig:pima_f1m}\includegraphics[width=0.19\linewidth]{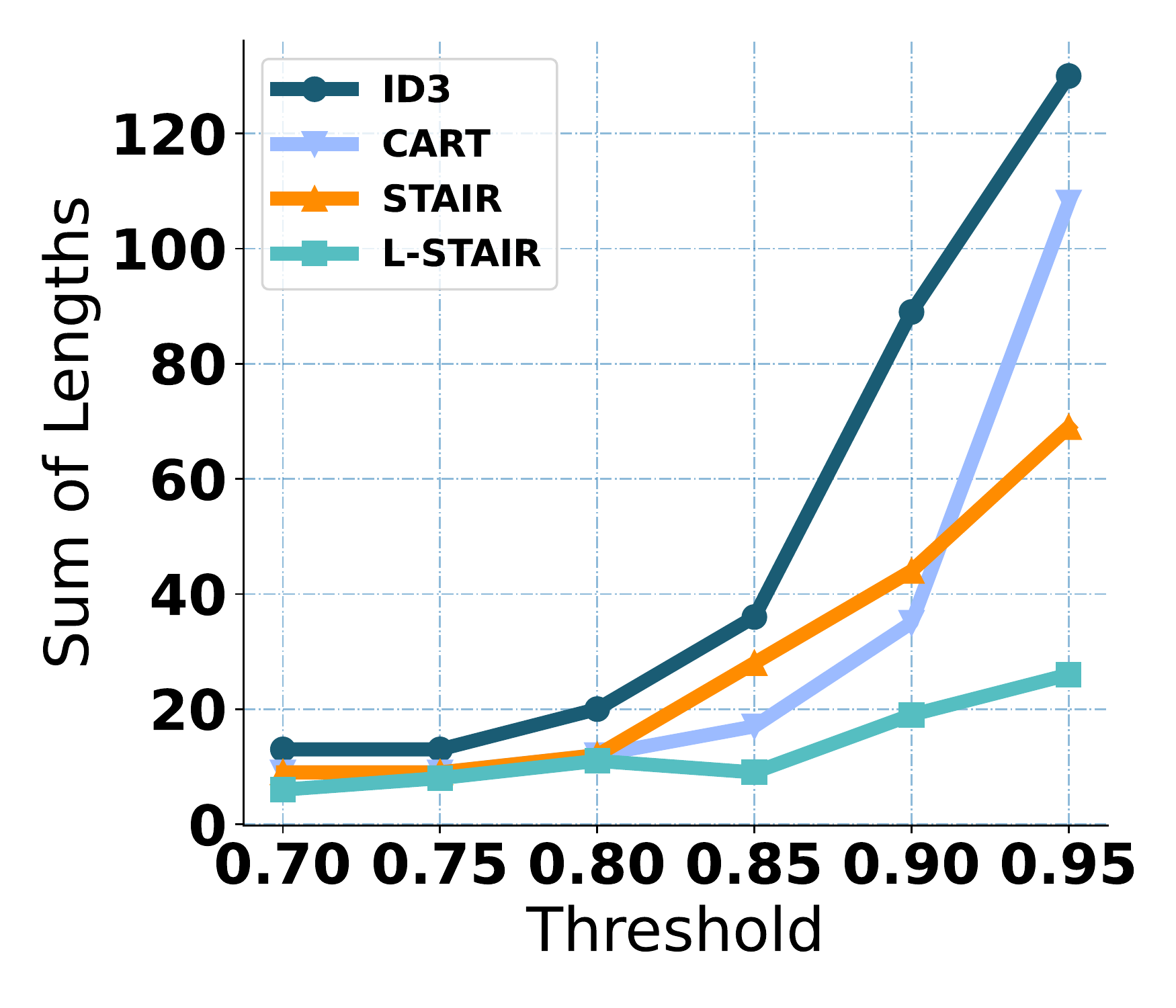}}
    \subfigure[Mammography]{\label{fig:mammography_f1m}\includegraphics[width=0.19\linewidth]{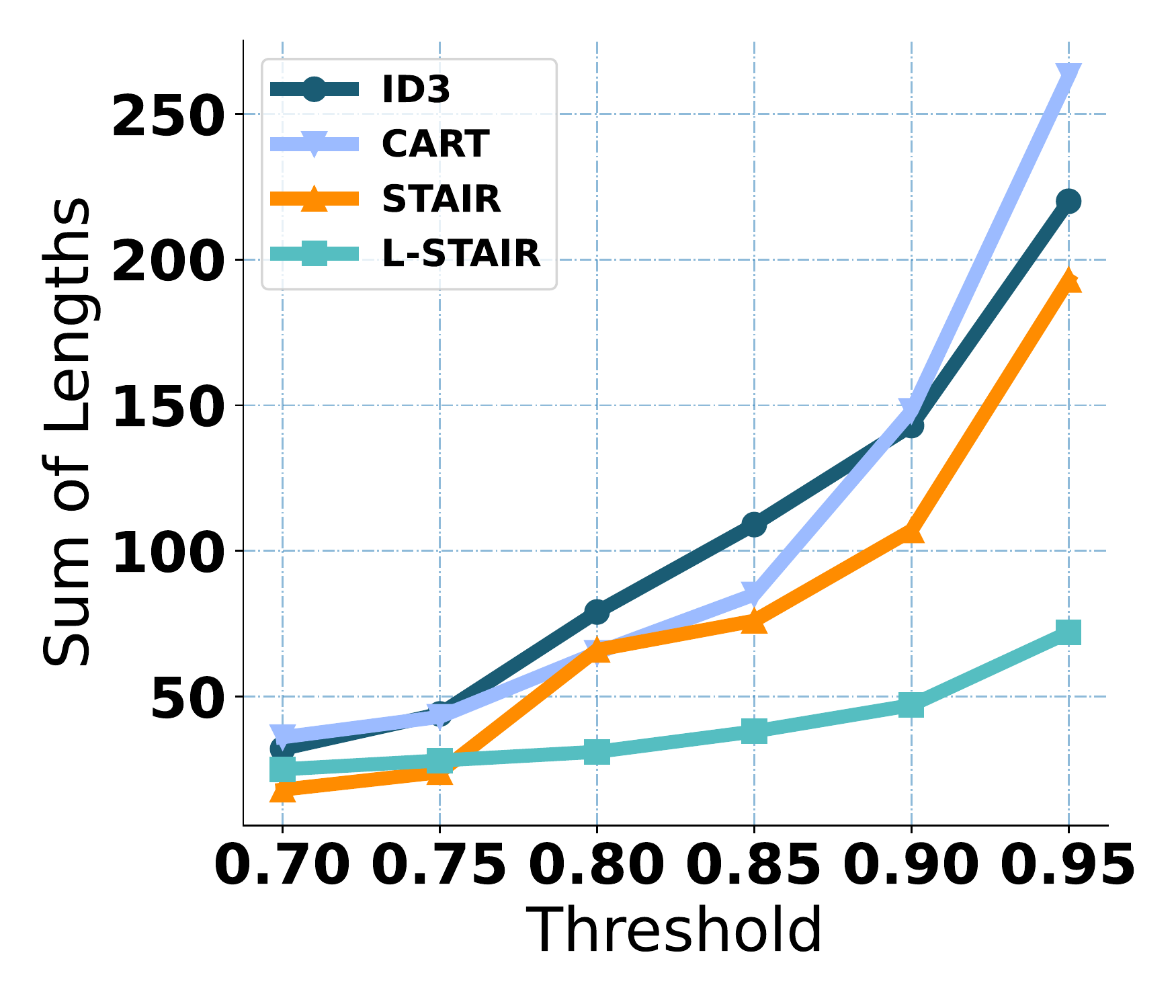}}
    \subfigure[Satimage-2]{\label{fig:Satimage-2_f1m}\includegraphics[width=0.19\linewidth]{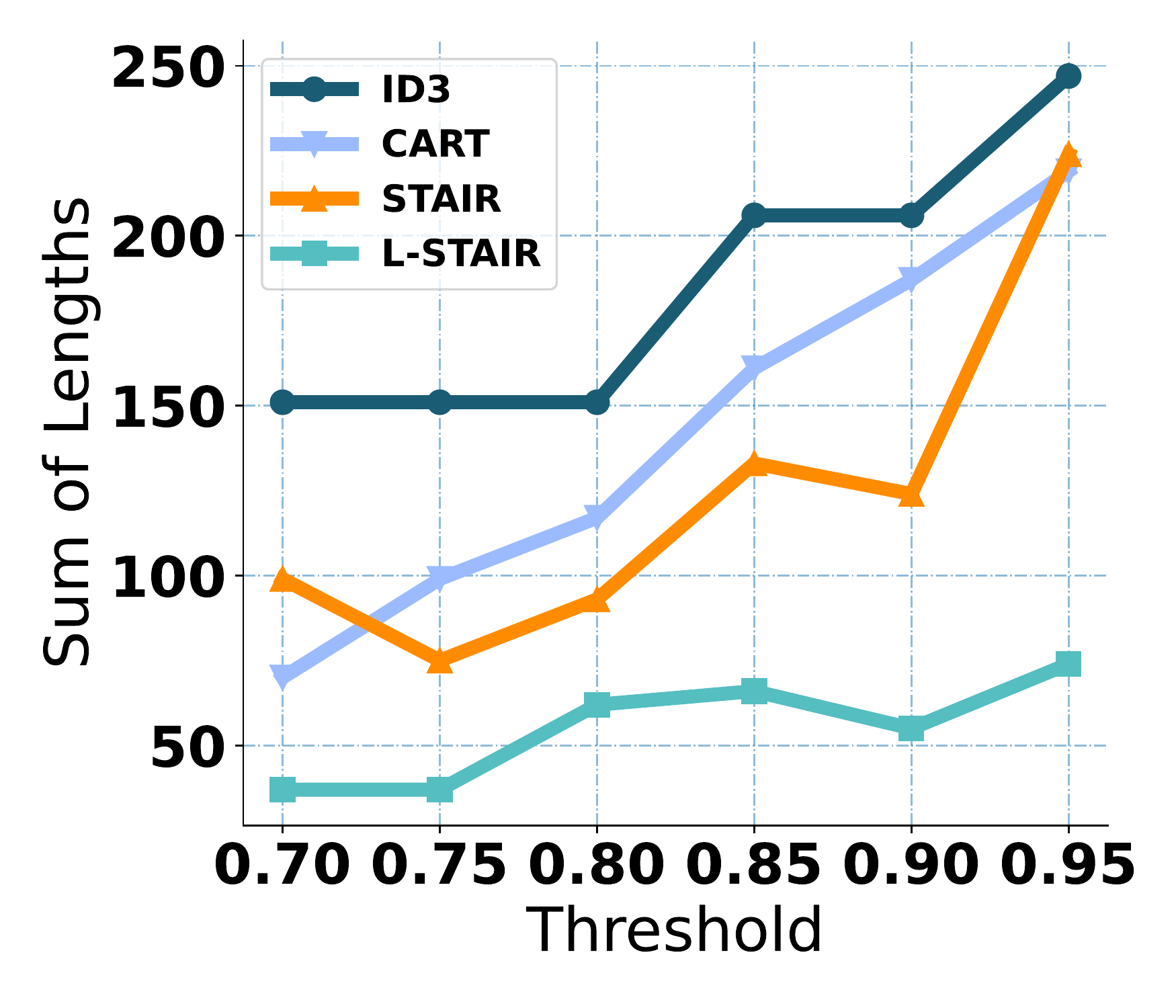}}
    \subfigure[Satellite]{\label{fig:satellite_f1m}\includegraphics[width=0.19\linewidth]{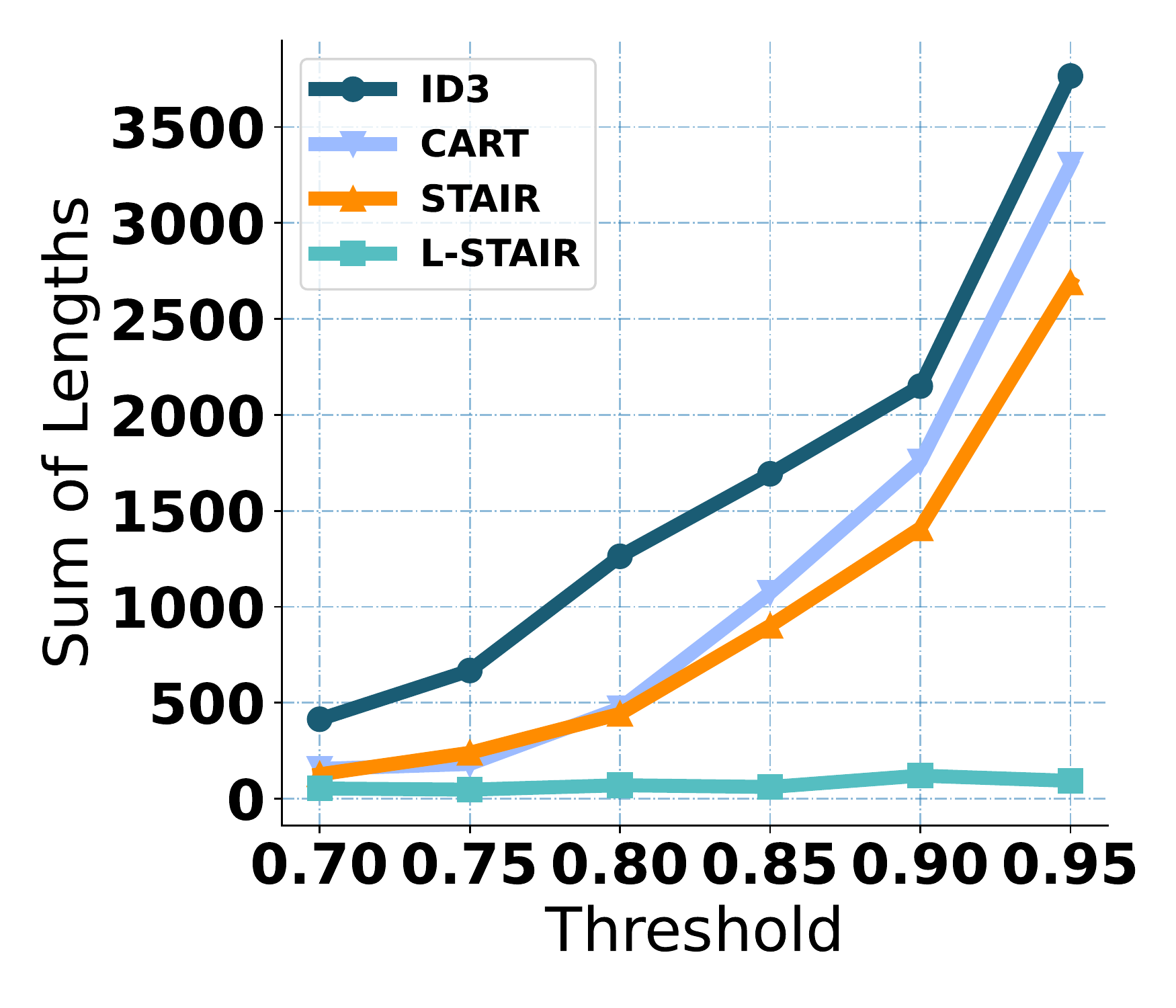}}
    \subfigure[SpamBase]{\label{fig:SpamBase_f1m}\includegraphics[width=0.19\linewidth]{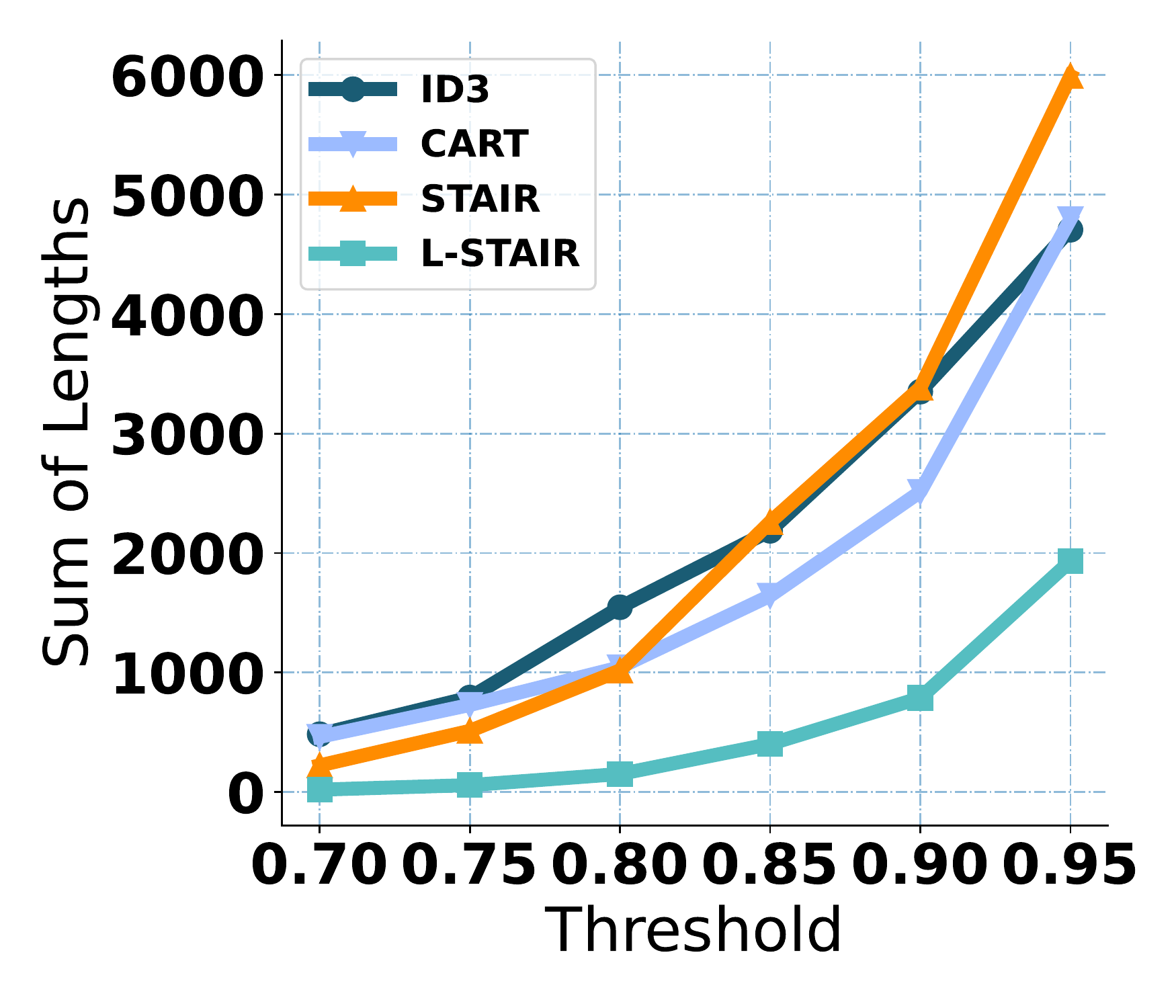}}
    \subfigure[cover]{\label{fig:cover_f1m}\includegraphics[width=0.19\linewidth]{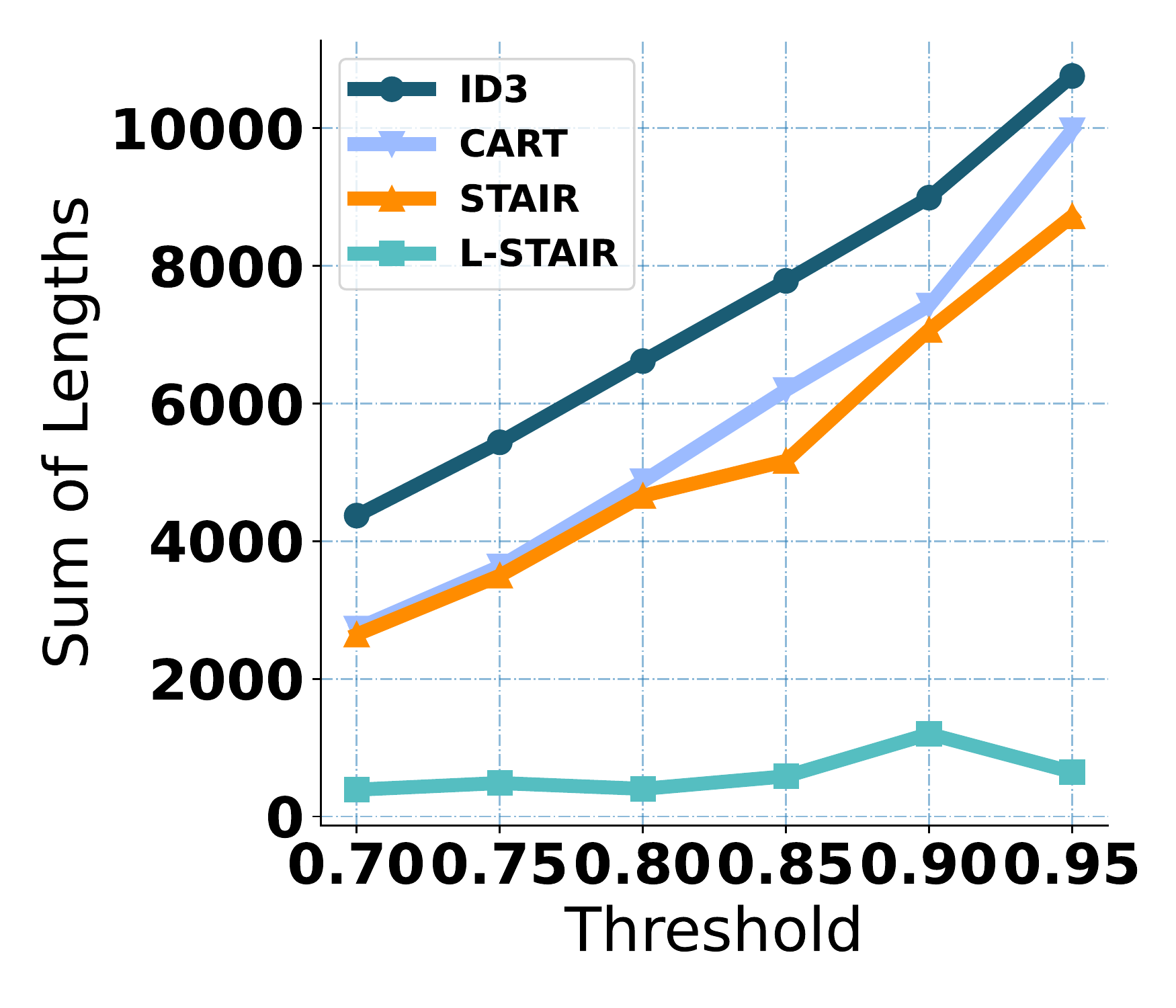}}
    % \subfigure[Http]{\label{fig:Http_lm}\includegraphics[width=0.19\linewidth]{figure/Http_lm.pdf}}
    \subfigure[Thursday-01-03]{\label{fig:Thursday_01_03_f1m}\includegraphics[width=0.19\linewidth]{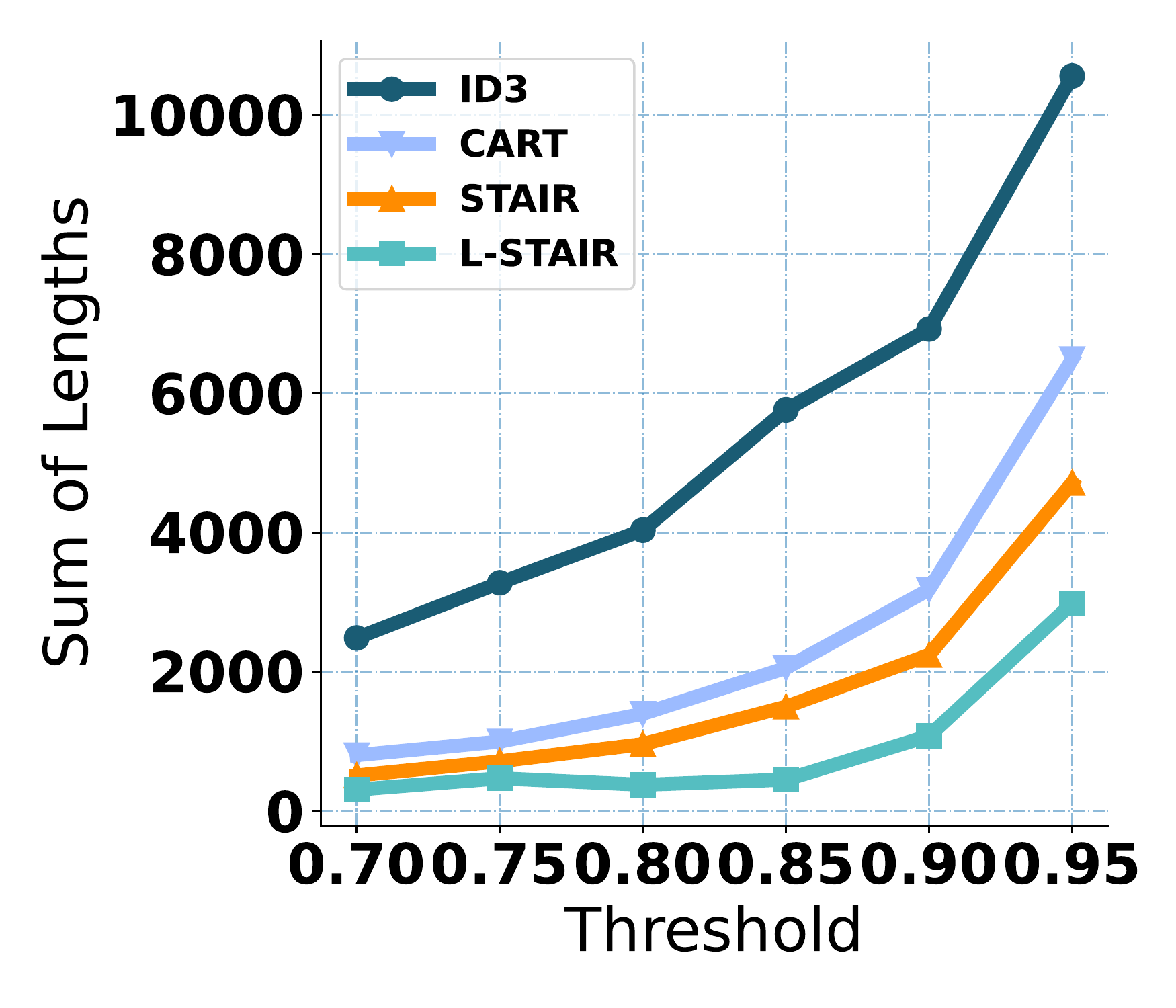}}
    \caption{The effect of the threshold $F1_m$ on all methods (Q3).}
    \label{fig:effects_of_f1m}
    \vspace{-8pt}
\end{figure*}

In this set of experiments, we measure the total length of the rules produced by each algorithm when they produce trees with similar $F_1$ score.
We set the maximal length of the rules $L_m$ to 10 and the $F_1$ score threshold $F1_m$ to 0.8. 
Because ID3 and CART do not use the $F_1$ score threshold in their algorithms, we tune their hyper-parameters to produce the simplest tree with a $F_1$ score slightly higher than 0.8.
This ensures that all algorithms have the {\bf similar} $F_1$ score.
L-\sys automatically determines the number of data partitions with initial partition number picked from \{2, 4, 8\}. 
We set the maximal iteration to $10$.

% \srm{I'm missing the overall accuracy results?  Where are those shown?
% Also, it seems like we should compare against a decision tree with a fixed shallow depth and see how well it performs.} 

From the results shown in Table~\ref{tab:overall_performance_comparison}, we draw the following conclusions: (1) Compared to ID3 and CART, \sys is able to produce much simpler rules which are amenable for humans to evaluate,  significantly reducing the total length of the rules by up to 76.3\% as shown in the results of the dataset \emph{Thursday-01-03}. This is because the summarization and interpretation-aware optimization objective (Eq.~\ref{eq:formulation}) of \sys simultaneously minimizes the complexity of the tree and maximizes the classification accuracy; (2) The performance of \sys on the SpamBase dataset is not satisfying potentially due to its large dimensionality. SpamBase has 57 attributes. It thus might be over-complicated to use a single small tree to summarize the whole dataset; (3) L-\sys which partitions the data and produces one tree for each data partition solves the problem mentioned in (2) and outperforms the basic \sys by up to 91.37\% as shown in the results of the dataset \emph{Cover}.

% \srm{I don't understand why we don't have a single table that shows accuracy and length of rules for the shallow trees, ID3, CART, and the STAIR variants.  It feels like we are spreading this info across several tables / results and it makes it hard to assess the overall performance of the algorithm.}\wy{This might be a little hard since ID3, CART, STAIR all have the F1 score around 0.8, and comparing their total length would be simpler. For shallow trees, the F1 score and the total length are both much lower than the other three algorithms, which means we need to contain the total length and the F1 score in the same table for comparison. In this case, since the $F_1$ score in Table \ref{tab:overall_performance_comparison} are all around 0.8, there will be lots of "0.8" in the table, which might be a little redundant.}

% \srm{Also relate the depth of the trees to the number of rules in STAIR / L-STAIR so we can compare STAIR to thte fixed tree methods in terms of rules.}

\begin{table*}[t]
    \centering
    \caption{The number of partitions $n$ in L-\sys (Q4).}
    \label{tab:study_of_n}
    \begin{tabular}{c|ccc|ccc|ccc}
    \toprule
         & \multicolumn{3}{c|}{L-STAIR ($n$=2)} & \multicolumn{3}{c|}{L-STAIR ($n$=4)}  & \multicolumn{3}{c}{L-STAIR ($n$=8)}  \\
         \midrule
        Dataset & Length & \# of R & \# of C & Length & \# of R & \# of C & Length & \# of R & \# of C \\
        \midrule
        PageBlock & 31 & 26 & 9 & 25 & 20 & 8 & 67 & 31 & 8 \\
        Pendigits & 60 & 20 & 2 & 86 & 37 & 4 & 68 & 39 & 8 \\
        Shuttle & 125 & 52 & 9 & 309 & 112 & 10 & 357 & 108 & 9 \\
        Pima & 10 & 6 & 2 & 17 & 12 & 4 & 27 & 20 & 8 \\
        Mammography & 31 & 20 & 6 & 24 & 15 & 5 & 29 & 21 & 8 \\
        Satimage-2 & 62 & 27 & 4 & 38 & 19 & 4 & 49 & 31 & 8 \\
        Satellite & 70 & 18 & 2 & 78 & 28 & 4 & 209 & 74 & 8 \\
        SpamBase & 150 & 51 & 8 & 218 & 73 & 10 & 278 & 72 & 9 \\
        Cover & 402 & 117 & 5 & 470 & 136 & 4 & 621 & 195 & 8 \\
        Thursday-01-03 & 477 & 183 & 11 & 479 & 182 & 11 & 440 & 169 & 11 \\
        \bottomrule
    \end{tabular}
    \vspace{-10pt}
\end{table*}

\begin{figure}[ht]
    \centering
\subfigure[PageBlock]{\label{fig:pageblock_lomdt}\includegraphics[width=0.32\linewidth]{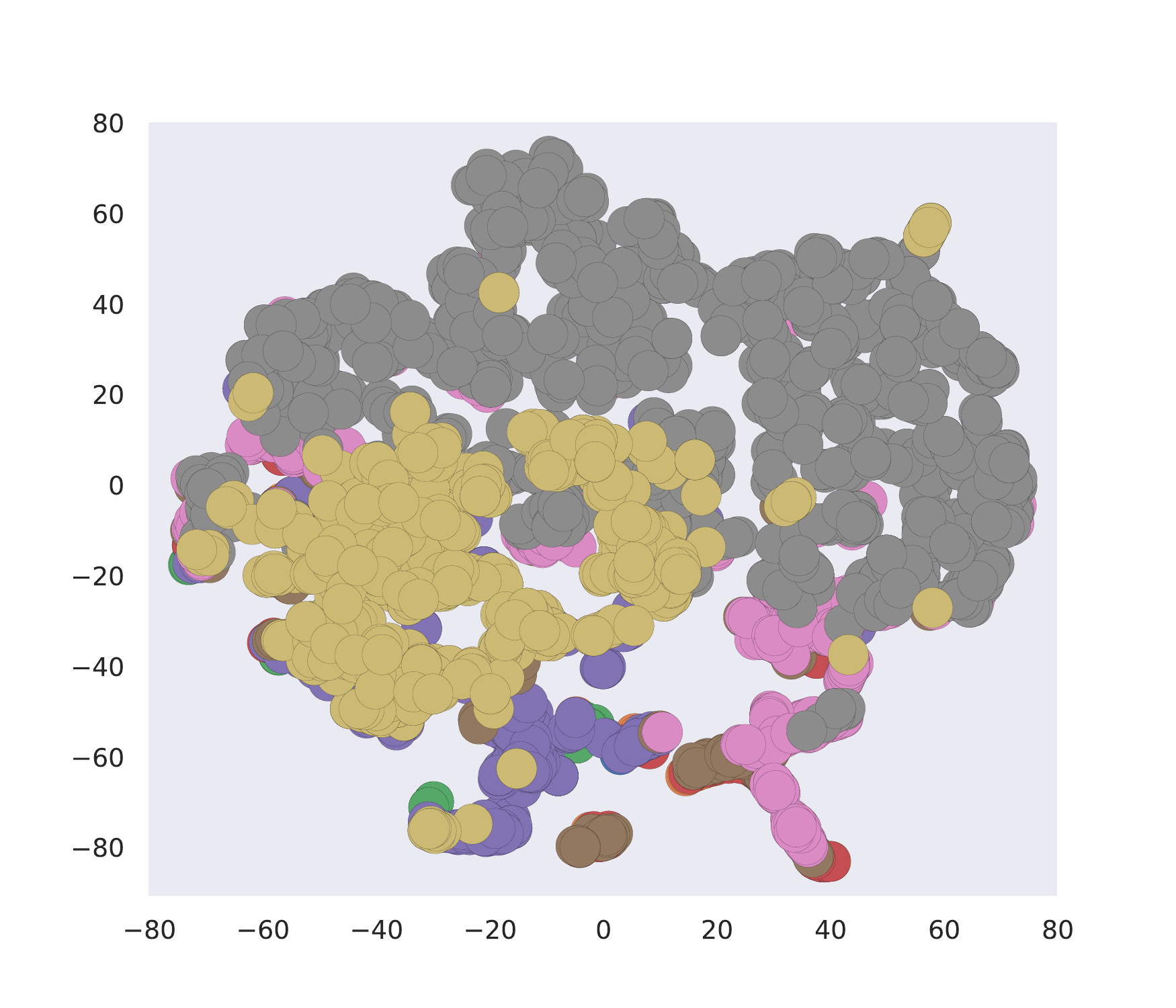}}
\subfigure[Pendigits]{\label{fig:pendigits_lomdt}\includegraphics[width=0.32\linewidth]{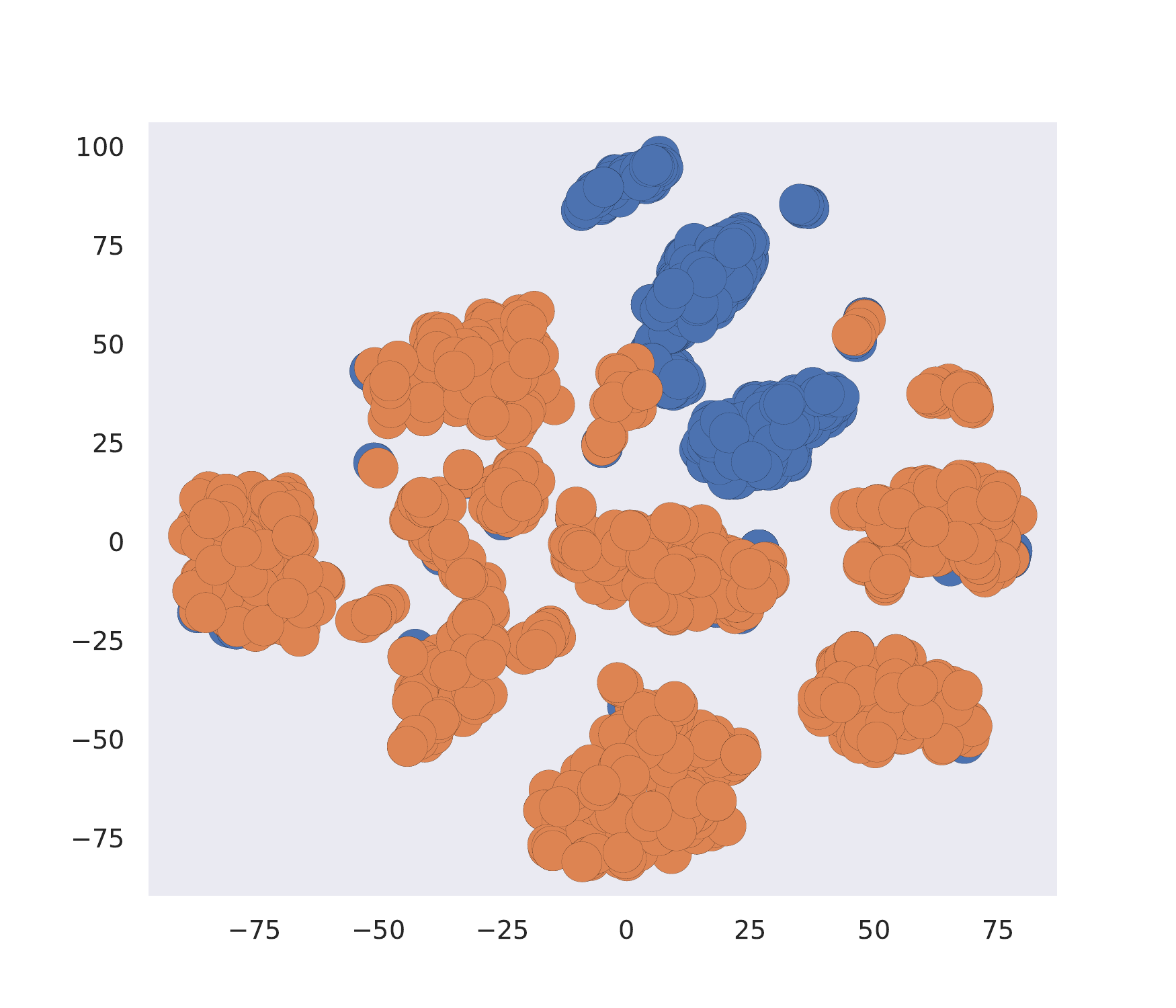}}
\subfigure[Mammography]{\label{fig:mammography_lomdt}\includegraphics[width=0.32\linewidth]{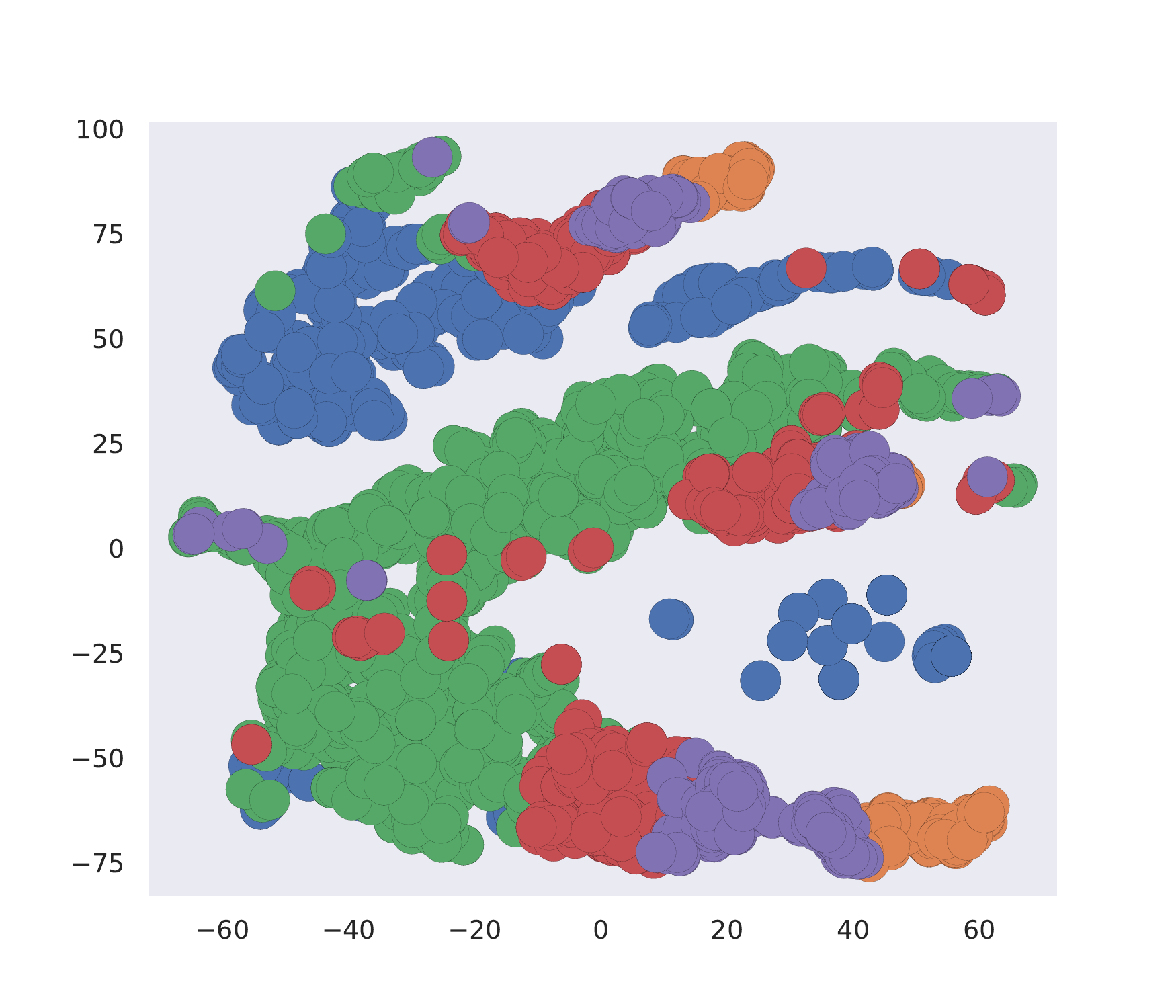}}
\subfigure[Satimage-2]{\label{fig:satimage_lomdt}\includegraphics[width=0.32\linewidth]{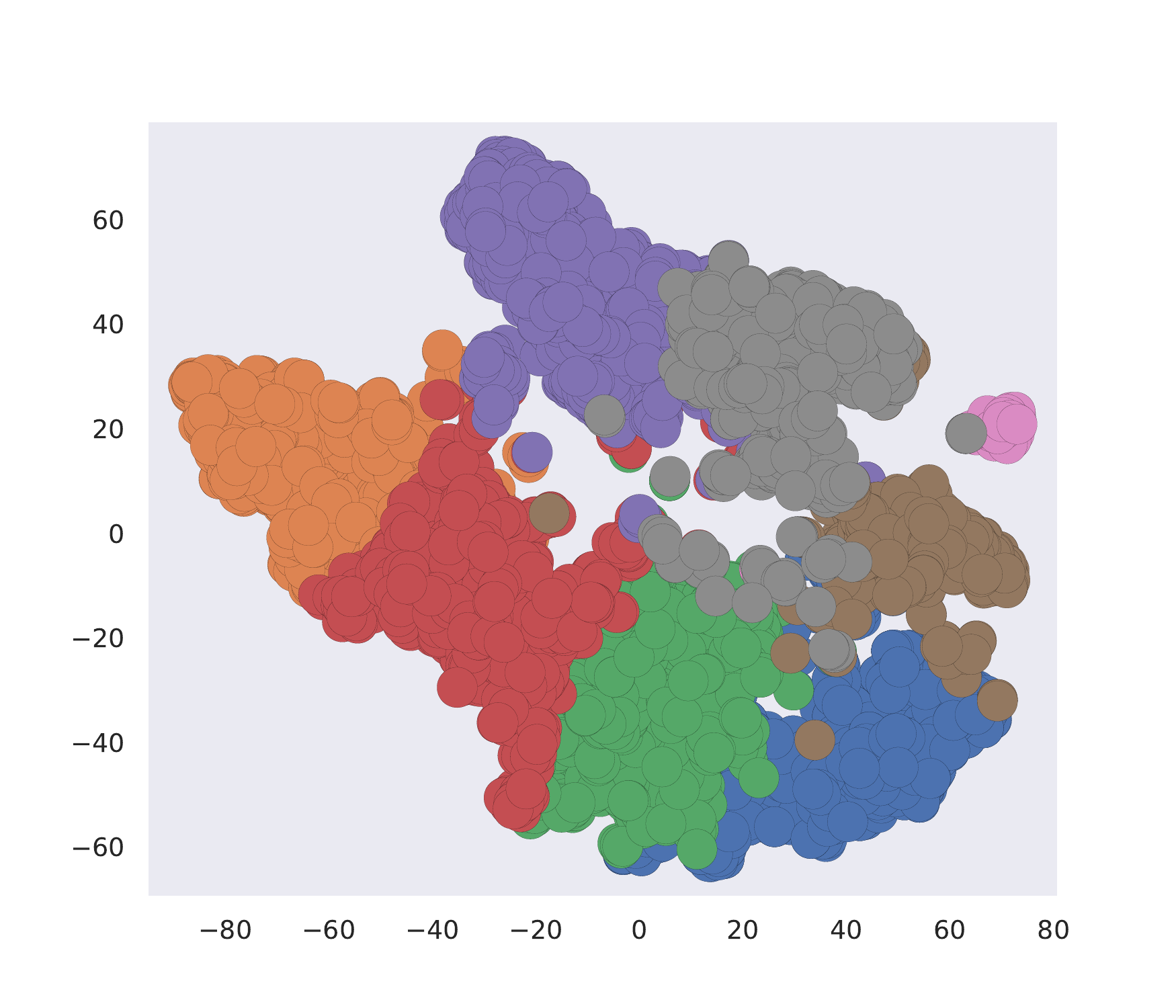}}
\subfigure[SpamBase]{\label{fig:spamBase_lomdt}\includegraphics[width=0.32\linewidth]{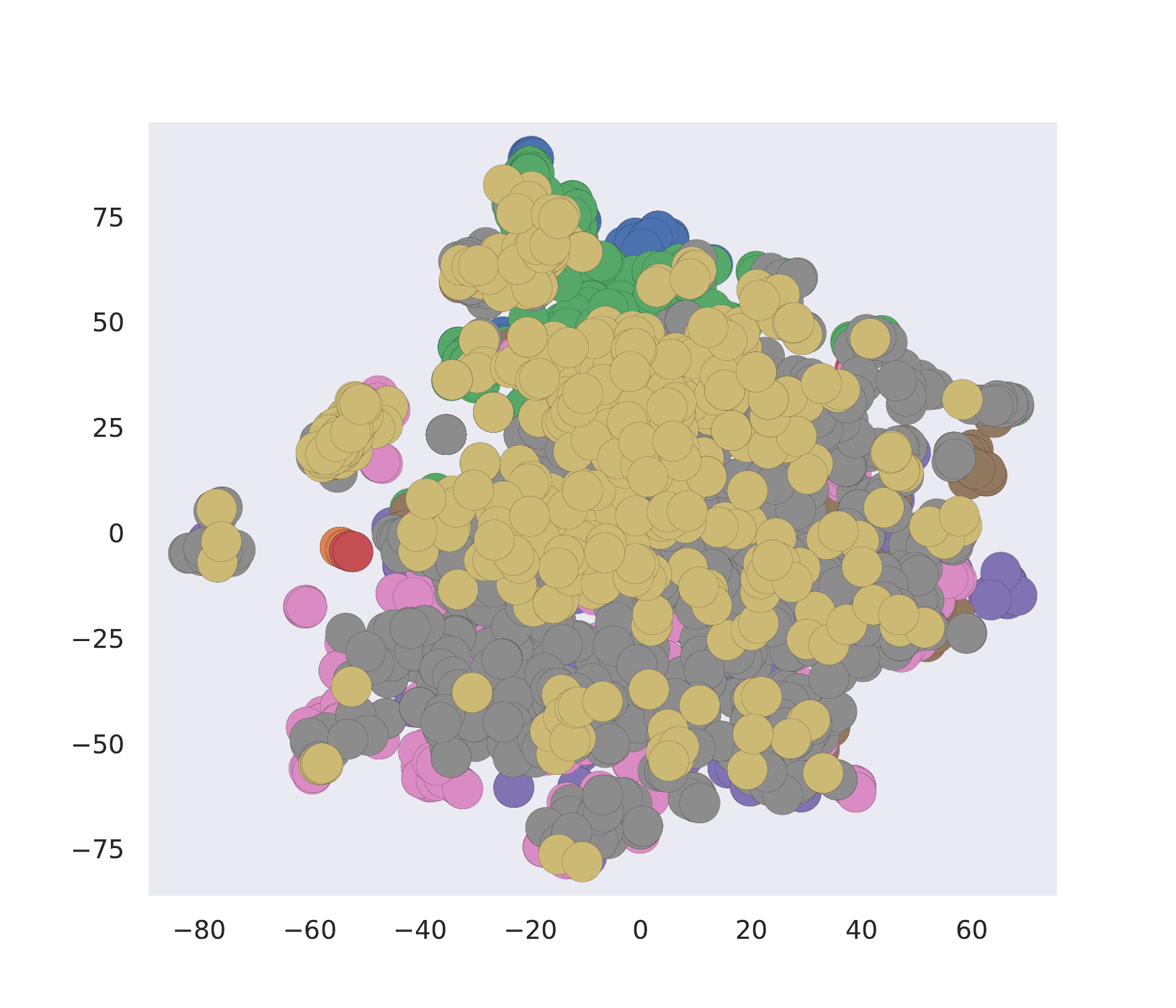}}
\subfigure[Satillite]{\label{fig:satillite_lomdt}\includegraphics[width=0.32\linewidth]{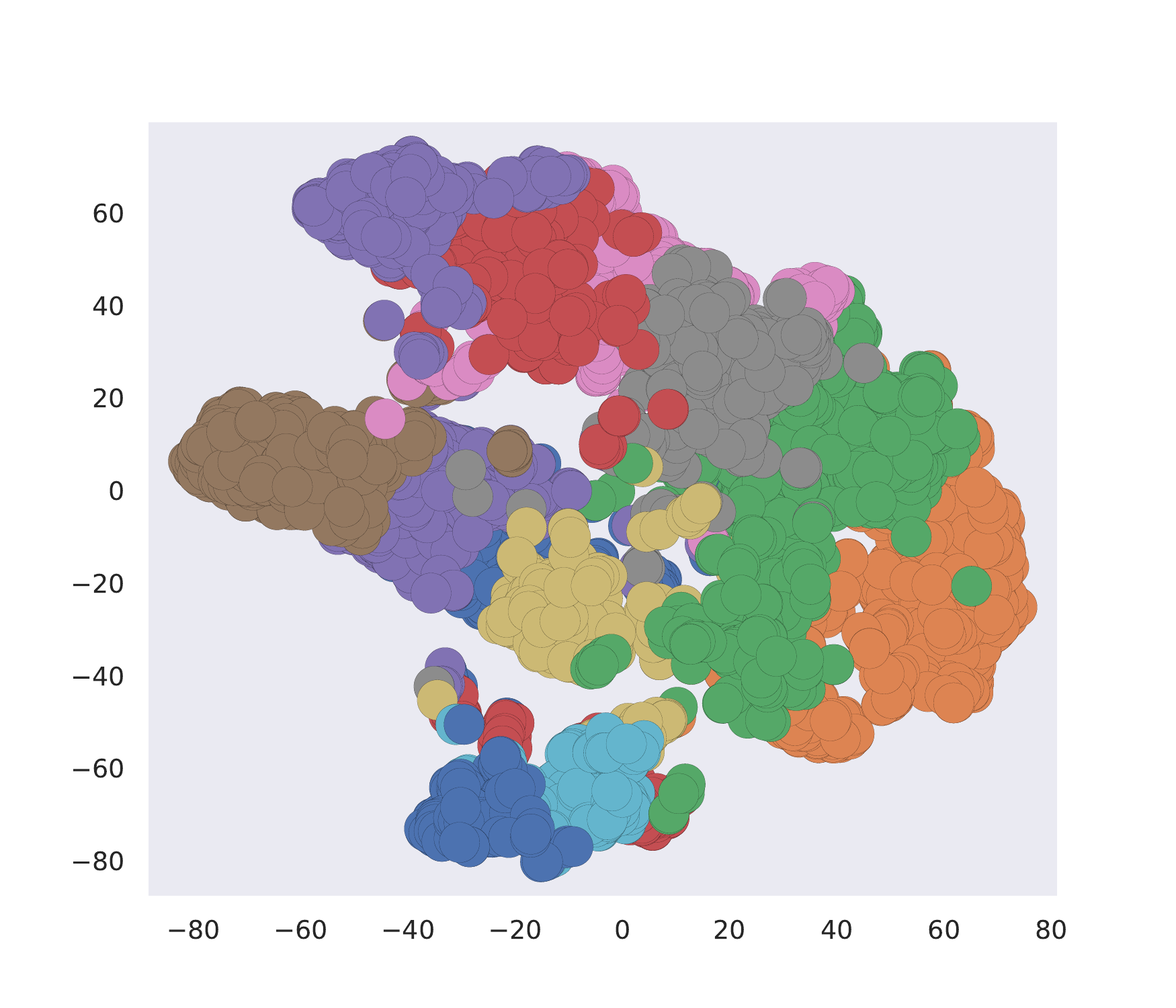}}
% \subfigure[Pima]{\label{fig:pima_lomdt}\includegraphics[width=0.32\linewidth]{figure/Pima_cluster.pdf}}
% \subfigure[Shuttle]{\label{fig:shuttle_lomdt}\includegraphics[width=0.32\linewidth]{figure/Shuttle_cluster.pdf}}
% \subfigure[Cover]{\label{fig:cover_lomdt}\includegraphics[width=0.32\linewidth]{figure/cover_cluster.pdf}}
    \caption{Visualization of L-\sys (Q5).}
    \label{fig:cluster_visualization_of_lstair}
\end{figure}

\subsection{Comparison Against Baselines (Q2): $F_1$ Score}
\label{sec.exp.F1}
In this section, we evaluate the $F_1$ score of each algorithm when they produce a rule set with similar total length.
For each dataset, we vary the total rule length by selecting 10 numbers within a range from 0 to the total rule length result with respect to the ID3 algorithm in Table~\ref{tab:overall_performance_comparison}.
For instance, in Table~\ref{tab:overall_performance_comparison} the total rule length of ID3 on the dataset Pendigits is 290. We thus select ten numbers: 29, 58, $\cdots$, 290 as the candidate total lengths. 

Then given one total length $l$, we run ID3, CART, and \sys in the following way to obtain the corresponding $F_1$ score: (1) For ID3, we gradually increase the depth of the tree until it generates a rule set with a total length slightly higher than $l$; (2) For CART, we first build a decision tree that is as accurate as possible and then prune it until it has a length close to $l$; (3) For \sys, we update the breaking condition of Algorithm~\ref{alg:stair} such that it will terminate after reaching the length $l$. 
We run this experiment on the Pendigits and Thursday-01-03 datasets.
As shown in Figure~\ref{fig:AccwrtLength}, \sys is more accurate than ID3 and CART when they produce a rule set with the similar total length, indicating that given the same budget on the total rule length, \sys produces rules with higher accuracy.

\subsection{Effect of Hyper-parameters $L_m$ and $F1_m$ (Q3)}
\label{Study_of_the_parameter_lm_and_f1m}

% \todo{why not L-\sys?}

In this set of experiments, we first study how the maximal length $L_m$ affects \sys and L-\sys. We fix the F1 score threshold $F1_m$ as 80\% and then vary $L_m$ from 2 to 12 and measure the how the total rule length changes. 
Note in some cases when $L_m$ is too small, e.g. 2, the learned tree cannot meet the F1 score requirement. 
As shown in Figure~\ref{fig:effects_of_lm}, as $L_m$ gets larger, the total rule length will get smaller. This is because with a looser constraint, \sys gets a larger search space and hence better chance to find a simple tree. When \sys gets better, L-\sys will also gets better. Besides, we observe that L-\sys could reach the minimal total rule length with smaller $L_m$. This shows the power and benefits of localization. 

Next, we investigate how the F1 score threshold $F1_m$ affects \sys and L-\sys. 
We fix $L_m$ to 10 and vary $F1_m$ from 0.70 to 0.95.
As shown in Figure~\ref{fig:effects_of_f1m}, in most of the cases \sys outperforms ID3 and CART, while L-\sys consistently outperforms all other methods in all scenarios by up to 94.0\% as shown in the results on the \emph{Cover} dataset when the threshold is set as 0.95.
The larger the $F1_m$ threshold is, the more L-\sys outperforms other baselines. This is because partitioning allows L-\sys to get a set of localized trees, each of which produces high accurate classification results on the corresponding data subset.  

\begin{figure*}[t]
\centering     %%% not \center
\subfigure[PageBlock]{\label{fig:pageblock_m}\includegraphics[width=0.19\linewidth]{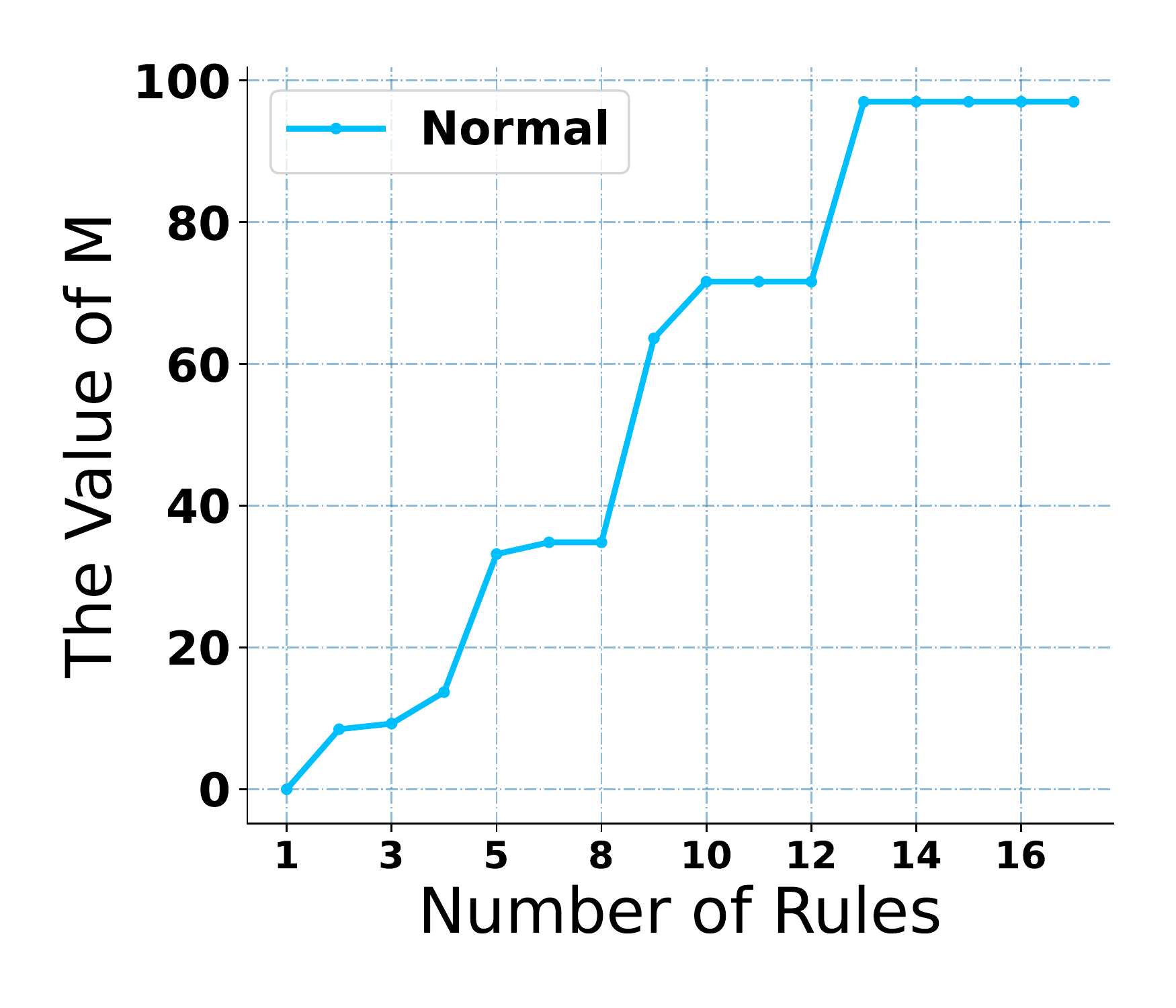}}
\subfigure[Pendigits]{\label{fig:pendigits_m}\includegraphics[width=0.19\linewidth]{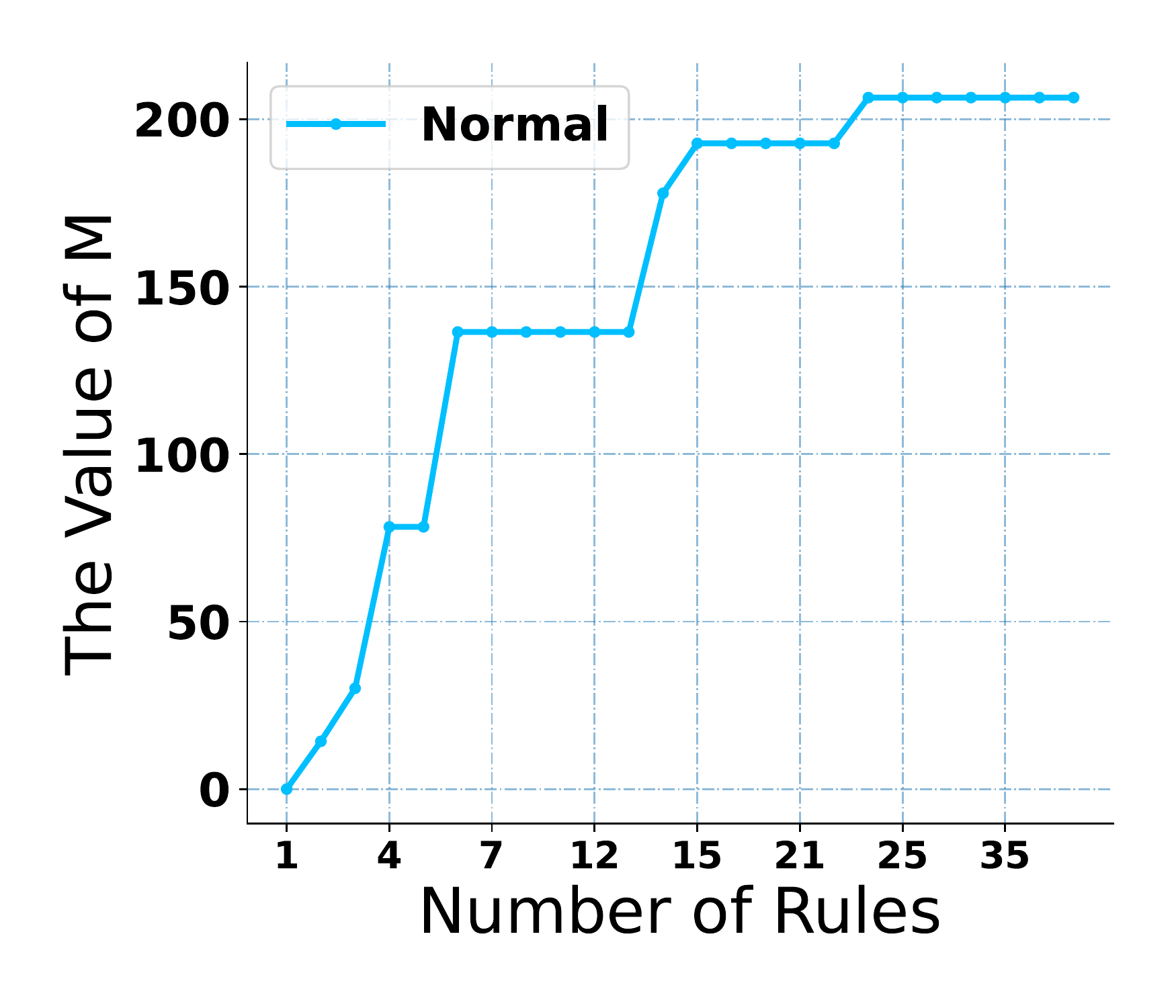}}
\subfigure[Shuttle]{\label{fig:shuttle_m}\includegraphics[width=0.19\linewidth]{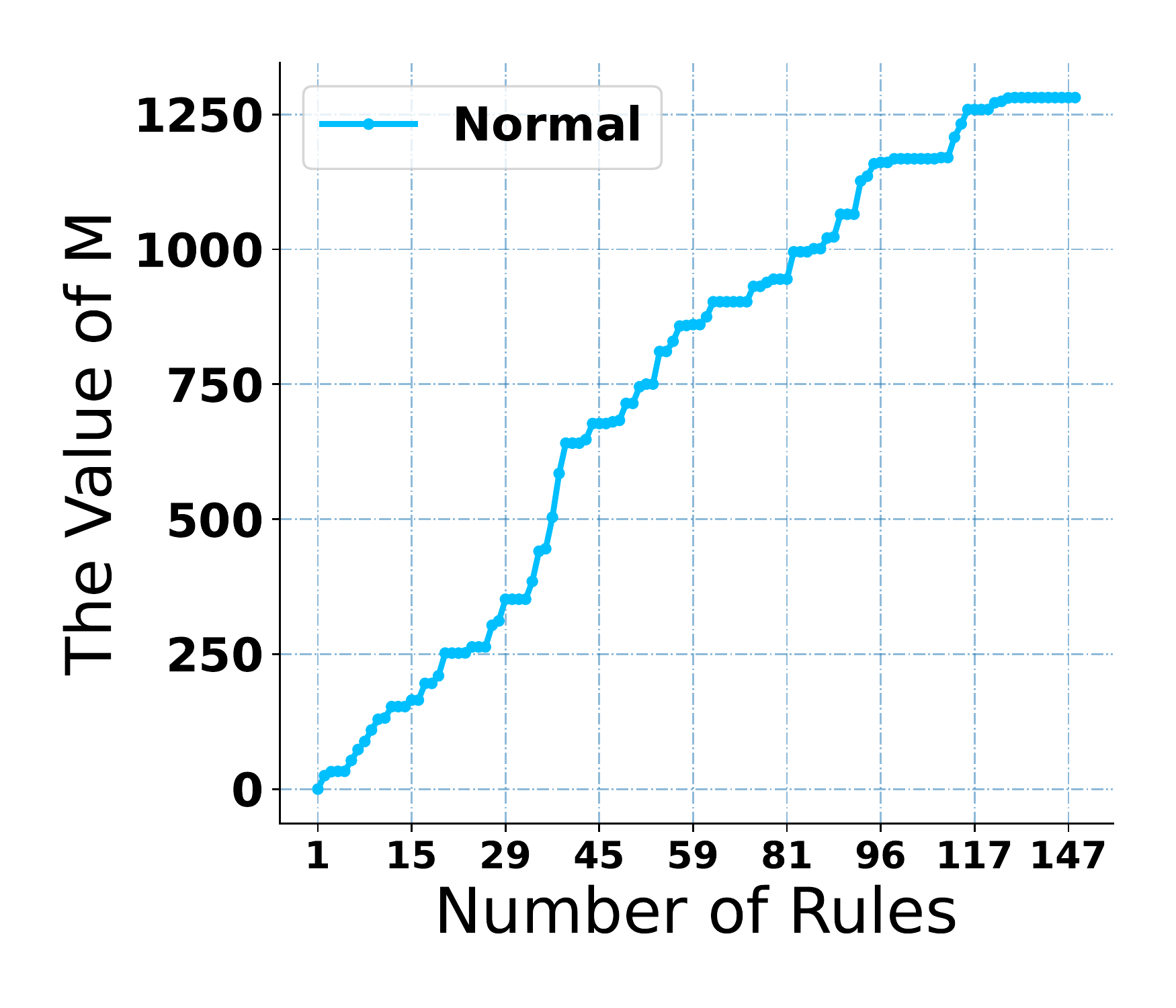}}
\subfigure[Pima]{\label{fig:pima_m}\includegraphics[width=0.19\linewidth]{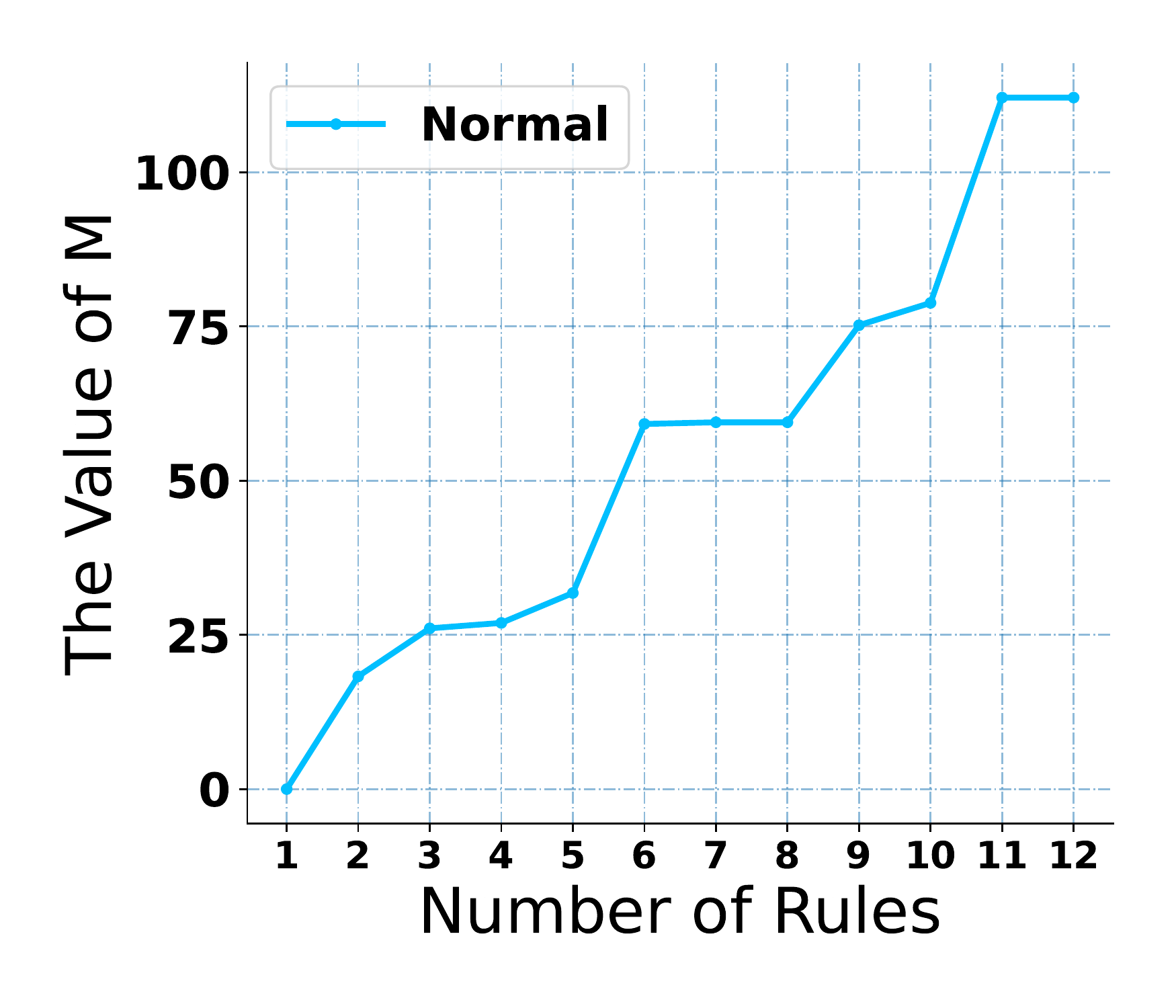}}
\subfigure[Mammography]{\label{fig:mammography_m}\includegraphics[width=0.19\linewidth]{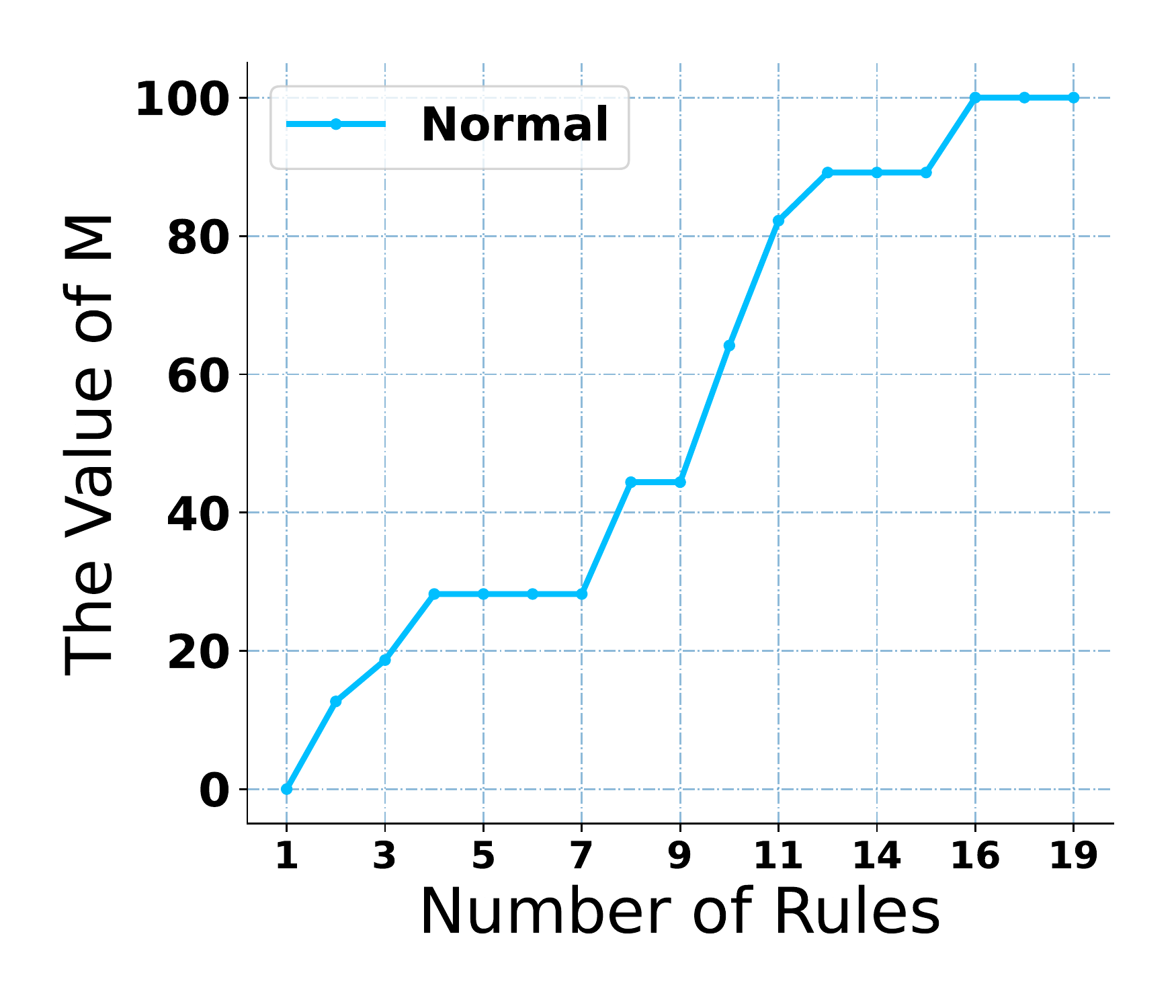}}
% \subfigure[Musk]{\label{fig:http_m}\includegraphics[width=0.19\linewidth]{figure/Musk_increasing_of_M.pdf}}
\subfigure[Satimage-2]{\label{fig:satimage2_m}\includegraphics[width=0.19\linewidth]{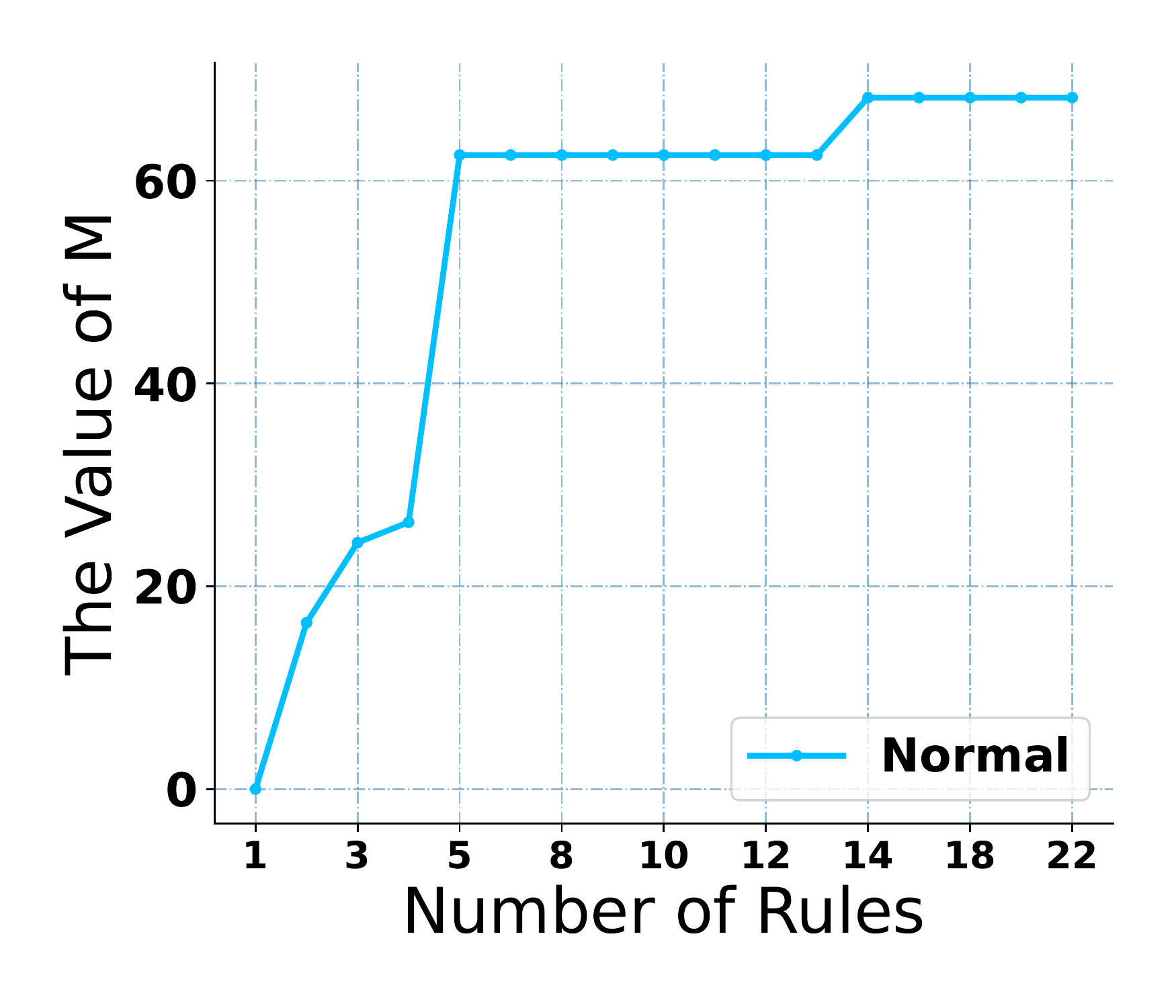}}
\subfigure[Satellite]{\label{fig:satellite_m}\includegraphics[width=0.19\linewidth]{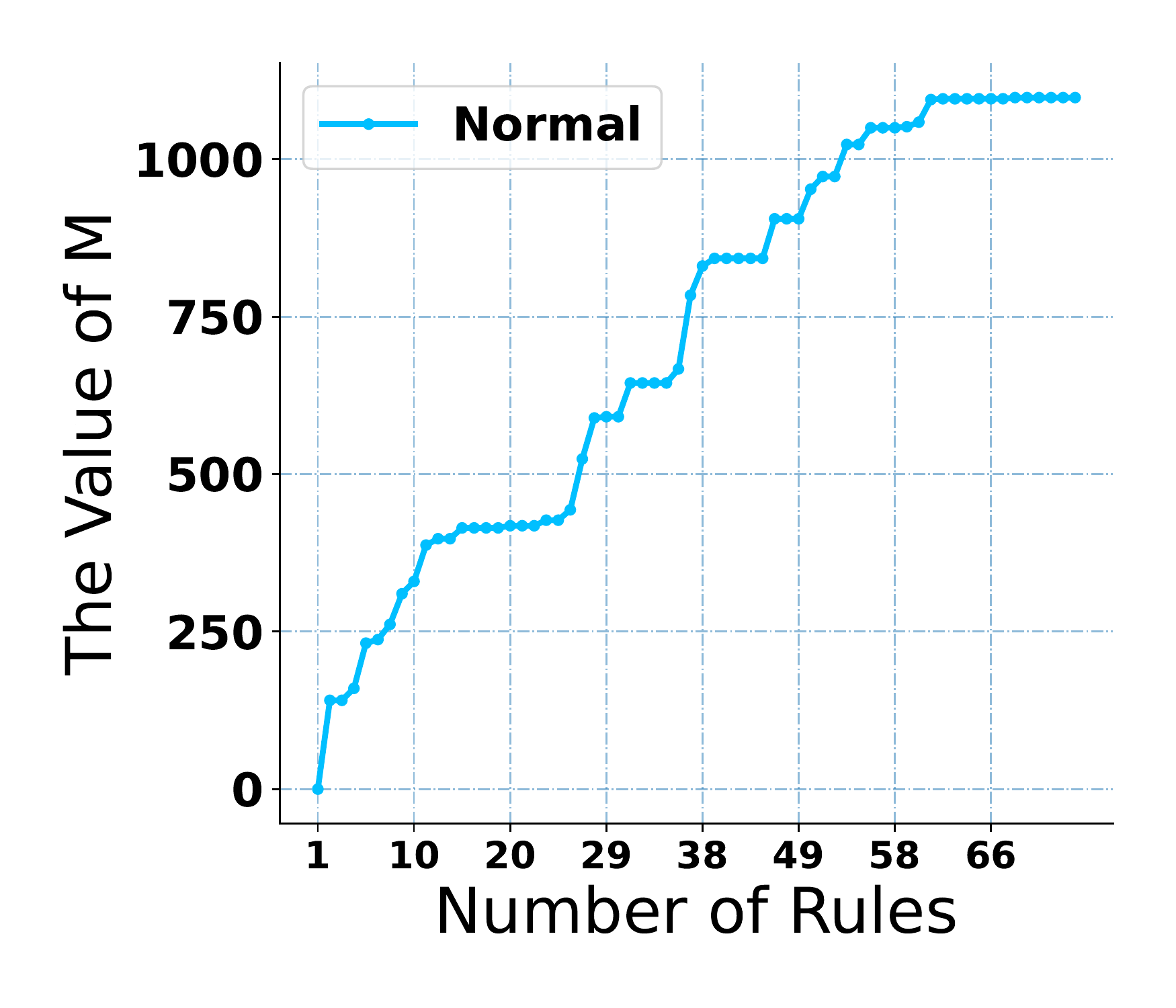}}
\subfigure[SpamBase]{\label{fig:spambase_m}\includegraphics[width=0.19\linewidth]{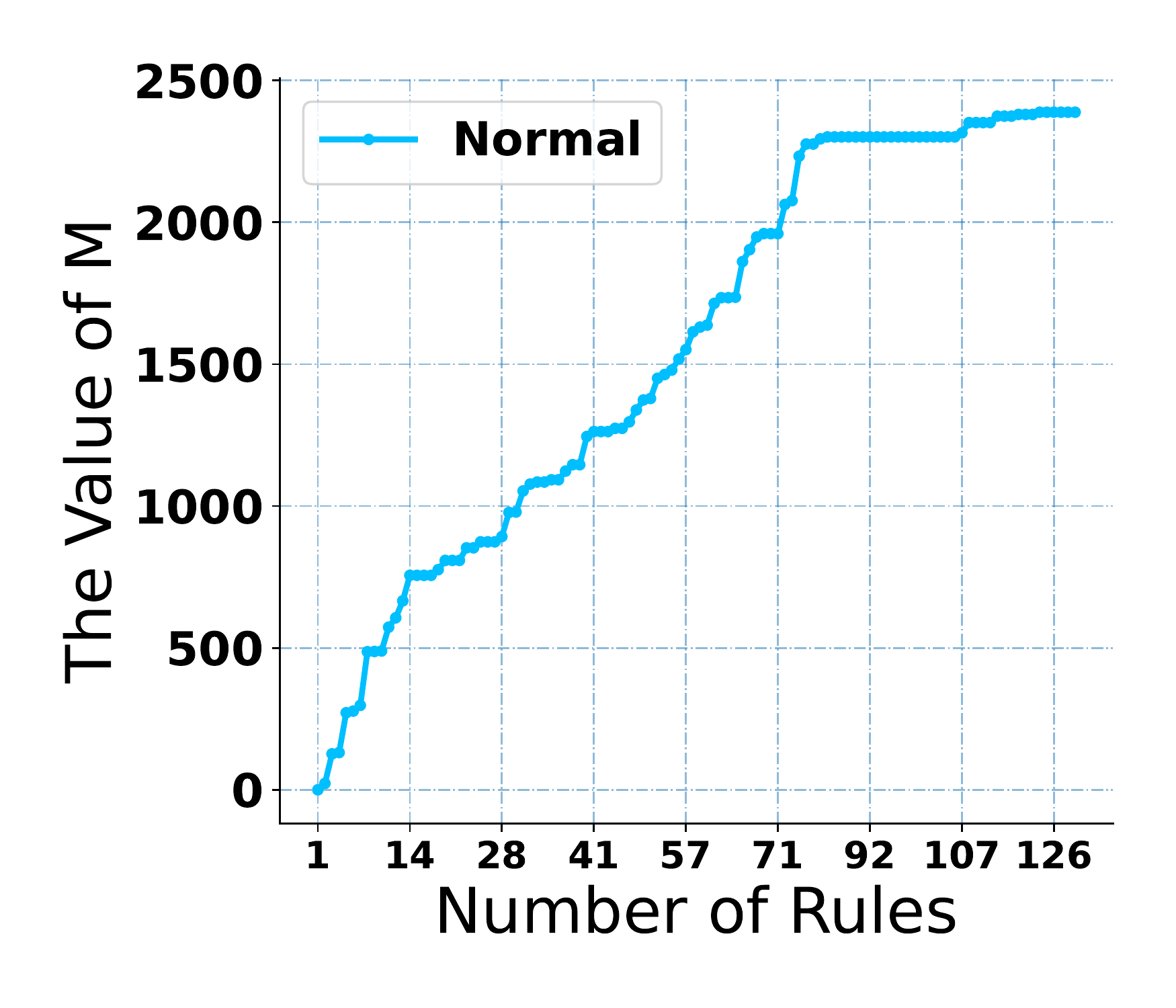}}
\subfigure[Cover]{\label{fig:cover_m}\includegraphics[width=0.19\linewidth]{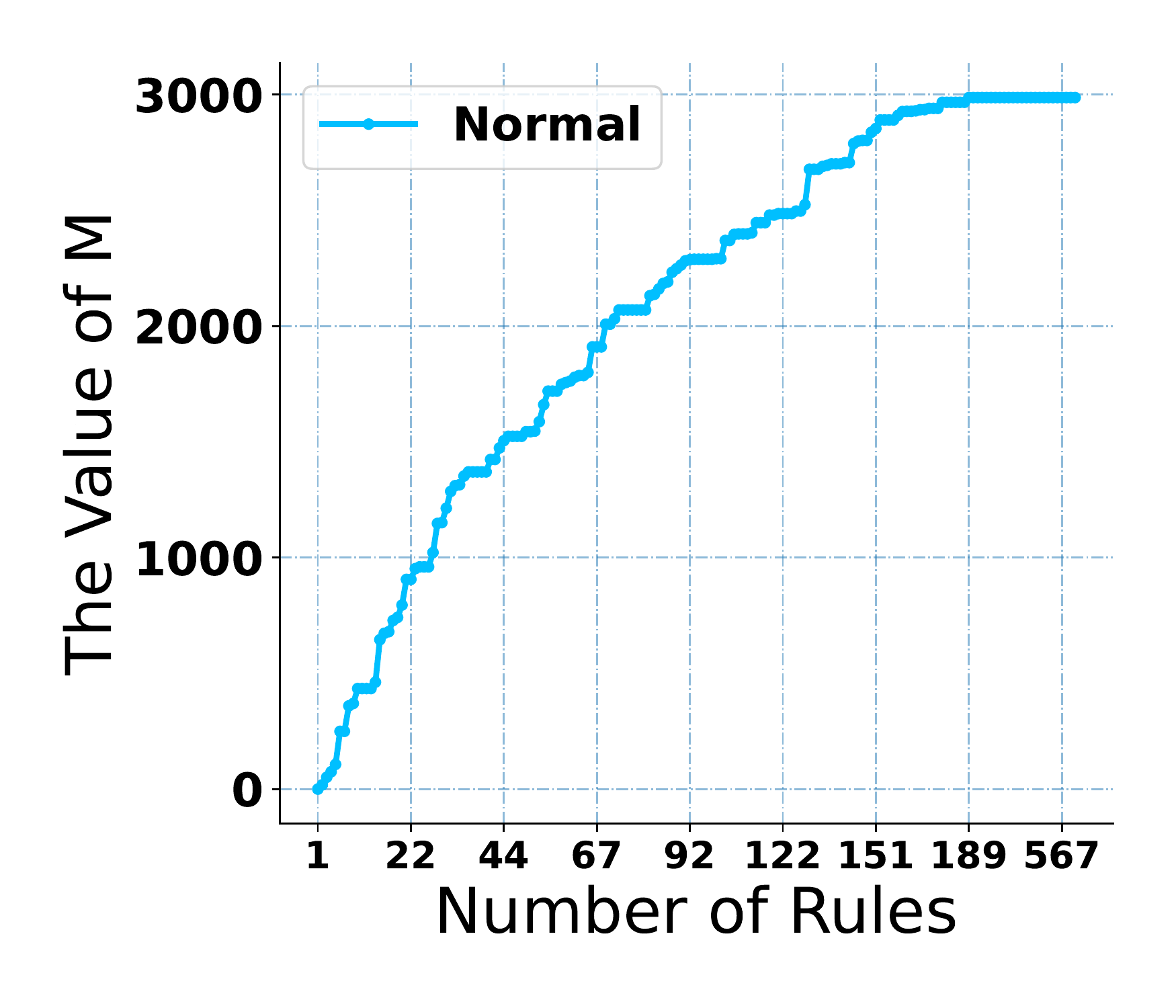}}
\subfigure[Thursday-01-03]{\label{fig:thursday-01-03_m}\includegraphics[width=0.19\linewidth]{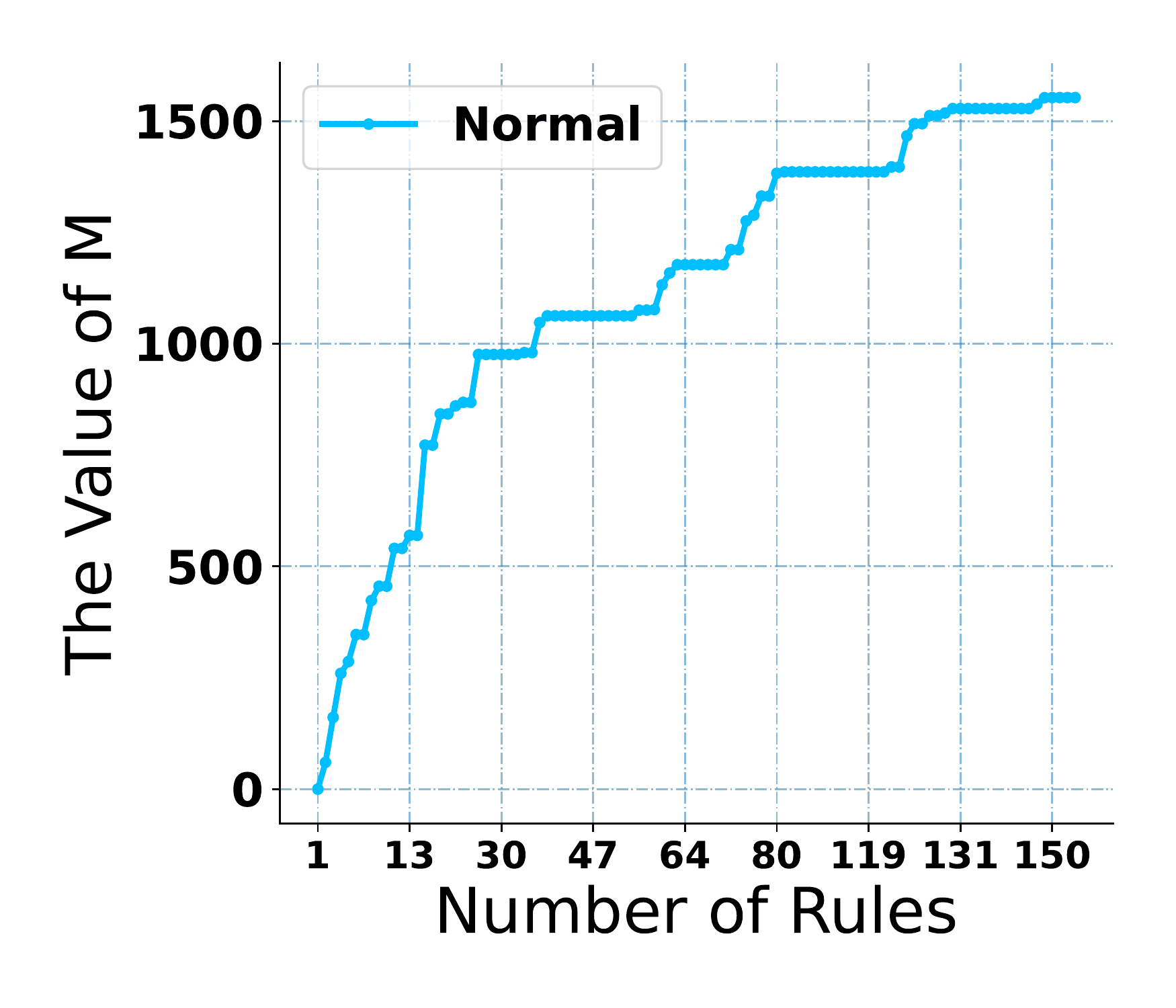}}
%\vspace{-0.4cm}
\caption{The dynamic $M$ during training (Q6).}
% \vspace{-0.4cm}
\label{fig:Analysis_on_increasing_of_M}
\end{figure*}

\vspace{-5pt}
\subsection{Number of partitions in L-\sys (Q4)}
\label{study_of_the_number_of_clusters}
We study how the initial number of the partitions $n$ affects L-\sys. In this set of experiments, $n$ is selected from \{2, 4, 8\}. In addition to the total rule length, we also report the number of rules in the final ruleset. 
From the results shown in Table~\ref{tab:study_of_n}, we have the following observations: (1) Compared to the results in Table~\ref{tab:overall_performance_comparison}, no matter what $n$ L-\sys starts with, it consistently outperforms other methods; (2) L-\sys always performs well when starting with a small $n$ compared to other initial $n$ values, indicating that $n$ is not a hyper-parameter that requires careful tuning. 

%and comparing against the visualization of K-means clustering.
\begin{table}[t]
\centering
\caption{Multi-class classification: total rule length.}
\vspace{-5pt}
\label{tab:performance_comparison_for_wine_quality}
\begin{tabular}{c|c|c|c|cc}
\toprule
Dataset & \multicolumn{1}{c|}{ID3} & \multicolumn{1}{c|}{CART}  & \multicolumn{1}{c|}{STAIR} & \multicolumn{1}{c}{L-STAIR}  \\
 \midrule
Wine Quality & 5133 & 3217 & 2251 & \textbf{1538} \\ 
\bottomrule
\end{tabular}
\vspace{-5pt}
\end{table}

\begin{table*}[ht]
    \centering
    \caption{Multi-class classification: the number of partitions $n$ in L-\sys}
    \label{tab:study_of_n_wine}
    \begin{tabular}{c|ccc|ccc|ccc}
    \toprule
         & \multicolumn{3}{c|}{L-STAIR ($n$=2)} & \multicolumn{3}{c|}{L-STAIR ($n$=4)}  & \multicolumn{3}{c}{L-STAIR ($n$=8)}  \\
         \midrule
        Dataset & Length & \# of R & \# of C & Length & \# of R & \# of C & Length & \# of R & \# of C \\
        \midrule
        Wine Quality & 1642 & 614 & 11 & 1692 & 620 & 13 & 1538 & 635 & 17 \\
        \bottomrule
    \end{tabular}
\end{table*}

\begin{figure}[ht]
    \centering
    \subfigure[Different Thresholds]{\label{fig:thresholds_wine}\includegraphics[width=0.48\linewidth]{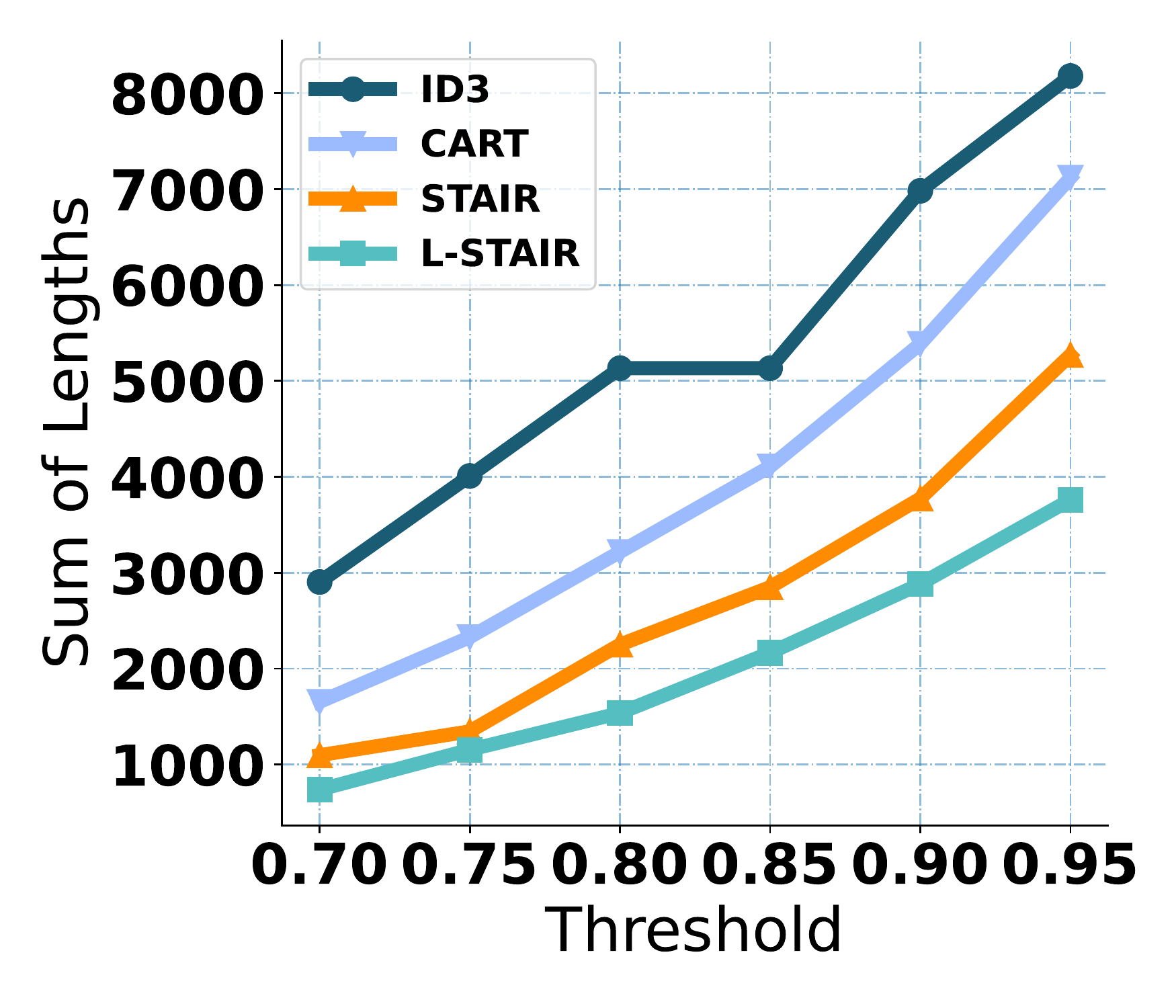}}
    \subfigure[Different $L_m$]{\label{fig:l_m_wine}\includegraphics[width=0.48\linewidth]{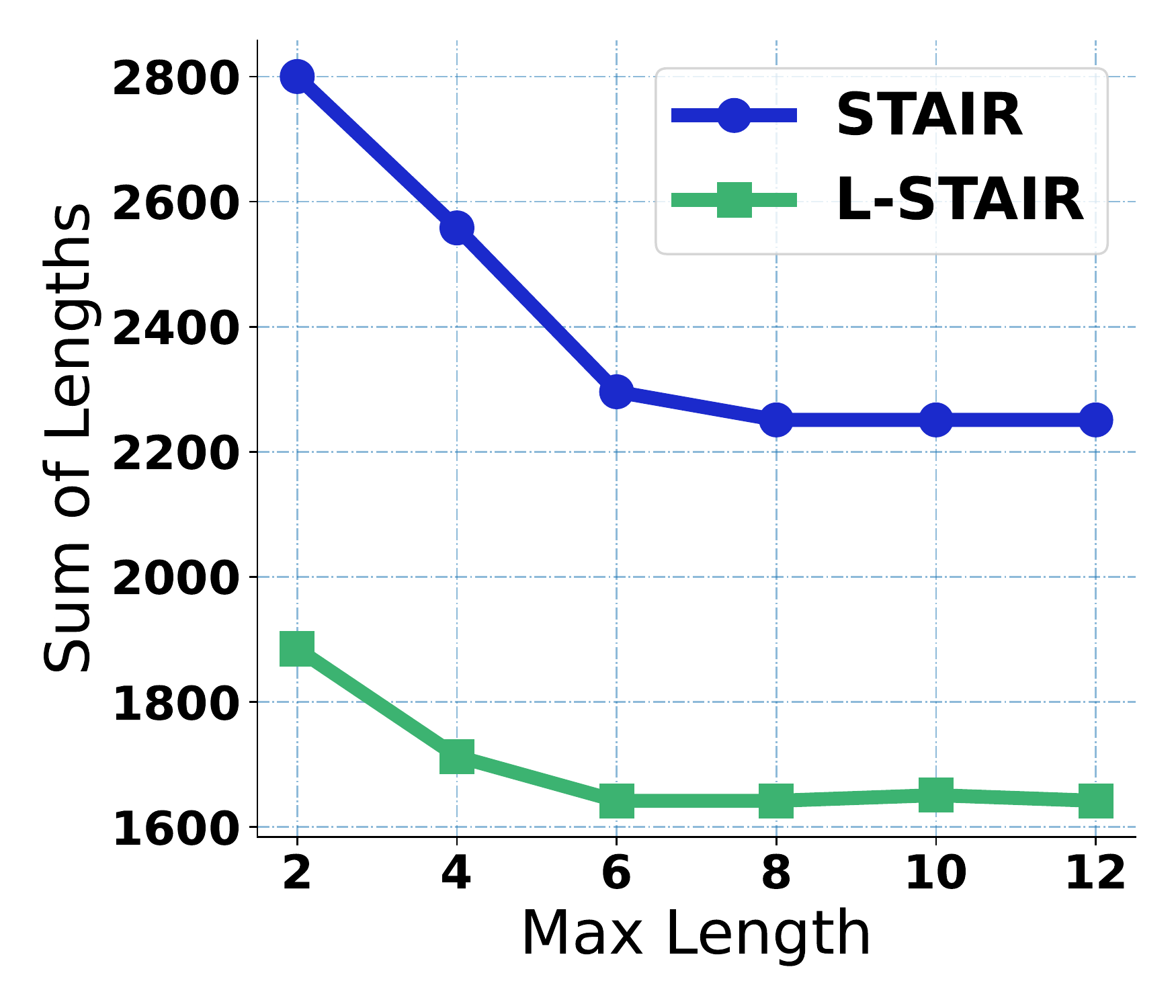}}
    \subfigure[Cluster Visualization]{\label{fig:wine_cluster}\includegraphics[width=0.48\linewidth]{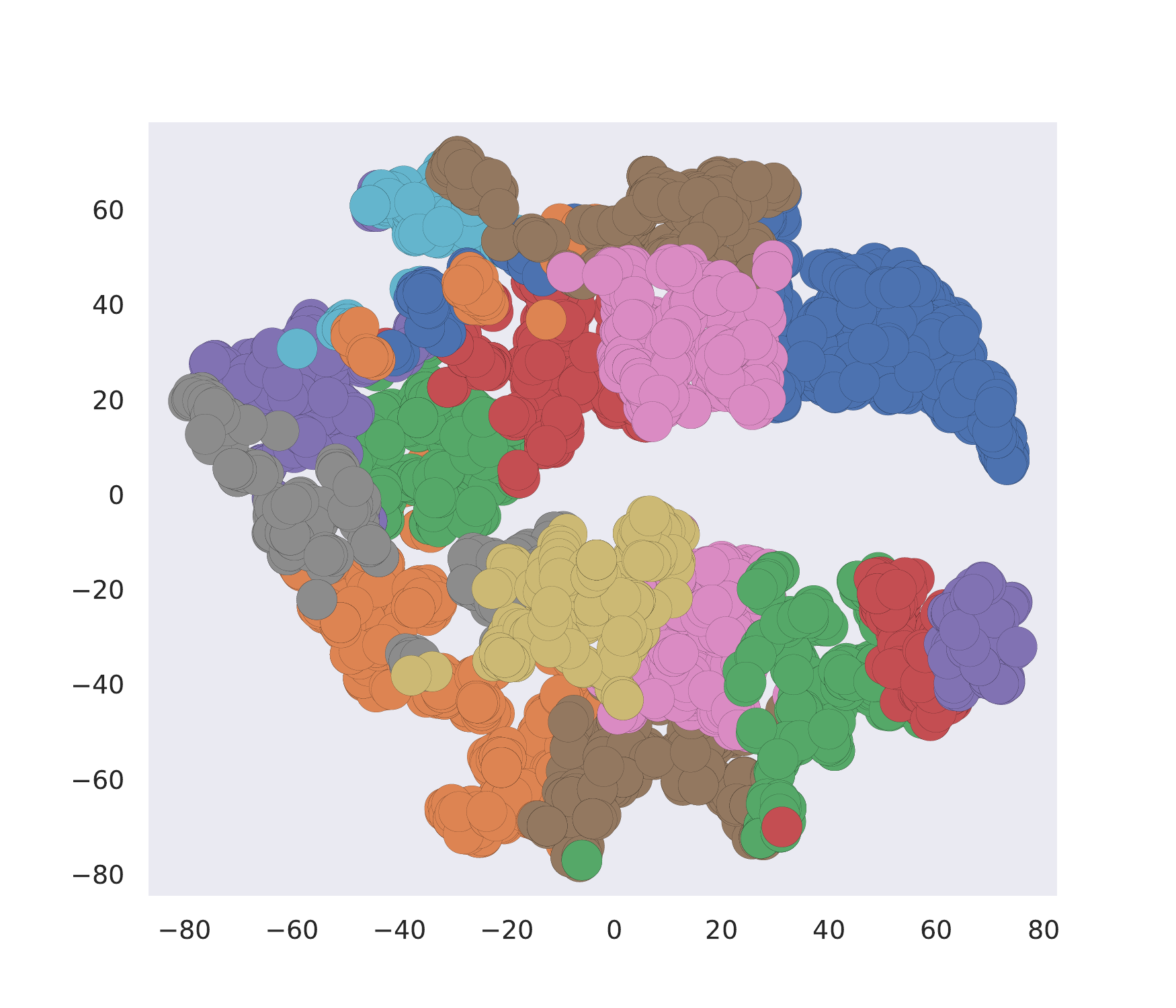}}
    \subfigure[Increasing of M]{\label{fig:increasing_of_M_wine}\includegraphics[width=0.48\linewidth]{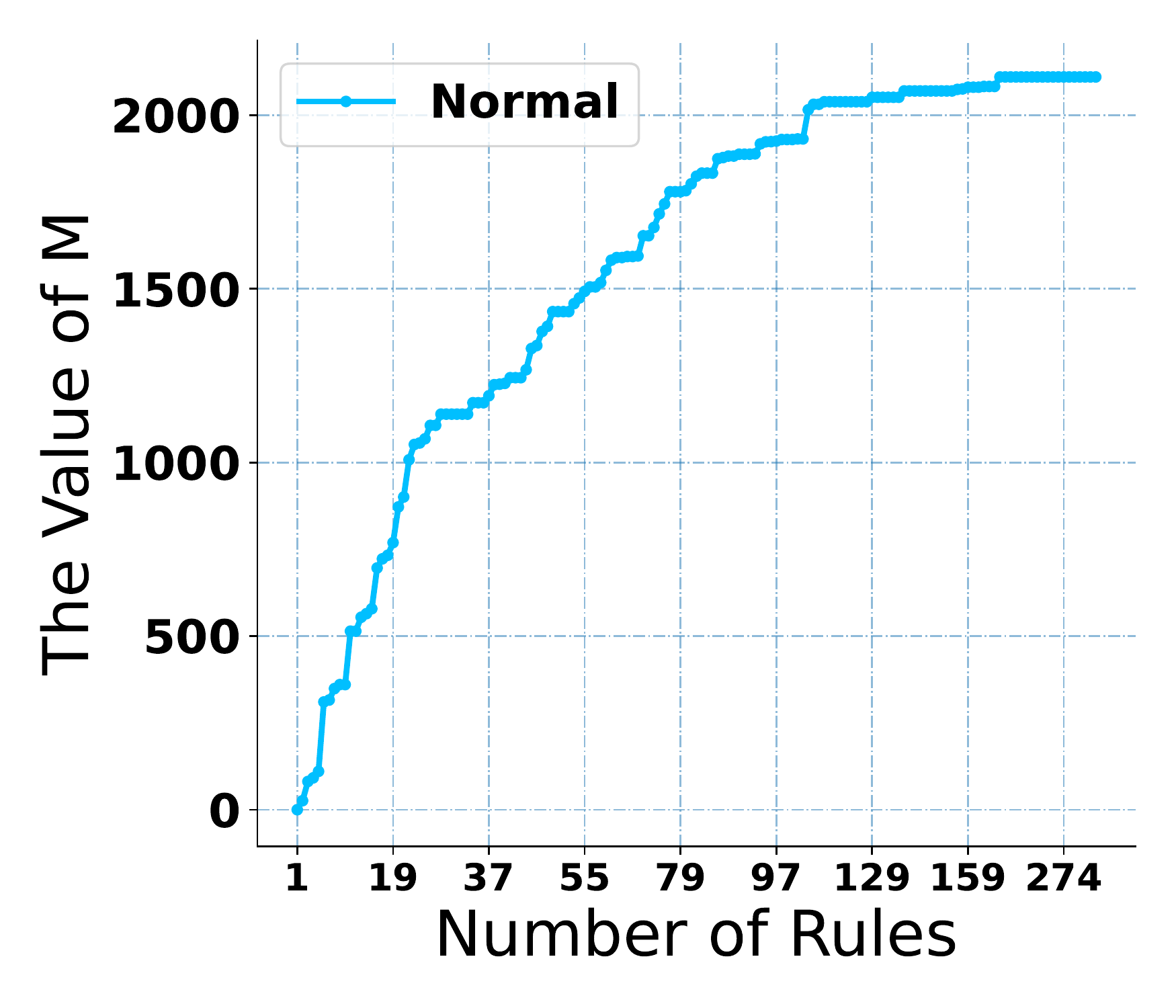}}
    \caption{Multi-class dataset \emph{Wine Quality}: ablation study}
    \label{fig:Wine.effects_of_lm}
    \vspace{-10pt}
\end{figure}
% \vspace{3mm}
\vspace{-5pt}
\subsection{L-\sys: Locality-Preserving (Q5)}
Next, we evaluate if the partitioning of L-\sys is able to preserve the locality of the data. 
We show this by visualizing its data partitioning.
Before visualization, We apply T-SNE to embed the data into 2D. 
We plot different partitions in different colors. 
Due to space limit, we only plot the partitioning of 6 datasets. 
As shown in Figure~\ref{fig:cluster_visualization_of_lstair}, on all datasets the partitioning of L-\sys preserves the locality. This thus guarantees the interpretability of each localized tree. 

\vspace{-5pt}
\subsection{Dynamically Adjusting the Value of M (Q6)}
In this set of experiment we show how \sys automatically adjusts the value of stabilizer $M$ introduced in our summarization and interpretation-aware optimization objective (Sec.~\ref{sec.objective.optimized}). 
To better understand the influence of a dynamically adjusting $M$, we use the number of rules produced in the training process as the reference variable, corresponding to the x-axis. From Figure~\ref{fig:Analysis_on_increasing_of_M}, we observe: (1) The value of $M$ continuously increases during the training process to split nodes and thus produce valid rules; (2) The values of $M$ are different across different datasets, indicating that it is hard to get an appropriate $M$ by manual tuning. 

\vspace{-5pt}
\subsection{Multi-class classification Problems (Q7)}
We use this set of experiments to show that \sys and L-\sys are generally applicable to the more complicated multi-class classification problems. 
We use one of the most popular classification datasets \emph{Wine Quality}\footnote{https://archive.ics.uci.edu/ml/datasets/wine+quality}~\cite{WineQuality}, which contains 4898 instances and 12 attributes. 
We regard the attribute ``score'' as the target which corresponds to integers within the range from 0 to 10 and run a classification task on it. 
Our \sys and L-\sys could be easily extended to multi-class settings by replacing the F1 score with the {\it classification accuracy}. 
%In this experiment, we directly use the ground truth label as the targets to explain. 

As shown in Table~\ref{tab:performance_comparison_for_wine_quality}, we report the total rule length, same to the outlier detection scenario.
We observe from the results: (1) L-\sys and \sys significantly outperform ID3 and Cart by up to 70.0\%;
(2) As shown in Table-\ref{tab:study_of_n_wine}, the initial number $n$ of the partitions make little difference to the resulted lengths, indicating L-\sys is not sensitive to the hyper-parameter $n$;
(3) As illustrated in Figure~\ref{fig:thresholds_wine} and Figure~\ref{fig:l_m_wine}, \sys and L-\sys always outperform the baselines no matter how the accuracy threshold and the maximal rule length threshold vary; 
(4) Figure~\ref{fig:wine_cluster} visualizes the partitions produced by L-\sys. The locality of the data partitions is well-preserved; 
(5) As shown in Figure~\ref{fig:increasing_of_M_wine}, the dynamic update of the value of stabilizer $M$ is important in splitting the nodes and producing valid rules, similar to the case of outlier detection.

\section{Related Work}
\label{related_work}

\noindent\textbf{Outlier Summarization and Interpretation.} 
To the best of our knowledge, the problem of summarizing and interpreting outlier detection results has not been studied. Focused on a special type of outliers, Scorpion~\cite{DBLP:journals/pvldb/0002M13} produces explanations for outliers in aggregation queries by looking at the provenance of each outlier. 
If removing some objects from the aggregation significantly reduces the abnormality of a given outlier, these objects will be considered as the cause of this outlier.
Similar to Scorpion, Cape~\cite{DBLP:journals/pvldb/MiaoZLGKR19} targets explaining the outliers in aggregation queries. 
But rather than rely on provenance, Cape uses the objects that counterbalance the outliers as the explanation. More specifically, if including some additional data objects into the aggregation query could neutralize an outlier, these objects effectively explain the outlier.
Both works do not tackle the problems of summarizing outliers. 
Macrobase~\cite{bailis2017macrobase} explains outliers by correlating them to some external attributes such as location or occurring time using associate rule mining. These external attributes are not used to detect anomalies. In many applications, however, such external attributes do not exist. %It is thus effective only in a narrow scope.
Further, Macrobase only explains the outliers detected by its default density-based outlier detector customized to streaming data and does not easily generalize to other outlier detection methods. 
 
\noindent \textbf{Interpretable AI}. Some works~\cite{DBLP:journals/corr/abs-2105-05328,phillips2020four} target on interpreting the machine learning models, or understanding the model better with extra information~\cite{debias}. Some works such as LIME~\cite{DBLP:conf/kdd/Ribeiro0G16}, Anchor~\cite{DBLP:conf/aaai/Ribeiro0G18}, LORE~\cite{DBLP:journals/corr/abs-1805-10820} produce an explanation with respect to each inference. 
LIME~\cite{DBLP:conf/kdd/Ribeiro0G16} explains the predictions of a classifier by learning a linear model locally around the prediction with respect to a particular testing object. It then uses the attributes that are most important to the linear prediction as the explanation. 
Because LIME has to learn one linear model for each individual object, using it to explain a large number of prediction results tends to prohibitively expensive. 
Some other methods~\cite{DBLP:conf/aaai/Ribeiro0G18,DBLP:journals/corr/abs-1805-10820,DeepLIFT,bach2015pixel,lundberg2017unified} explain classification results in the similar fashion.
Taking the explanations w.r.t all testing objects as input, Pedreschi et al. proposed to select a subset of the explanations to constitute a global explanation~\cite{DBLP:journals/corr/abs-1806-09936}. 
However, this work 
% yet to go through the peer review process, 
is not scalable to big datasets because it requires constructing the explanations for all testing objects. 
In addition, some techniques, including gradient-based~\cite{gradientbased1,gradientbased2} and attention based~\cite{NMTJLAT} methods, focus on particular types of deep learning model, thus hard to be used in the outlier summarization scenario. 
Lakkaraju et al~\cite{DBLP:conf/kdd/LakkarajuBL16} proposed to build a prediction model that is more interpretable than deep learning models. Follow-up methods~\cite{DBLP:conf/emnlp/SushilSD18,DBLP:conf/iccS/Pacaci0MH19,DBLP:journals/tvcg/MingQB19} work on augmenting the data to produce models with better interpretability. 
Instead of inventing new prediction models, our work focuses on explaining and summarizing the results produced by any outlier detection method.

%There are also other methods re-inventing tree-growing approaches such as DT-Extract~\cite{DBLP:journals/corr/BastaniKB17}, ADS~\cite{DBLP:journals/corr/abs-1910-12207} to utilize decision tree algorithms to generate rules for interpretation. These algorithms for generating decision sets could yield good performances. Nevertheless, they either require many steps of iteration and adjustment before obtaining reasonable and effective rules, or they might generate artificial points which is not unified with the later interpretation algorithms, which could be sub-optimal. 

\textbf{Outlier Detection}. Due to the importance of outlier detection\citep{autood}, many unsupervised outlier detection methods have been proposed including the density based method LOF~\cite{DBLP:conf/sigmod/BreunigKNS00}, the statistical-based Mahalanobis method~\cite{DBLP:books/sp/Aggarwal2013}, the distance-based methods~\cite{DBLP:conf/pkdd/AngiulliP02,Knorr99findingintensional,ramaswamy2000efficient}, and Isolation Forest~\cite{liu2008isolation}. These methods do not utilize any human-labeled data. 

As a crowd sourcing-based method, HOD~\cite{10.1145/3318464.3389772} proposed to leverage human input to improve the performance of outlier detection in text data. 
It produces some questions which once answered by humans, could help verify the status of multiple outlier candidates returned by the unsupervised outlier detectors. 
However, the question-generation process of HOD is still time-consuming. 
In addition, instead of focusing on text data, our \sys method is generally applicable to different types of data including numerical, categorical, and text data. 

\textbf{Decision Tree Algorithms.}
Because a simple tree tends to avoid overfitting and have better generalization ability, CART~\cite{reason:BreFriOlsSto84a} proposed to prune a learned decision tree to lift its performance in a post-processing step. It will remove a node in the tree if the cross-validation error rate does not increases. 
However, unlike our \sys which treats producing simple hence human interpretable rules as the first class citizen in its objective, the post-processing of CART is not very effective in minimizing the complexity of the rules, as confirmed in our experiments. Other decision tree algorithms~\cite{biggs1991method,kass1980exploratory,ritschard2013chaid} mostly suffer from the same problem that the simplicity of the rules is not considered to be as important as the classification accuracy.

\vspace{-8pt}
\section{Conclusion}
\label{conclusion}
This work targets reducing the human effort in evaluating outlier detection results. To achieve this goal, we propose \sys to learn a compact set of human understandable \textit{rules} which summarizes anomaly detection results into groups and explains why each group of objects is considered to be abnormal.
It features an outlier summarization and interpretation-ware optimization objective, a learning algorithm which optimally maximizes this objective in each iteration, and a partitioning driven \sys approach which simultaneously divides the data and produces localized rules for each data partition.
Experiment results show that \sys effectively summarize the outlier detection results with human interpretable rules, where the complexity of the rules is much lower than those produced by other decision tree methods.

% \begin{acks}
% \end{acks}

% \bibliographystyle{abbrv}
\bibliographystyle{ACM-Reference-Format}
\bibliography{ref}

\end{document}